\documentclass{article}

\usepackage{microtype}
\usepackage{hyperref}

\usepackage[accepted]{icml2024}

\usepackage{amsmath}
\usepackage{amssymb}
\usepackage{mathtools}
\usepackage{amsthm}

\usepackage[capitalize,noabbrev]{cleveref}

\theoremstyle{plain}
\newtheorem{theorem}{Theorem}[section]

\newtheorem{lemma}[theorem]{Lemma}

\theoremstyle{definition}
\newtheorem{definition}[theorem]{Definition}

\theoremstyle{remark}
\newtheorem{remark}[theorem]{Remark}

\icmltitlerunning{Optimal Batched Linear Bandits}

\usepackage{mmll}
\usepackage[colorinlistoftodos, textsize=tiny, textwidth=2.2cm, backgroundcolor=blue!10, linecolor=magenta, bordercolor=magenta]{todonotes}
\usepackage{graphicx}
\usepackage{subcaption}
\usepackage{enumitem}
\usepackage{multicol}

\newcommand{\algrule}[1][.1pt]{\par\vskip.2\baselineskip\hrule height #1\par\vskip.2\baselineskip}
\def\algname{E$^4$}%
\def\algoful{{rs-OFUL}}
\def\algendoa{{EndOA}}
\def\algped{{PhaElimD}}
\def\algids{{IDS}}

\allowdisplaybreaks

\begin{document}

\twocolumn[
\icmltitle{Optimal Batched Linear Bandits}

\icmlsetsymbol{equal}{*}

\begin{icmlauthorlist}
\icmlauthor{Xuanfei Ren}{ustc}
\icmlauthor{Tianyuan Jin}{nus}
\icmlauthor{Pan Xu}{duke}
\end{icmlauthorlist}

\icmlaffiliation{ustc}{University of Science and Technology of China}
\icmlaffiliation{nus}{National University of Singapore}
\icmlaffiliation{duke}{Duke University}

\icmlcorrespondingauthor{Pan Xu}{pan.xu@duke.edu}
\icmlkeywords{Machine Learning, ICML}

\vskip 0.3in
]

\printAffiliationsAndNotice{}  %

\begin{abstract}
We introduce the E$^4$  algorithm for the batched linear bandit problem, incorporating an Explore-Estimate-Eliminate-Exploit framework. With a proper choice of exploration rate, we prove E$^4$ achieves the finite-time minimax optimal regret with only $O(\log\log T)$ batches, and the asymptotically optimal regret with only $3$ batches as $T\rightarrow\infty$, where $T$ is the time horizon. We further prove a lower bound on the batch complexity of linear contextual bandits showing that any asymptotically optimal algorithm must require at least $3$ batches in expectation as $T\rightarrow\infty$, which indicates E$^4$ achieves the asymptotic optimality in regret and batch complexity simultaneously. To the best of our knowledge, E$^4$ is the first algorithm for linear bandits that simultaneously achieves the minimax and asymptotic optimality in regret with the corresponding optimal batch complexities. In addition, we show that with another choice of exploration rate E$^4$ achieves an instance-dependent regret bound requiring at most $O(\log T)$ batches, and maintains the minimax optimality and asymptotic optimality. We conduct thorough experiments to evaluate our algorithm on randomly generated instances and the challenging \textit{End of Optimism} instances \citep{lattimore2017end} which were shown to be hard to learn for optimism based algorithms. Empirical results show that E$^4$ consistently outperforms  baseline algorithms with respect to regret minimization, batch complexity, and computational efficiency.
\end{abstract}
\section{Introduction}

Sequential decision-making problems including bandits and reinforcement learning \citep{bubeck2012regret,slivkins2019introduction,lattimore2020bandit} have become pivotal in modeling decision processes where an agent actively interacts with an uncertain environment to achieve a long-term goal. In particular, we study the linear contextual bandit problems \citep{langford2007epoch,li2010contextual,abbasi2011improved,chu2011contextual,agrawal2013thompson,kirschner2021asymptotically,xu2022langevin}, where the agent can choose from $K$ arms to play at each step, and each arm is associated with a feature vector in the $\RR^d$ space. Each arm when played emits a noisy reward with its mean assumed to be a linear function of the feature vector with the unknown weight parameter shared across different arms. These bandit problems are prevalent in numerous real-world applications, from online advertising \citep{abe2003reinforcement} to recommendation systems \citep{agarwal2008online,li2010contextual}, where abundant side information is available and thus provides powerful feature representation of each arm. The most critical dilemma in designing algorithms for learning bandit problems is the trade-off between exploitation and exploration, where the agent must decide whether to exploit the currently known best option or to explore unknown possibly suboptimal options for maximizing the long-run gains. 

In conventional bandit setting, the agent adjusts its strategy or policy step-by-step, by first choosing an arm to play, observing a reward for this arm, and adjusting its strategy accordingly in the immediate next step. This ideal setting is also referred to as the fully-sequential setting, where immediate outcome can be obtained and switch polices  is cost-efficient. In real-world problems, fully sequential learning strategies are usually expensive or infeasible in critical applications where the interaction with the environment takes months to observe the rewards like in clinical trials or medical treatment \citep{robbins1952some}, or switching between different policies is expensive in e-commerce \citep{bertsimas2007learning}. Therefore, 
the batched bandit setting is more realistic and feasible, where an agent selects a batch of actions rather than a single action at one decision point to play, which significantly minimizes the number of policy switching and expedites the learning process by enabling parallel experiments \citep{perchet2016batched,gao2019batched,esfandiari2021regret,han2020sequential,jin2021double,jin2021almost,ruan2021linear,jin2023optimal}. 

To evaluate the effectiveness of bandit algorithms, a crucial metric is the \textit{regret}, which represents the gap between the expected cumulative rewards of the learning agent and those of an oracle who knows the optimal action at each step in hindsight. We denote the regret 
as $R_T$, where $T$ is the learning time horizon. The goal of a bandit algorithm is to find a tight bound for the regret to guarantee good performance theoretically.
Generally speaking, there are three types of bounds of regret that are considered in the bandit literature. The first type of regret bound is called the worst-case regret bound where we measure the performance of a bandit algorithm with respect to the worst possible bandit instance. For linear bandits with $K$ arms, it is well-established \citep{lattimore2020bandit} %
that there exists a bandit instance such that, for all bandit strategies denoted by $\pi$, the regret satisfies $R_T\geq \Omega(\sqrt{dT})$. %
The second type of regret is called the instance-dependent regret bound, where the regret $R_T$ is bounded by problem dependent parameters including the time horizon, the mean rewards of different arms, the number of arms, etc. Instance-dependent regret bounds provide performance evaluation of a bandit algorithm specific to certain bandit instances, offering more delicate insights into an algorithm's behavior in practice. %
The above regret bounds are both in the finite-time regime, where we fixed the time horizon $T$. 
The third type of regret bound, the asymptotic regret bound, characterizes the performance of an algorithm when $T$ goes infinite. It has been proved that the regret of any consistent algorithm satisfies $\liminf_{T\rightarrow\infty}R_T/\log T\geq c^*$, where $c^*:=c^*(\btheta^*)$ is a statistical complexity measure of the problem under the bandit instance  $\btheta^*$ \citep{graves1997asymptotically,combes2017minimal,lattimore2017end}, which we provide a formal definition in \Cref{section-preliminary}. %

Achieving the optimal regret across various regret types poses a significant challenge. Notably, instance-dependent analysis proves to be more intricate than worst-case analysis. Many algorithms, while considered minimax optimal, may encounter difficulties in specific instances \citep{lattimore2017end}. Conversely, certain existing instance-optimal algorithms lack a guarantee for worst-case regret \citep{wagenmaker2023instance}.
It is even more challenging to achieve optimal regret for batched bandit algorithms. Due to their batch design, learning agents cannot update the exploration policy after each data point, resulting in a rarely-switching nature. 
In this paper, we show that we can achieve the same order of all three types of regret as full-sequential algorithm but with much fewer batches. In particular, we first provide a lower bound on the batch complexity for asymptotic algorithms, and then propose an algorithm that attains both non-asymptotic and asymptotic optimal regret simultaneously and attains the corresponding optimal batch complexities. At the key of our algorithm is a careful design of the Explore-Estimate-Eliminate-Exploit framework leveraging D-optimal design exploration, optimal allocation learning, and arm elimination. %

\begin{table*}[th] 
\caption{Regret and batch complexity comparison in the finite-time and asymptotic settings for linear bandits. \label{table:related-work-comparison}}
\resizebox{\textwidth}{!}{
\begin{sc}
\begin{small}
\begin{tabular}{lcccc}
\toprule
\multirow{2}{*}{Algorithm} & \multicolumn{2}{c}{Non-asymptotic setting} & \multicolumn{2}{c}{Asymptotic setting} \\ \cmidrule(r){2-3} \cmidrule(l){4-5} 
                          & Worst-case regret       & Batch complexity      & Asymptotic regret        & Batch complexity       \\ \midrule
                   \citet{abbasi2011improved}       &    $\tilde O(d\sqrt{T})$           &    $O(\log T)$                    &         -      &       -                \\
                   \citet{esfandiari2021regret}       &    $\tilde O(\sqrt{dT})$           &    $O(\log T)$                    &         -      &        -               \\ 
                   \citet{ruan2021linear}       &    $\tilde O(\sqrt{dT})$           &    $O(\log\log T)$                    &         -      &        -               \\ 
                    \citet{hanna2023contexts}       &    $\tilde O(\sqrt{dT})$           &    $O(\log\log T)$                    &         -      &        -               \\ 
                   Lower bound \citep{gao2019batched}      &     $\Omega(\sqrt{dT})$         &     $\Omega(\log\log T)$                        &           -    &           -             \\ \midrule
                 \citet{lattimore2017end}          &       -             &         -          &   Optimal            &          Sequential              \\ 
                   OSSB\ \citep{combes2017minimal}        &      -       &     -               & Optimal           &      Sequential                  \\ 
                    OAM\ \citep{hao2020adaptive}      &      -              &      -               &       Optimal        &             Sequential           \\ SOLID\ \citep{tirinzoni2020asymptotically}      &        $\tilde O\big((d+\log K)\sqrt{T}\big)$            &          Sequential               &   Optimal            &               Sequential           \\
                     IDS\ \citep{kirschner2021asymptotically}      &      $\tilde O(d\sqrt{T})$              &         $\geq O\big(d^4\log^4T/\Delta_{\min}^2\big)$

                     &    Optimal           &            $\geq O\big(\log^4T\big)$          \\ \midrule
                     Lower bound (\Cref{them:batch-lower-bound-adaptive-batch})   &     -       &     -         &       Optimal                  &         $\geq 3$                          \\
                       \algname\ (\Cref{algorithm:optimal-algorithm})       &      $\tilde O(\sqrt{dT})$        &     $O(\log\log T)$                       &       Optimal                  &         $3$                          \\\bottomrule
\end{tabular}
\end{small}
\end{sc}}
\end{table*}

\noindent
\textbf{Our contributions} are summarized as follows.
\begin{itemize}[nosep,leftmargin=*]
    \item We propose the algorithm, {Explore-Estimate-Eliminate-Exploit} (\algname), for solving batched linear contextual bandits. With a proper choice of exploration rate, we prove that \algname\ attains $\tilde O(\sqrt{dT})$ worst-case regret after only $O(\log\log T)$ batches, which matches the existing lower bounds on the regret and batch complexity for linear contextual bandits with finite arms \citep{gao2019batched}. %
    Under the same configuration, when $T\rightarrow\infty$, \algname\ achieves the asymptotically optimal regret with only $3$ batches. We also prove that any asymptotically optimal algorithm must have at least $3$ batches in expectation as $T\rightarrow\infty$. Therefore, \algname\ is the first algorithm in linear bandits that achieves simultaneously optimal in terms of regret bound and batch complexity in both the finite-time worst-case setting and the asymptotic setting. We present the comparison of regret and batch complexities with existing work in \Cref{table:related-work-comparison} for the readers' reference.

    \item With a different exploration rate, we prove that our \algname\ algorithm attains an instance-dependent regret bound $O(d\log(KT)/\Delta_{\min})$ with at most $O(\log T)$ batches, where $\Delta_{\min}$ is the minimum gap between the mean rewards of the best arm and the suboptimal arms. Under the same configuration, \algname\ also maintains the minimax optimality and asymptotic optimality. Our result nearly matches the batch complexity lower bound \citep{gao2019batched} that states any algorithm with $O(d\log T/\Delta_{\min})$ regret must have at least $\Omega(\log T)$ batches.

    \item We conduct experiments on challenging linear bandit problems, including the \textit{End of Optimism} instances \citep{lattimore2017end}, where it has been demonstrated that algorithms based on optimism are suboptimal, as well as on a suite of randomly generated instances. Our empirical evaluation verifies that \algname\ not only outperforms the existing baselines in terms of regret and batch complexity but attains superior computational efficiency by exhibiting a significant speedup over baseline algorithms.%
\end{itemize}

\paragraph{Notation} We denote a set $\{1,\ldots,K\}$ as $[K]$, $K\in \NN^+$. We use bold letter $\xb\in\RR^d$ to denote a vector. $\Vert\xb\Vert_2=\sqrt{\xb^\top \xb}$ is its Euclidean norm. For a semipositive definite matrix $\Vb\in\RR^{d\times d}$, $\Vert \xb\Vert_{\Vb}=\sqrt{\xb^\top \Vb \xb}$ is the Mahalanobis norm.
For an event $E$ within a probability space, we denote its complement as $E^c$. For $f(T)$ and $g(T)$ as functions of $T$, we write $g(T)=O(f(T))$ to imply that there is a constant $C$ independent of $T$ such that $g(T)\leq Cf(T)$, and use $\tilde{O}(f(T))$ to omit logarithmic dependencies. We also define $f(T)=\Omega(g(T))$ if $g(T)=O(f(T))$. When both $f(T)=O(g(T))$ and $g(T)=O(f(T))$ hold, we write $f(T)=\Theta(g(T))$. We assume $0\times \infty=0$.
\section{Related Work}\label{section-related-work}

Existing works on batched linear bandits predominantly concentrated on achieving non-asymptotic optimal regret with minimal batch complexity. \citet{gao2019batched} argued that, for an algorithm to achieve minimax optimality, it should incorporate at least $\Theta(\log\log T)$ batches. Additionally, to attain an instance-dependent regret of order $O(d\log T/\Delta_{\min})$, an algorithm should include at least $\Theta(\log T)$ batches. In the contextual linear bandits setting, \citet{han2020sequential, ruan2021linear,zhang2021almost,hanna2023contexts} delved into linear bandits with both stochastic and adversarial contexts. However, the techniques proposed by \citet{han2020sequential} cannot be directly applied to our fixed $K$ arms setting, as it assumes that each arm is i.i.d. sampled from some distribution. \citet{ruan2021linear,zhang2021almost,hanna2023contexts} researched on stochastic contextual linear bandit problems, and our setting can be seen as a special case of theirs when the number of contexts is one, meaning the arm set is fixed. They both achieved $O(d\sqrt{T})$ regret bound with $O(\log\log T)$ batches. %
In the context of linear bandits with fixed $K$ arms, a simple algorithm with $O(\log T)$ batch complexity achieves an optimal regret bound of $\tilde O(\sqrt{dT})$ \citep{esfandiari2021regret, lattimore2020bandit}. However, this falls short of the optimality compared to the lower bound in \citet{gao2019batched}.

Another related line of research is the algorithms that achieve asymptotically optimal regret for linear bandits
\citep{lattimore2017end, combes2017minimal, hao2020adaptive, tirinzoni2020asymptotically, kirschner2021asymptotically, wagenmaker2023instance}, which strive for the best possible performance in a sufficiently large horizon. Nevertheless, traditional well-performing optimistic algorithms like Upper Confidence Bound (UCB) and Thompson Sampling (TS) have been proven to fall short of achieving asymptotic optimality \citep{lattimore2017end}. %
\citet{lattimore2017end} demonstrated an asymptotic lower bound for the regret of any consistent algorithm and proposed a simple algorithm with a matching upper bound, but no finite-time performance was guaranteed. Similarly, \citet{combes2017minimal} achieved asymptotic optimality for a large class of bandit problems, but no finite-time regret guarantees were provided either.
  \citet{hao2020adaptive} achieved asymptotic optimality with good empirical performance within finite horizon, considering a contextual bandit setting with a finite number of contexts, sharing similarities with our setting. They also proved that an optimistic algorithm could have bounded regret when the set of optimal arms spans the arm space, aligning with the concept of the \textit{End of Optimism} phenomenon \citep{lattimore2017end}. 

Further, \citet{tirinzoni2020asymptotically} provided the first asymptotically optimal regret bound with non-asymptotic instance-dependent and worst-case regret bounds, achieving the minimax optimality. \citet{kirschner2021asymptotically} proposed the frequentist IDS algorithm that leverages an information theoretic method, achieving minimax and asymptotic optimality simultaneously with a non-asymptotic instance-dependent regret bound. \citet{wagenmaker2023instance} proposed a method for achieving instance optimality in both asymptotic and non-asymptotic settings, along with a unifying framework for non-asymptotic instance optimality. Even so, they fell short of achieving minimax optimality. 
More importantly, all of the mentioned works that focus on asymptotic regret analysis are designed in a fully sequential way and cannot be applied to the batched linear bandit setting.

\section{Preliminary} \label{section-preliminary}
In this paper, we study the linear contextual bandit problem, where the reward of an arm $\xb\in\cX\subset\RR^d$ is given by $r(\xb)=\la \xb,\btheta^*\ra+\varepsilon$. Here $\xb$ is the arm feature, $\cX$ is the arm set with $|\cX|=K$ arms in total, $\btheta^*\in\RR^d$ is an unknown weight parameter, and $\varepsilon$ is a zero-mean $1$-subgaussian noise. For simplicity, we assume $\cX$ spans $\RR^d$. Otherwise, we can always find a corresponding arm set with a lower dimension that does by applying an orthogonal transformation to both $\btheta^*$ and actions.
For normalization purposes, we assume bounded expected rewards: $\vert\la \xb,\btheta^*\ra\vert \leq L$, $\forall \xb\in\cX$, where $L>0$ is some constant. %

Conventional bandit algorithms operate in a fully sequential way. In each round $t=1,\ldots, T$, the learner chooses an action $\xb_t\in\cX$ and then observes a reward $r_t=r(\xb_t)$. The objective is to minimize the \textit{accumulated regret} of the algorithm defined as $R_T=\EE[\max_{\xb\in\cX}\sum_{t=1}^T\la \xb-\xb_t,\btheta^*\ra]$.
In the batched setting, the algorithm pulls a batch of $b_\ell$ arms at round $\ell$. For $T$ total arms played, we can equivalently rewrite the regret as $R_T=\EE[\max_{\xb\in\cX}\sum_{\ell=1}^{B}\sum_{j=1}^{b_\ell}\la \xb-\xb_{\ell,j},\btheta^*\ra]$.
where $b_\ell$ is the batch size of batch $\ell$, $B$ is the total number of batches, and $\sum_{\ell=1}^B b_\ell=T$. Note that the two definitions are exactly the same by re-indexing the arms pulled in each batch.

We assume the best arm  $\xb^*=\argmax_{\xb\in\cX}\la \xb,\btheta^*\ra$ is unique. We define $\Delta_{\xb}=\la \xb^*-\xb,\btheta^*\ra$ to be the suboptimality gap for an arm $\xb\neq \xb^*$, and $\Delta_{\min}=\min_{\xb\neq \xb^*}\Delta_{\xb}$. We also define $\mu_{\xb}=\la \xb,\btheta^*\ra$ to be the mean reward for pulling arm $\xb$. For a $K$-dimensional non-negative vector $\balpha\in\RR^K_{\geq 0}$, we use $\balpha=\{\alpha_{\xb}\}_{\xb\in \cX}$ to denote an allocation of arm pulls over arm set $\cX$.

\paragraph{Asymptotic lower bound} Asymptotic lower bounds are usually defined for consistent policies \citep{lattimore2020bandit}. A policy $\pi$ is called \textit{consistent} \citep{lattimore2017end} if its regret is sub-polynomial for any bandit instance, i.e., $R_T=o(T^p)$ for all $\btheta^*$ and $p>0$.

For a linear bandit instance with arm set $\cX$ and weight $\btheta^*$, define $\Hb_\alpha=\sum_{  \xb\in\cX}\alpha_\xb  \xb  \xb^\top $, where $\balpha\in\RR^K_{\geq 0}$ is an allocation for the number of pulls of each arm.\footnote{For the ease of presentation, we may overload the notation of covariance matrix $\Hb$ a little bit in this paper. Specifically, for an integer $t\in\NN_+$, we define $\Hb_t=\sum_{s=1}^t\xb_s\xb_s^\top$. For a vector $\wb=\{w_{\xb}\}_{\xb\in\cX}\in\RR^K_{\geq 0}$, we define $\Hb_\wb=\sum_{\xb\in\cX}w_{\xb} \cdot \xb\xb^\top$.} $c^*(\btheta^*)$ is the solution to the following convex program,
{%
\begin{align}\label{eq:program}
\begin{split} 
    c^*(\btheta^*)\overset{\Delta}{=}&\inf_{\alpha\in\RR^K_{\geq0}}\sum_{  \xb\in\cX} \alpha_\xb \Delta_\xb\\ 
    &\text{s.t.} \quad\Vert  \xb-\xb^*\Vert^2_{\Hb^{-1}_\alpha}\leq\frac{\Delta_{  \xb}^2}{2}, \forall   \xb\in \cX^-,
\end{split}
\end{align}}%
where $\cX^-=\cX-\{\xb^*\}$. We write $c^*:=c^*(\btheta^*)$ when no ambiguity arises. 
\begin{lemma}[\citet{lattimore2017end}]\label{lemma:asym-lower-bound}
    Any consistent algorithm $\pi$ for the linear bandits setting with Gaussian noise has regret $R_T$ satisfying $\liminf_{T\rightarrow\infty}R_T/\log T\geq c^*$,
    where $c^*$ is given by \eqref{eq:program}.
\end{lemma}
We call an algorithm asymptotically optimal if it achieves $\lim_{T\rightarrow\infty}R_T/\log T= c^*$.

\paragraph{Least square estimators} 
For a dataset $\{(\xb_s,r_s)\}_{s=1}^t$, we define the following least squares estimators which will be used in our algorithm design. Let $\Hb=\sum_{s=1}^t\xb_s \xb_s^\top$, and
\begin{subequations}\label{eq:least_square}
\begin{align}
    \textstyle
    \hat\btheta    &=\Hb^{-1}{\textstyle\sum_{s=1}^t \xb_s r_s}, \label{eq:least_square_solution}\\
    \hat\mu_{\xb}(t)  &=\la\hat\btheta,\xb\ra, \label{eq:least_square_mean_reward}\\
    \hat \xb^*        &={\textstyle\argmax_{\xb\in\cX}\la \hat\btheta,\xb\ra},\\
    \hat\Delta_{\xb}    &=\la \hat\btheta,\hat \xb^*-\xb\ra.
\end{align}
\end{subequations}

\section{Optimal Batched Linear Bandits Algorithm}\label{section-algorithm}

\begin{algorithm}[th]
\caption{Explore, Estimate, Eliminate, and Exploit (\algname)\label{algorithm:optimal-algorithm}}
\textbf{Input:} arm set $\cX$, horizon $T$, $\alpha>0$, $\gamma>0$, D-optimal design allocation function $f(\cdot,\cdot,\cdot)$ defined in \Cref{definition:D-optimal-design}, exploration rate $\{T_1,T_2,\ldots\}$, elimination parameters $\{\varepsilon_1,\varepsilon_2,\ldots\}$.  \\
\textbf{Initialization:} set $\cA\leftarrow\cX$ and $t\leftarrow0$. %
\algrule 
     \quad~~~{\color{gray} Batch $\ell=1$}
    \begin{algorithmic}[1]
    \STATE {\it Exploration:}
    Pull arm $\xb\in\cA$ for $f(\xb,\cA,T_1)$ times.  %
    \STATE $t\leftarrow t+\sum_{\xb\in\cA}f(\xb,\cA,T_1)$.
    \STATE {\it Estimation:} Update least squares estimators $\hat\btheta, \hat \xb^*,\hat\Delta$  based on the current batch of data by \eqref{eq:least_square}.
    \STATE Calculate $\wb(\hat\Delta)=\{w_{\xb}\}_{\xb\in\cA}$ by \Cref{definition-tracking-proportion}.
    \algrule 
     {\color{gray} Batch $\ell=2$}\\
    \STATE  {\it Exploration:} Pull  arm $\xb\in\cA$ for $f(\xb,\cA,T_2)+n_{\xb}$ times, where
    $n_{\xb}=\min\big\{w_{\xb} \cdot \alpha\log T,(\log T)^{1+\gamma}\big\}$. %
    \\
    \STATE $b_2\leftarrow\sum_{\xb\in\cA}f(\xb,\cA,T_2)+n_{\xb}$, $t\leftarrow t+b_2$. \label{algline:batch-size}
    \STATE {\it Estimation:} Update least squares estimators $\hat\btheta, \hat \xb^*,\hat\Delta$  based on the current batch of data by \eqref{eq:least_square}.\\
    \STATE {\it Elimination:} \label{algline:stope-rule-elimination}\\
    \quad\textbf{if} stopping rule \eqref{eq:stopping-rule} holds\\
    \quad\qquad $\cA\leftarrow\{\hat \xb^*\}$.  \\
    \quad\textbf{end if}
    \algrule 
    {\color{gray} Batch $\ell\geq 3$}
    \STATE $\ell\leftarrow 3$.
        \WHILE{$t<T$ \textbf{and} $\vert\cA\vert>1$} 
            \STATE {\it Exploration:}
            Pull arm $\xb\in\cA$ for $f(\xb,\cA,T_\ell)$ times.  %
            \STATE {\it Estimation:} 
            Update least squares estimator $\hat\btheta$ based on the current batch of data by \eqref{eq:least_square}.\\
            \STATE {\it Elimination:}
                 $\cA\leftarrow\big\{\xb\in\cA: \max_{\yb\in\cA}\la \hat\btheta,\yb-\xb\ra\leq 2\varepsilon_\ell \big\}$.
            \STATE $\ell\leftarrow\ell+1$, $t\leftarrow t+\sum_{\xb\in\cA}f(\xb,\cA,T_\ell)$. \; 
       \ENDWHILE
    \STATE {\it Exploitation:} Pull arm $\xb\in\cA$ for $T-t$ times \label{algline:exploitation}. 
    \end{algorithmic}
\end{algorithm}

In this paper, we design a general algorithm framework, called \textbf{Explore-Estimate-Eliminate-Exploit} (\algname), which is presented in \Cref{algorithm:optimal-algorithm}.  The algorithm proceeds in a batched fashion.  

In each batch we use an \textit{Exploration} stage that explores the arm space in some way we would specify below, an \textit{Estimation} stage that calculate some statistics with least squares estimators, and an \textit{Elimination} stage that eliminates low-rewarding arms. 
Note that in a specific Estimation stage, we exclusively use data from that particular batch for the computation of estimators and statistics. This  ensures the independence, aligning with our theoretical objectives. In \Cref{algorithm:optimal-algorithm}, we maintain an active set, denoted by $\cA\subseteq\cX$, which only consists of arms that we are unsure about its suboptimality and need to further explore. %
If after some batch, there is only one arm in the active arm set $\cA$, the algorithm enters the \textit{Exploitation} stage. This means it identifies the best arm and would pull it until the total pulls reach $T$.

Let $\ell$ be the index of batches. We denote $b_\ell$ as the total number of arms pulled in the $\ell$-th batch, which is also called as the batch size. And we denote $T_\ell$ as the D-optimal design exploration rate in the $\ell$-th batch. In what follows, we elaborate the details in \Cref{algorithm:optimal-algorithm} for different batches.

\paragraph{Batch $\ell=1$:} \textit{Explore with D-optimal design and estimate the exploration allocation $\wb$ for the next batch.}

\underline{\textit{Exploration:}} In the first batch, the agent has no prior information on arms and thus explores all arms through the classical D-optimal design \citep{lattimore2020bandit}. %
Specifically, given an arm set $\cA$ and and a user defined exploration rate $M$, we use D-optimal design rule to decide the number of pulls of each arm $\xb$, denoted by $f(\xb,\cA,M)$, which is defined as follows.
\begin{definition}\label{definition:D-optimal-design} 
For an arm set $\cA$ and an exploration rate $M$, the D-optimal design allocation of arm $\xb\in\cA$ is given by $f(\xb,\cA,M)=\lceil 2\pi^*_{\xb} g(\bpi^*)M/d\rceil$, where $\bpi^*:=\bpi^*(\cA)$ is a probability measure over the action set defined by
    \begin{align}\label{eq:D-optimal-solving}
        \textstyle \bpi^*(\cA)=\argmin_{\bpi}\max_{\xb\in\cA}\Vert \xb\Vert_{\Hb_{\bpi}^{-1}}^2,
    \end{align}
and $g(\bpi)=\max_{\xb\in\cA} \Vert \xb\Vert_{\Hb_{\bpi}^{-1}}^2$.   
\end{definition}
Throughout our algorithm, we employ this D-optimal design exploration rule. Specifically, we set the exploration rate $M=T_\ell$ for the $\ell$-th batch.
\begin{remark}
    We present \Cref{definition:D-optimal-design} here just for simplicity. In both our theoretical analysis and experiments, we use the Frank–Wolfe algorithm (see Chapter 21.2 in \citet{lattimore2020bandit}) to approximately solve the optimization problem in \eqref{eq:D-optimal-solving} and get a near-optimal solution $\tilde \bpi$. See \Cref{sec:proof_upper_bound} for more details.

\end{remark}

\underline{\textit{Estimation:}} 
Then \Cref{algorithm:optimal-algorithm} calculates least squares estimators by \eqref{eq:least_square}. Based on these estimators, we solve an allocation $\wb=\{w_{\xb}\}_{\xb\in\cX}$ using the following definition, which will be used in the \textit{Exploration} stage of batch 2. 
\begin{definition}\label{definition-tracking-proportion}
    Let $\epsilon=1/\log\log T$, and $\wb=\{w_{\xb}\}_{\xb\in\cX}\in[0,\infty)^K$ be the solution to the  problem
    \begin{align*}
    \textstyle \min_{\wb\in [0,\infty)^K}& \textstyle\sum_{  \xb\neq \hat \xb^*} w_{\xb}(\hat\Delta_{\xb}-4\epsilon) \\ \text{s.t.} \quad\Vert \xb-\hat \xb^*\Vert^2_{{\Hb}^{-1}_\wb}&\leq(\hat\Delta_{\xb}-4\epsilon)^2/2, \quad\forall   \xb\neq\hat \xb^*,\\
    \text{and}\quad w_{\hat \xb^*}&\leq (\log T)^\gamma/\alpha,
\end{align*}
where $\hat\Delta$ and $\hat\xb^*$ are defined in \eqref{eq:least_square}, $\Hb_\wb$ is defined in \Cref{section-preliminary}, $\gamma\in(0,1)$ and $\alpha=(1+1/\log\log T)(1+d\log\log T/\log T)$.
\end{definition}
\begin{remark}
    Our \Cref{definition-tracking-proportion} is novel and different from literature. 
    Compared with the naive allocation solved from an optimization problem like \eqref{eq:program},
    we use $\hat \Delta_{\xb}-4\epsilon$ instead of  $\hat\Delta_{\xb}$, and add a constraint $w_{\hat \xb^*}\leq (\log T)^\gamma/\alpha$ in the program. From \Cref{lemma:convergence-proportion-w} we can see the solution from \Cref{definition-tracking-proportion} convergences to the solution of naive method as $T\rightarrow\infty$, so our modification tricks in \Cref{definition-tracking-proportion} doesn't influence the asymptotic regret analysis.  In fact, with high probability, estimated gaps based on data in the second batch of our algorithm are larger than those underestimated gaps $\hat \Delta_{\xb}-4\epsilon$. We elaborate this trick in our proof of \Cref{lemma:break-of-the-second-batch}, which shows that with high probability, the stopping rule \eqref{eq:stopping-rule} holds after the second batch. 
\end{remark}

\paragraph{Batch $\ell=2$:} \textit{Explore with D-optimal design and the allocation $\wb$ from \Cref{definition-tracking-proportion}, and eliminate all suboptimal arms based on Chernoff's stopping rule.}

\underline{\textit{Exploration:}} The exploration stage of batch $2$ requires playing arms according to two allocations:
(1) D-optimal design exploration specified by \Cref{definition:D-optimal-design}, and (2) optimal allocation specified by \Cref{definition-tracking-proportion}. More specifically, for the second part, we pull $\xb\in\cA$ for 
$\min\{ w_{\xb}\cdot \alpha\log T, (\log T)^{1+\gamma} \}$
times, where $\gamma\in(0,1)$ is an arbitrarily small constant and $\wb=\{w_\xb\}_{\xb\in\cX}$ is solved from \Cref{definition-tracking-proportion} based on the estimators from batch 1.

\underline{\textit{Estimation:}} We calculate the least squares estimators according to \eqref{eq:least_square} based on the current batch of data.

\underline{\textit{Elimination:}} We then test the Chernoff's stopping rule and only keeps the current best arm based on our \textit{Estimation} stage if it holds. This means we will eliminate all suboptimal arms and directly go to the Line \ref{algline:exploitation}, \textit{Exploitation} stage of \Cref{algorithm:optimal-algorithm}. In particular, the stopping rule is defined as 
\begin{align}\label{eq:stopping-rule}
     \textstyle \big\{ Z(b_2)\geq \beta(b_2,1/T) \text{ and } \sum_{s=1}^{b_2}{\xb}_s{\xb}_s^\top \geq c \Ib_d\big\},
\end{align}
where $b_2$ is the batch size of the second batch defined in Line \ref{algline:batch-size} of \Cref{algorithm:optimal-algorithm}, $c=\max_{\xb\in\cX}\Vert \xb\Vert_2^2$ and 
{%
\begin{align}
    Z(b_2)&=\min_{\xb\neq \hat \xb^*}\hat\Delta_{\xb}^2/(2\Vert  {\xb}-\hat \xb^*\Vert^2_{\Hb_{b_2}^{-1}}),\label{eq:stopping-statistic-Z-t}\\
    \beta(t,\delta)&=(1+1/\log\log T)\log\big((t\log\log T)^{\frac{d}{2}}/\delta\big).\label{eq:beta-definition}
\end{align}}%
If the stop rule \eqref{eq:stopping-rule} does not hold, \Cref{algorithm:optimal-algorithm} will enter the \textbf{while} loop and conduct more batches of exploration.

The use of Chernoff's stopping rule is inspired from best arm identification problems \citep{garivier2016optimal,jedra2020optimal,jin2023optimal}. Specifically, if \eqref{eq:stopping-rule} holds, we can prove that $\hat\xb^*$ is the true best arm with probability at least $1-1/T$. This further ensures that no more exploration is needed and the regret is small.

\paragraph{Batch $\ell\geq3$:} \textit{Perform phased elimination with D-optimal design exploration.}
When \Cref{algorithm:optimal-algorithm} does enter these batches, it means the previous two batches are not sufficient to identify the best arm and thus more exploration is needed. To this end, we adopt the phased elimination with the D-optimal design exploration \citep{lattimore2020bandit}. 

\underline{\textit{Exploration:}} We use D-optimal design exploration defined in \Cref{definition:D-optimal-design} with rate $T_\ell$ for the $\ell$-th batch.

\underline{\textit{Estimation:}} We calculate the least squares estimators according to \eqref{eq:least_square} based on the current batch of data.

\underline{\textit{Elimination:}} We eliminate arms whose estimated mean reward on the estimated parameters are smaller than that of the empirically best arm by a margin of $2\varepsilon_\ell$. The active arm set is updated as 
$\cA\leftarrow\{\xb\in\cA: \max_{\yb\in\cA}\la \hat\btheta,\yb-\xb\ra\leq 2\varepsilon_\ell \}$.

\paragraph{Final batch:} \textit{Pull the estimated best arm $\hat \xb^*$ until the end}. If the Chernoff's stopping rule \eqref{eq:stopping-rule} holds at the end of the second batch, or $\vert\cA\vert=1$ after some batch $\ell\geq 3$, \Cref{algorithm:optimal-algorithm} will enter Line \ref{algline:exploitation}, the \textit{Exploitation} stage. In this stage, the agent just commits to the estimated best arm $\hat \xb^*\in\cA$ until total pulls reach $t=T$. 

\section{Theoretical Analysis}\label{section-theoretical-analysis}

In this section, we provide theoretical analysis results on the regret optimality and batch complexity of \Cref{algorithm:optimal-algorithm}.

First we give a formal definition of \textit{batch}.
\begin{definition}
    For a linear  bandit problem with time horizon $T$, we say that the batch complexity of a learning algorithm
is (at most) $M$, if the learner decides $T_1$ and $\pi_1$ before the learning process starts,
and executes $\pi_1$ in the first $T_1$ time steps (which corresponds the first batch). 
Based on the data (context sets, played actions and the rewards) obtained from the first $T_1$ steps,
the learner then decides $T_2$ and $\pi_2$, and executes $\pi_2$ for $T_2$ time steps (the second batch). 
The learner repeats the process  for $M$ times/batches.  In general, at the beginning of the $k$-th batch, 
the learner decides $T_k$ (the size of the batch) and $\pi_k$ based on the data collected from the first $(k - 1)$ batches.
The batch sizes should satisfy that $\Sigma_{k=1}^M T_k=T$.
\end{definition}

We  present the following result on the batch complexity lower bound of asymptotically optimal algorithms.
\begin{theorem}\label{them:batch-lower-bound-adaptive-batch}
    If an algorithm achieves asymptotic optimality defined in \Cref{lemma:asym-lower-bound}, then on some bandit instances it must have at least $3$ batches in expectation as $T\rightarrow\infty$.
\end{theorem}

\begin{remark}
    In our proof detailed in \Cref{sec:proof_lower_bound}, we first show that any algorithm that performs at most 2 batches is not asymptotically optimal. Then we further show that even with randomly chosen batch sizes, the expected number of batches of an asymptotically optimal algorithm is at least $3$. This is the first batch complexity lower bound in the literature for asymptotically optimal algorithms. Though the lower bound in \Cref{them:batch-lower-bound-adaptive-batch} seems to be conservative, we will show in the next theorem that our proposed algorithm \algname\ achieves asymptotic optimality in regret and its this batch matches the lower bound in \Cref{them:batch-lower-bound-adaptive-batch}.
\end{remark}

Next we present the upper bounds on the regret and the batch complexity of  \Cref{algorithm:optimal-algorithm}.
\begin{theorem}\label{theorem:mainthem-optimal-algorithm}
In \Cref{algorithm:optimal-algorithm}, let the parameters be set to $\alpha=(1+1/\log\log T)(1+d\log\log T/\log T)$, $\delta=1/(KT^2)$,  $\gamma\in(0,1)$, and $\varepsilon_\ell=\sqrt{{d\log(1/\delta)}/{T_\ell}}$, $\forall \ell\geq 1$. Let the exploration rate be defined as
{%
\begin{align*}
    \cT_1=\{&T_1=(\log T)^{1/2},T_2=(\log T)^{1/2}, T_3=(\log T)^{1+\gamma},\\&T_\ell=T^{1-\frac{1}{2^{\ell-3}}}\ \text{for}\ \ell\geq 4 \}.
\end{align*}}%
Then the regret of \Cref{algorithm:optimal-algorithm} satisfies 
\begin{align*}
    R_T&=O\big(\log\log T\cdot\sqrt{dT\log(KT)}\big)=\tilde O\big(\sqrt{dT}\big),
\end{align*}
and only needs at most $O(\log\log T)$ batches.
As $T\rightarrow\infty$,
 \begin{align*}
 \limsup_{T\rightarrow\infty} \frac{R_T}{\log T  }\leq c^*, 
\end{align*}
only runs with $3$ batches in expectation. Furthermore, \Cref{algorithm:optimal-algorithm} runs with $3$ batches with probability at least $1-2/(\log T)^2$ as $T\rightarrow\infty$.
\end{theorem}

\begin{remark}
Our $\tilde O(\sqrt{dT})$ regret bound nearly matches the lower bound in \citet{lattimore2020bandit}, and the $O(\log\log T)$ batch complexity matches the lower bound of $\Omega(\log\log T)$ in \cite{gao2019batched}. Furthermore, according to \Cref{them:batch-lower-bound-adaptive-batch}, our algorithm also achieves the asymptotically optimal regret bound with the optimal batch complexity, i.e., \algname\ only need $3$ batches to achieve the asymptotic optimality. To the best of our knowledge, \algname\ is the only algorithm in the literature that is simultaneously optimal in terms of regret bound and batch complexity in both the finite-time worst-case setting and the asymptotic setting.

In contrast, existing linear bandit algorithms can be categorized as follows: (1) achieving the minimax optimal regret bound but with $O(\log T)$ batch complexity \citep{esfandiari2021regret}; (2) achieving the minimax optimal regret bound with optimal batch complexity but no asymptotic guarantees  \citep{ruan2021linear,hanna2023contexts}; (3) achieving the asymptotically optimal regret but no finite time minimax optimality \citep{lattimore2017end,combes2017minimal,hao2020adaptive,wagenmaker2023instance}; or (4) achieving the asymptotically optimal regret and the minimax optimality  but in a fully sequential way \citep{tirinzoni2020asymptotically}. 
Compared to these works, Frequentist Information-Directed Sampling (\algids) \citep{kirschner2021asymptotically} achieves both minimax and asymptotic optimality. Furthermore, it
   can be proved to have a batch complexity slightly larger than $O(\log^4 T)$\footnote{IDS is not a batched algorithm, but it has a rarely-switching structure and thus can be easily adapted into an batched algorithm. By some calculations, their batch complexity is larger than $ O(d^4\log^4T/\Delta_{\min}^2)$, which we also include in \Cref{table:related-work-comparison}.}, which is still much higher than the batch complexity of our algorithm \algname. Moreover, the IDS algorithm can only achieve a worst case regret bound $\tilde O(d\sqrt{T})$, which is suboptimal for $K$-armed linear bandits when $K$ is fixed. %
\end{remark}

\begin{remark}\label{remark:two-batch-sizes-same-performance}
    Lastly, we would like to highlight that when $T$ is large enough, the stopping rule \eqref{eq:stopping-rule} used in Line \ref{algline:stope-rule-elimination} of \Cref{algorithm:optimal-algorithm} holds with  probability at least $1-2/(\log T)^2$. When this happens, our algorithm would skip the \textbf{while} loop and directly goes to the final exploitation batch. Therefore, the algorithm will only need $3$ batches to achieve the asymptotic optimality. Notably, since the stopping rule in Line \ref{algline:stope-rule-elimination} only depends on the first two batches of \Cref{algorithm:optimal-algorithm}, different choices of exploration rates in later do not change the asymptotic behavior of \Cref{algorithm:optimal-algorithm} when $T\rightarrow\infty$, as we will see in the next theorem.
\end{remark}

In the following theorem, we show that with a slightly different choice of exploration rates, \Cref{algorithm:optimal-algorithm} can  achieve an instance-dependent regret bound with the optimal batch complexity as well.
\begin{theorem}\label{theorem:instance-optimal-algorithm}
Let the exploration rate in \Cref{algorithm:optimal-algorithm} be chosen as
{%
\begin{align*}
    \cT_2=\{&T_1=(\log T)^{1/2},T_2=(\log T)^{1/2},T_3=(\log T)^{1+\gamma},\\&T_\ell=d\log(KT^2)\cdot 2^{\ell-3}\ \text{for}\ \ell\geq 4\},
\end{align*}}%
and all other parameters remain the same as in \Cref{theorem:mainthem-optimal-algorithm}. Then \Cref{algorithm:optimal-algorithm} satisfies 
$R_T= O(\sqrt{dT}\cdot \log(KT))=\tilde O(\sqrt{dT})$,
and
\begin{align*}
    R_T&= O((\log T)^{1+\gamma}+{d\log(KT)}/{\Delta_{\min}})%
\end{align*}
with at most $O(\log T)$ batches  and in expectation $O(\log(1/\Delta_{\min}))$ batches. %
Furthermore, the regret of \Cref{algorithm:optimal-algorithm} satisfies $\limsup_{T\rightarrow\infty} {R_T}/{\log T  }\leq c^*$
with $3$ batches in expectation. %
\end{theorem}
\begin{remark}
\Cref{theorem:instance-optimal-algorithm} shows that \algname\ can achieve the minimax optimal regret and the instance-dependent regret bound $O(d\log T/\Delta_{\min})$ with at most $O(\log T)$ batches. \citet{gao2019batched} proved that any algorithm with $O(d\log T/\Delta_{\min})$ regret  have at least $\Omega(\log T/\log\log T)$ batches. Hence, in this context, \algname\ attains optimal batch complexity while achieving this instance-dependent regret bound. Notably, the instance-dependent bound gives a tighter regret bound within the class of bandit instances where $\Delta_{\min}\geq \sqrt{d/T}$ than the worst-case regret bound.  
\end{remark}
\begin{remark}
   Notably, our instance-dependent batch complexity $O\big(\log(1/\Delta_{\min})\big)$  is novel and doesn't depend on time horizon $T$, which shows that for instances with large reward gaps the batch complexity of our \Cref{algorithm:optimal-algorithm} is a small constant and does not increase as the horizon $T$ increases.
\end{remark}

\section{Experiments}\label{section-experiments}
In this section, we conduct experiments on challenging linear bandit instances, assessing the performance of our \algname\ algorithm in comparison to several baseline algorithms. In particular, we first do simulations on hard linear bandit instances we constructed inspired from the End of Optimism instance \citep{lattimore2017end} on which optimistic algorithms such as  UCB and TS are proved to be suboptimal as $T\rightarrow\infty$. We then conduct experiments on more general random instances.  The implementation could be found at \href{https://github.com/panxulab/optimal-batched-linear-bandits}{https://github.com/panxulab/optimal-batched-linear-bandits}.

\textbf{Baselines:} We compare \algname (\Cref{algorithm:optimal-algorithm}) with Rarely Switching
OFUL (denoted by \algoful) \citep{abbasi2011improved}, the Optimal Algorithm (denoted by \algendoa) in \citet{lattimore2017end}, \algids\  \citep{kirschner2021asymptotically}, and Phased Elimination Algorithm with D-optimal design (denoted by \algped)  \citep{esfandiari2021regret,lattimore2020bandit}. Note that \algids\  was shown to outperform other asymptotically optimal algorithms \citep{lattimore2017end, combes2017minimal,hao2020adaptive,tirinzoni2020asymptotically}, and thus we do not include the rest of them in our experiments. %
\begin{figure}[!htbp]
    \centering
    \includegraphics[scale=0.5]{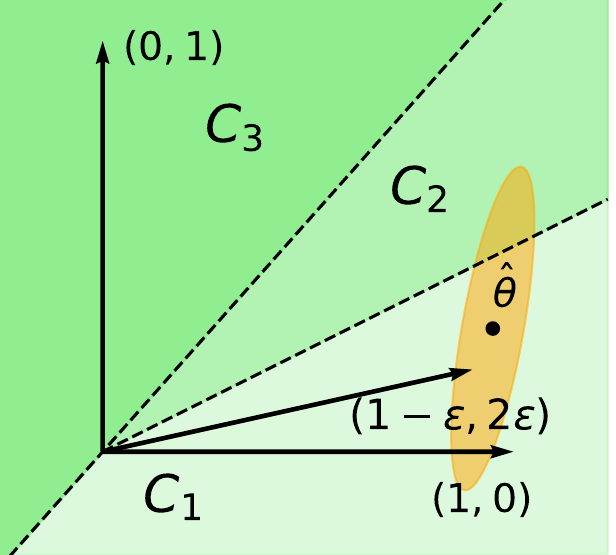}
    \caption{The \textit{End of Optimism} instance in $\RR^2$. The true parameter $\btheta^*$ is $(1,0)$. The arms are $\xb_1=(1,0),\xb_2=(0,1),\xb_3=(1-\varepsilon,2\varepsilon)$. Note that $\xb_i$ is the best arm if $\btheta^*$ lies in the colored region $C_i$, $i=1,2,3$.
    }\label{fig:end-of-opt-instance}
\end{figure}

\begin{figure*}[!htbp]
    \centering
    \subfigure[$d=5,K=9,\epsilon=0.01$]{\includegraphics[scale=0.208]{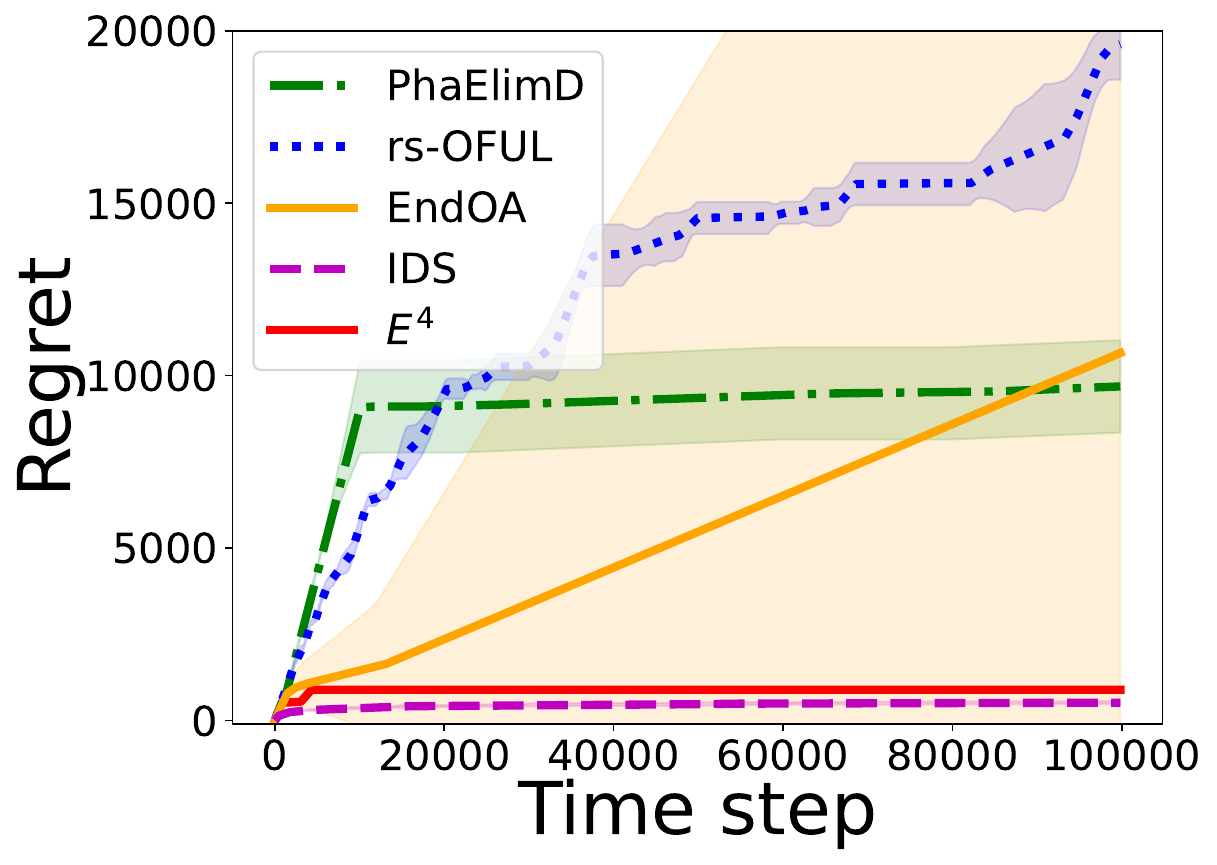}}
    \subfigure[$d=5,K=9,\epsilon=0.01$]{\includegraphics[scale=0.208]{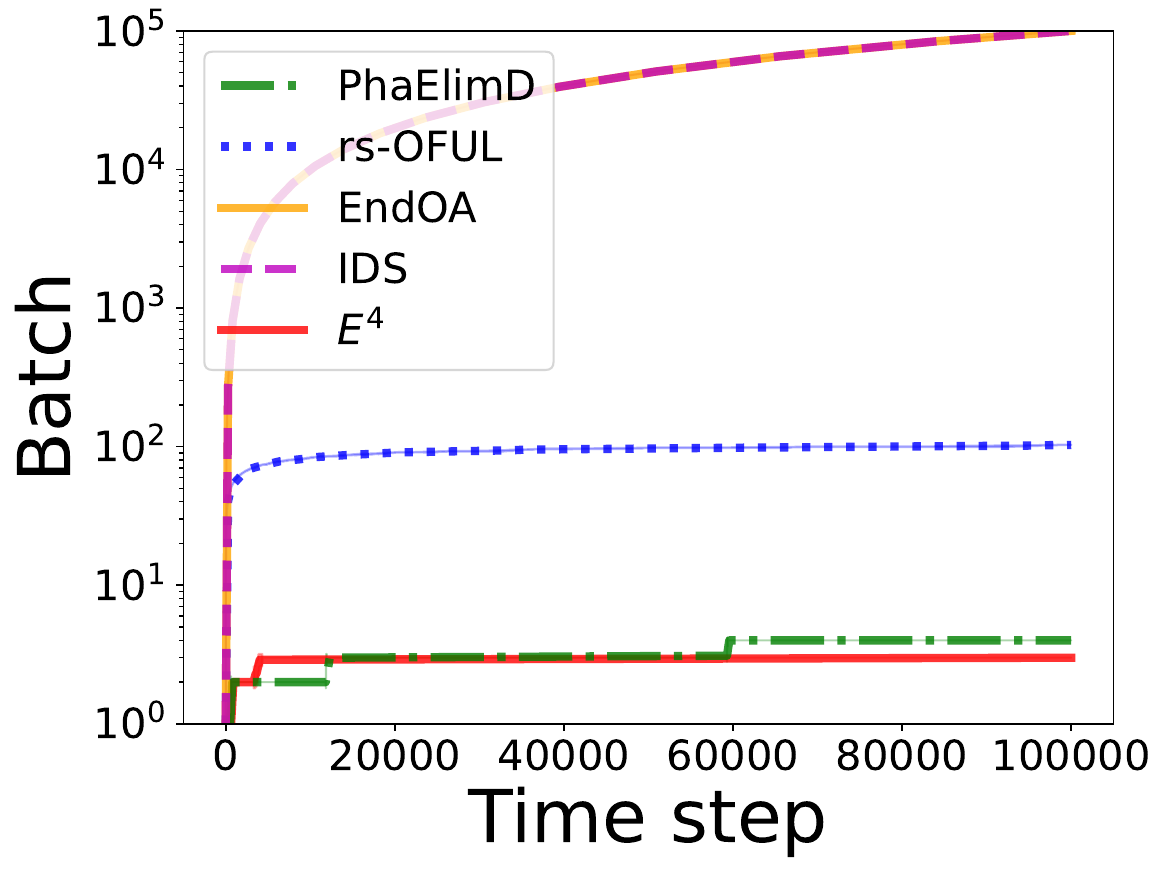}}
    \subfigure[$d=5,K=9,\epsilon=0.2$]{\includegraphics[scale=0.208]{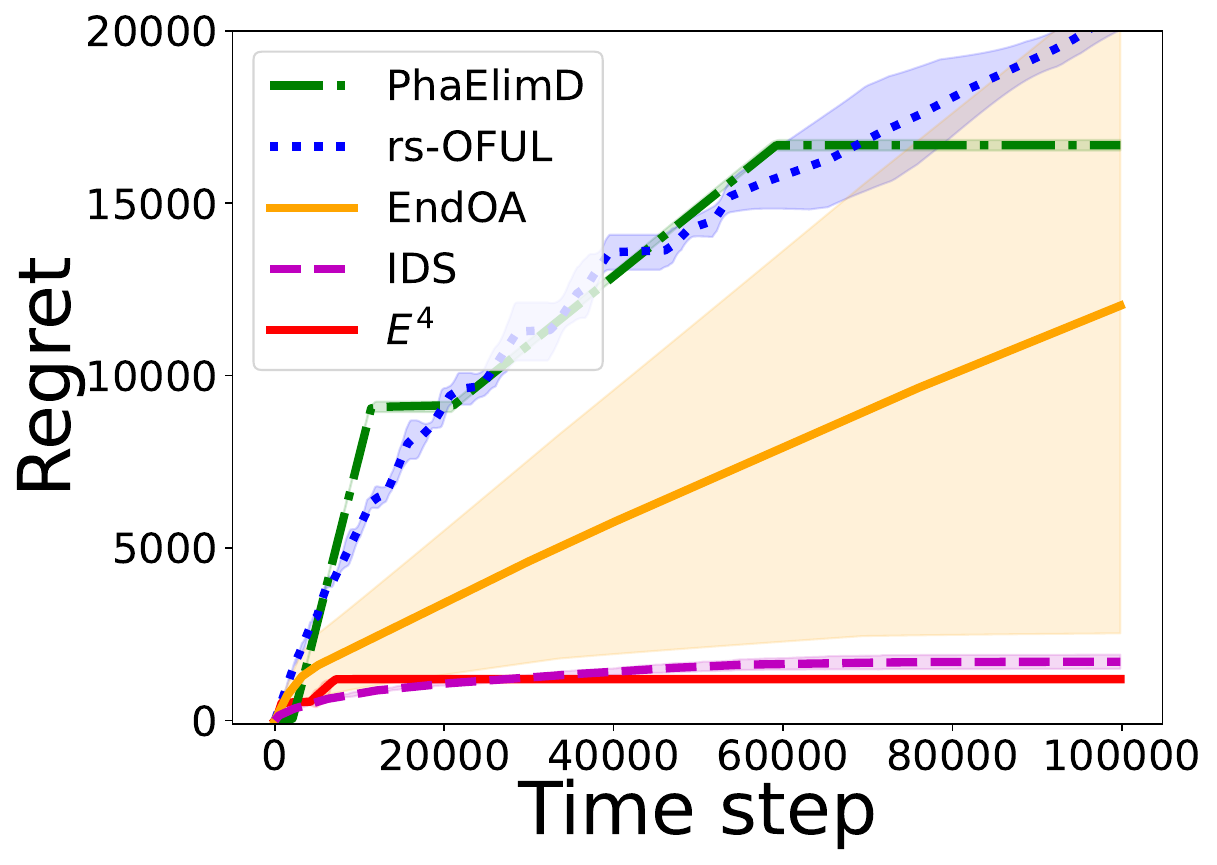}}
    \subfigure[$d=5,K=9,\epsilon=0.2$]{\includegraphics[scale=0.208]{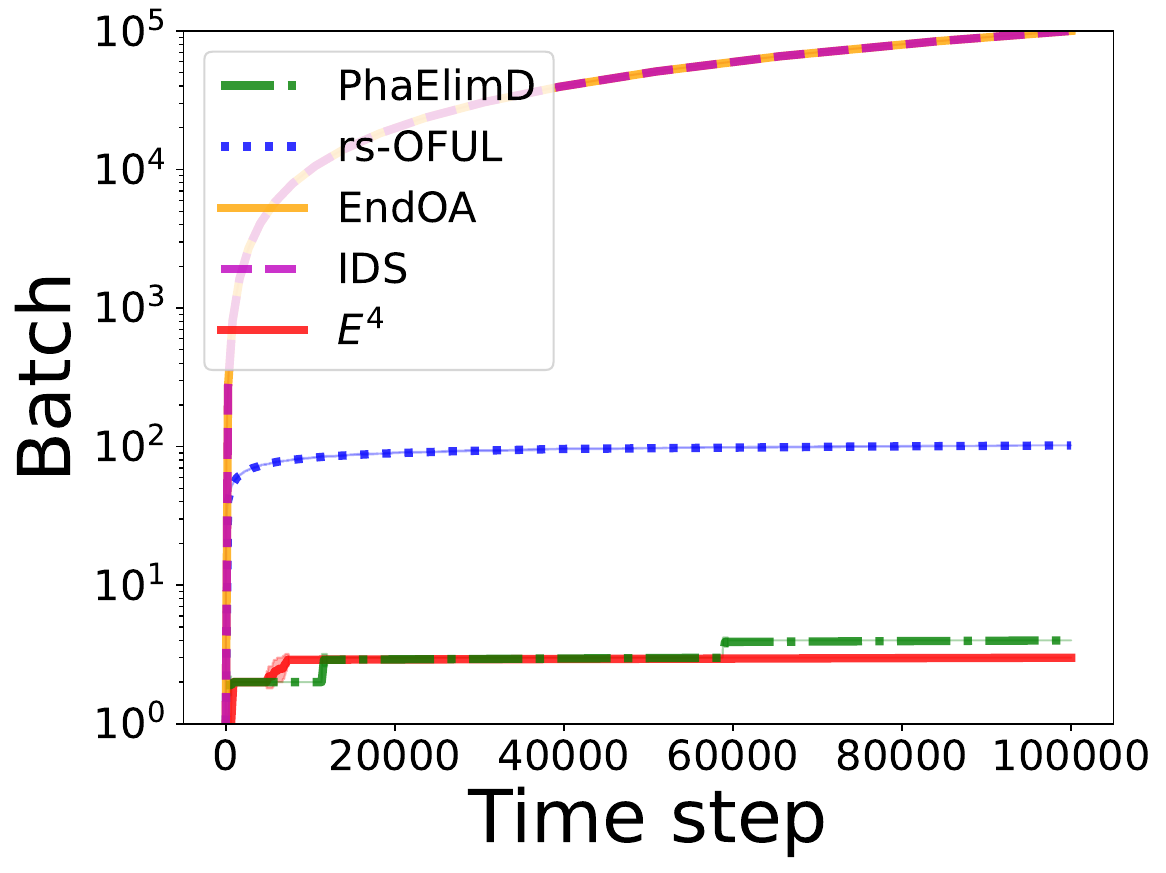}}
    \caption{Regret and Batch Analysis: \textit{End of Optimism} instances ($d=5,K=9$).
    }
    \label{fig:simulation-end-d-5}
\end{figure*}

\textbf{\textit{End of Optimism} Instances: } 
\citet{lattimore2017end}  designed a hard instance in $\RR^2$ such that optimism based algorithms cannot achieve the asymptotic optimality. For completeness, we present the original \textit{End of Optimism} instance in $\RR^2$ in \Cref{fig:end-of-opt-instance}. Following a similar idea, we design several hard instances in more complicated settings to evaluate our proposed algorithm. Let $\eb_j^d$ denote the natural basis vector in $\RR^d$ where the $j$-th variable is $1$ and all others are $0$. Then we construct the following hard instances inspired by the End of Optimism \citep{lattimore2017end}. Let $\btheta^*=\eb_1^d$, the arm set $\cX$ defined as follows
\begin{align}\label{eq:EndOA-instances}
    \cX=\{\eb_i^d\}_{i=1}^d\cup
    \{(1-\epsilon)\eb_1^d+2\epsilon \eb_j^d\}_{j=2}^d,
\end{align}
where $\epsilon=0.01,0.2$ and the dimension $d=2,3,5$. Hence there are $6$ different hard instances in total. For each instance, the number of arms is $K=2d-1$. We conduct experiments on these instances and present the result for $d=5$ in \Cref{fig:simulation-end-d-2}. Due to space limit, we defer the results for $d=2$ and $d=3$ to \Cref{sec:additional-experiments-EndOA-instances}.

\textbf{Implementation:} %
Based on the weight parameter $\btheta^*$ and arm set $\cX$ defined in \eqref{eq:EndOA-instances}, we generate the noise reward for arm $\xb\in\cX$ as $r(\xb)=\la \xb,\btheta^*\ra+\varepsilon$, where $\varepsilon\sim \cN(0,1)$.
For the parameters in \algname\ (\Cref{algorithm:optimal-algorithm}), we fix the exploration rate to be $\cT_1$. We set the other parameters as described in \Cref{theorem:mainthem-optimal-algorithm}.
For \algoful, we implement the rarely switching OFUL algorithm in \citet{abbasi2011improved} with switching parameter $C=0.5$.
For \algendoa, we follow the  Warmup phase and Success phase exactly the same way as \citet{lattimore2017end} describes, and use OFUL in its Recovery phase.
For \algids, we follow their computationally efficient variant \citep[Section 2]{kirschner2021asymptotically} step by step. 
For \algped, we follow the framework in \citet{esfandiari2021regret}, and improve the batch sizes design by using exploration rate $T_i=T^{1-1/2^i}$ instead of $q^i$ in their paper for better performance.
For each instance, we repeat the experiment for each method for $10$ simulations and calculate the mean value and standard error of the cumulative regret and the batch complexity. %

Notably, we proposed two options for the exploration rates in \Cref{theorem:mainthem-optimal-algorithm,theorem:instance-optimal-algorithm}. However, as we discussed in \Cref{remark:two-batch-sizes-same-performance}, our algorithm skips the \textbf{while} loop with extremely high probability. Consequently, different choices of exploration rates would not affect the performance of the algorithm in practice. Therefore, we fix the exploration rate to  be $\cT_1$ defined in \Cref{section-theoretical-analysis} for the sake of convenience.

\textbf{Results:} We observe that \algname\ (\Cref{algorithm:optimal-algorithm}) achieves a similar regret as \algids\ and outperforms all other baseline algorithms, which demonstrates the effectiveness of our proposed algorithm. Moreover,
\algname\ only needs $3$ batches to achieve this optimal regret, which is in sharp contrast with \algids\ which needs more than $10^5$ batches to achieve the same regret. The only algorithm that enjoys a similar batch complexity as \algname\ is \algped, which only runs in around $4$ batches. Nevertheless, the regret of \algped\ is significantly higher than our \algname\ algorithm. As $\epsilon$ goes from $0.01$ to $0.2$, the regret of our algorithm and \algids\ remains optimal and very small in magnitude, while the performance of other algorithms becomes worse. 

\begin{table}[th]
    \centering
    \caption{Runtime (seconds) comparison.
    \label{table:runtime}}   
    \begin{sc}
    \begin{small}
    \begin{tabular}{llccccc}
    \toprule
    \multicolumn{1}{c}{$\epsilon$} & Algorithm  & $d=2$ & $d=3$ & $d=5$ \\
    \midrule
    \multirow{5}{*}{$0.01$} & \algped & $0.18$ & $0.71$ & $1.46$ \\
    & \algoful & $0.45$ & $1.47$ & $3.72$ \\
    & \algendoa & $3.15$ & $3.17$ & $8.94$ \\
    & \algids & $9.48$ & $30.22$ & $178.31$ \\
    & \algname & $0.04$ & $0.12$ & $0.25$ \\
    \midrule
    \multirow{5}{*}{$0.2$} 
    & \algped & $0.15$ & $0.76$ & $1.40$ \\
    & \algoful & $0.28$ & $1.60$ & $2.90$ \\
    & \algendoa & $2.23$ & $3.87$ & $10.19$ \\
    & \algids & $6.42$ & $31.24$ & $246.53$ \\
    & \algname & $0.06$ & $0.15$ & $0.33$ \\
    \bottomrule
    \end{tabular}
    \end{small}
    \end{sc}
    \label{tab:my_label}
\end{table}

\textbf{Runtime comparison:} We also present the averaged runtime for all the competing algorithms across $10$ independent trials, as displayed in \Cref{table:runtime}. In all instances regardless of the dimension or $\epsilon$, our algorithm \algname\ is consistently the most computationally efficient one, offering \textbf{a speedup ranging from $5$-fold to $1000$-fold} when compared to other baselines.

Due to the space limit, we defer more experimental results to \Cref{sec:additional_experiments}. In particular, we conducted an ablation study on the sensitivity of different algorithms' performance on the instance parameter $\epsilon$ in \Cref{sec:ablation-study-EndOA-instance}. We also conducted extra experiments on some randomly generated instances and present the results in \Cref{sec:experiments-random-instances}.

\section{Conclusion and Future Work}\label{section-conclusion}
In this paper, we proposed the Explore, Estimate, Eliminate and Exploit (\algname) algorithm, %
which only needs $3$ batches to achieve the asymptotically optimal regret in linear bandits. To the best of our knowledge, \algname\ is the first batched linear bandit algorithm that is simultaneously optimal in terms of regret bound and batch complexity in both the finite-time worst-case setting and the asymptotic setting. We conducted numerical experiments on challenging linear bandit instances, which unequivocally show that our method outperforms the current baseline methods in multiple critical aspects: it achieves the lowest level of regret, it requires the minimal batch complexity, and it enjoys superior running time efficiency. 

For future research, an intriguing objective is to investigate whether we can attain the optimal instance-dependent regret while maintaining the optimal batch complexity. While recent work by \citet{wagenmaker2023instance} achieves a specially defined instance-dependent optimality, their algorithms fail to achieve the minimax optimality, and more importantly, are fully sequential and thus cannot be applied to the batch setting. We hope our strategy proposed in this paper will provide a viable solution towards building an optimal instance-dependent algorithm with optimal batch complexity, which we leave for future work. 

\section*{Impact Statement}
This paper presents work whose goal is to advance the field of Machine Learning. There are many potential societal consequences of our work, none which we feel must be specifically highlighted here.

\section*{Acknowledgements}
We would like to thank the anonymous reviewers for their helpful comments. This research is supported by the Whitehead Scholars Program, by the US
National Science Foundation Award 2323112, by 
the National Research Foundation, Singapore under its AI Singapore Program (AISG Award No: AISG-PhD/2021-01004[T]), and by the Singapore Ministry of Education Academic Research Fund (AcRF) Tier 2 under grant number A-8000423-00-00. In particular, P. Xu was supported in part by the National Science Foundation (DMS-2323112) and the Whitehead Scholars Program at the Duke University School of Medicine. T. Jin was supported by 
the National Research Foundation, Singapore under its AI Singapore Program (AISG Award No: AISG-PhD/2021-01004[T]), and by the Singapore Ministry of Education Academic Research Fund (AcRF) Tier 2 under grant number A-8000423-00-00.  The views and conclusions in this paper are those of the authors and should not be interpreted as representing any funding agencies.

\bibliographystyle{icml2024}
\bibliography{reference}
\newpage
\appendix
\onecolumn

\section{Additional Experiments}\label{sec:additional_experiments}
In this section, we provide more experimental results.

\subsection{More Results on End of Optimism Instances}\label{sec:additional-experiments-EndOA-instances}

We provide more experimental results on instances defined in \eqref{eq:EndOA-instances} for $d=2$ and $3$ in \Cref{fig:simulation-end-d-2} and \Cref{fig:simulation-end-d-3} respectively. We also present the detailed batch complexities of algorithms in \Cref{table:batch-complexity-end}. These experiments demonstrate that the proposed algorithm \algname\ consistently outperforms baseline algorithms in terms of regret bound and batch complexity across different bandit instances. 

\begin{figure*}[!htbp]
    \centering
    \subfigure[$d=2,K=3,\epsilon=0.01$]{\includegraphics[scale=0.2]{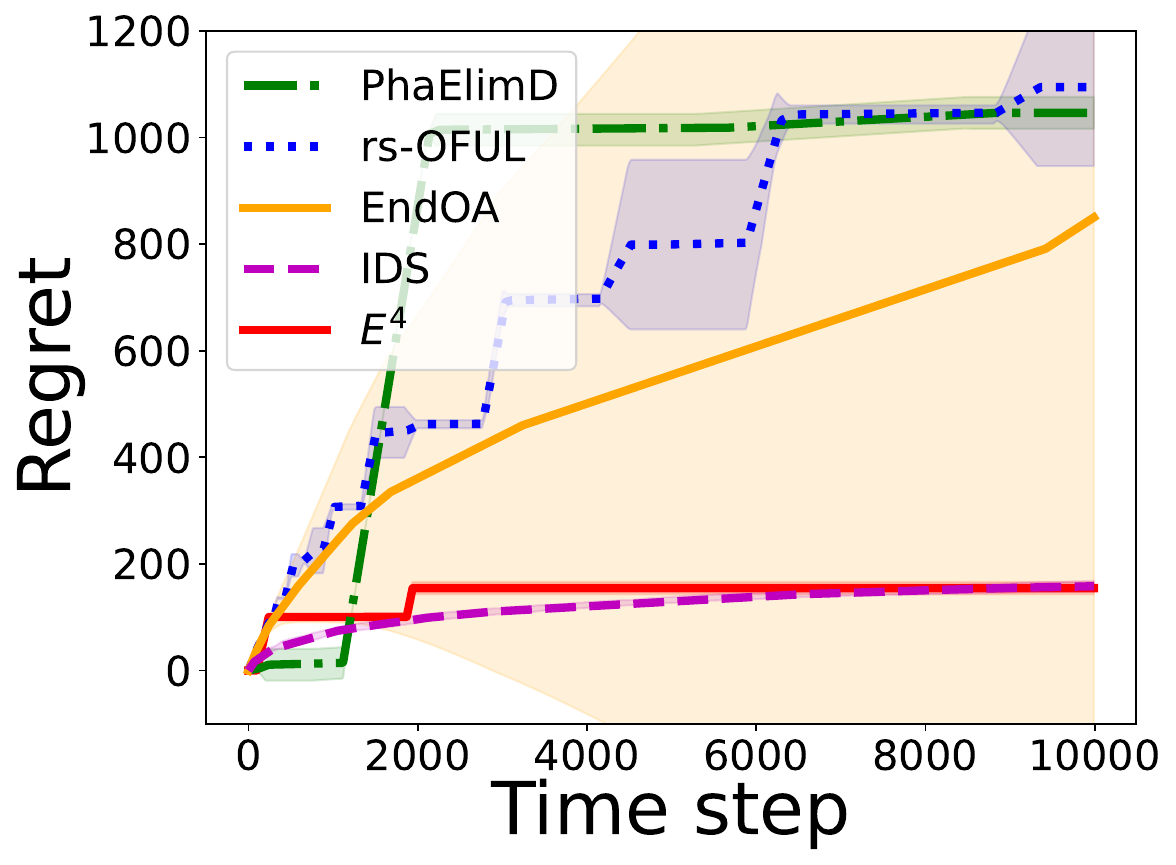}}
    \subfigure[$d=2,K=3,\epsilon=0.01$]{\includegraphics[scale=0.2]{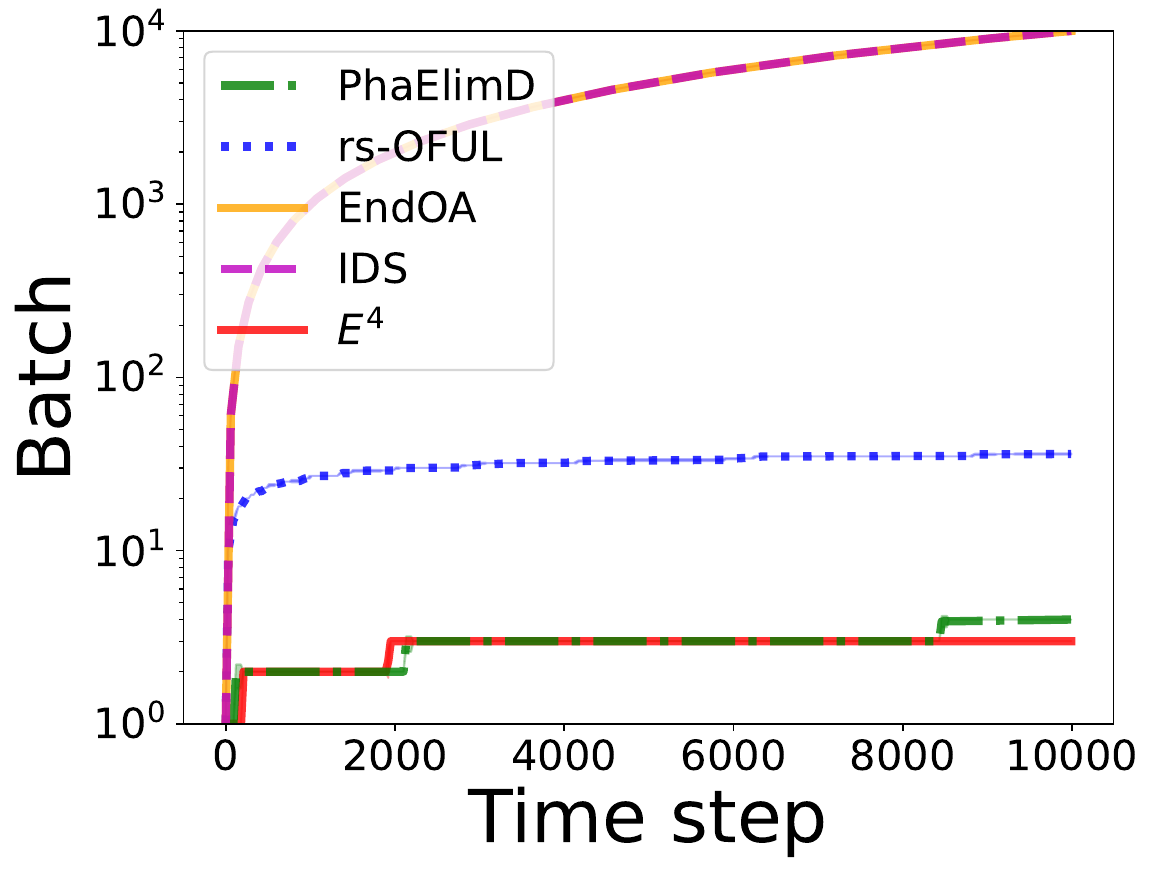}}
    \subfigure[$d=2,K=3,\epsilon=0.2$]{\includegraphics[scale=0.2]{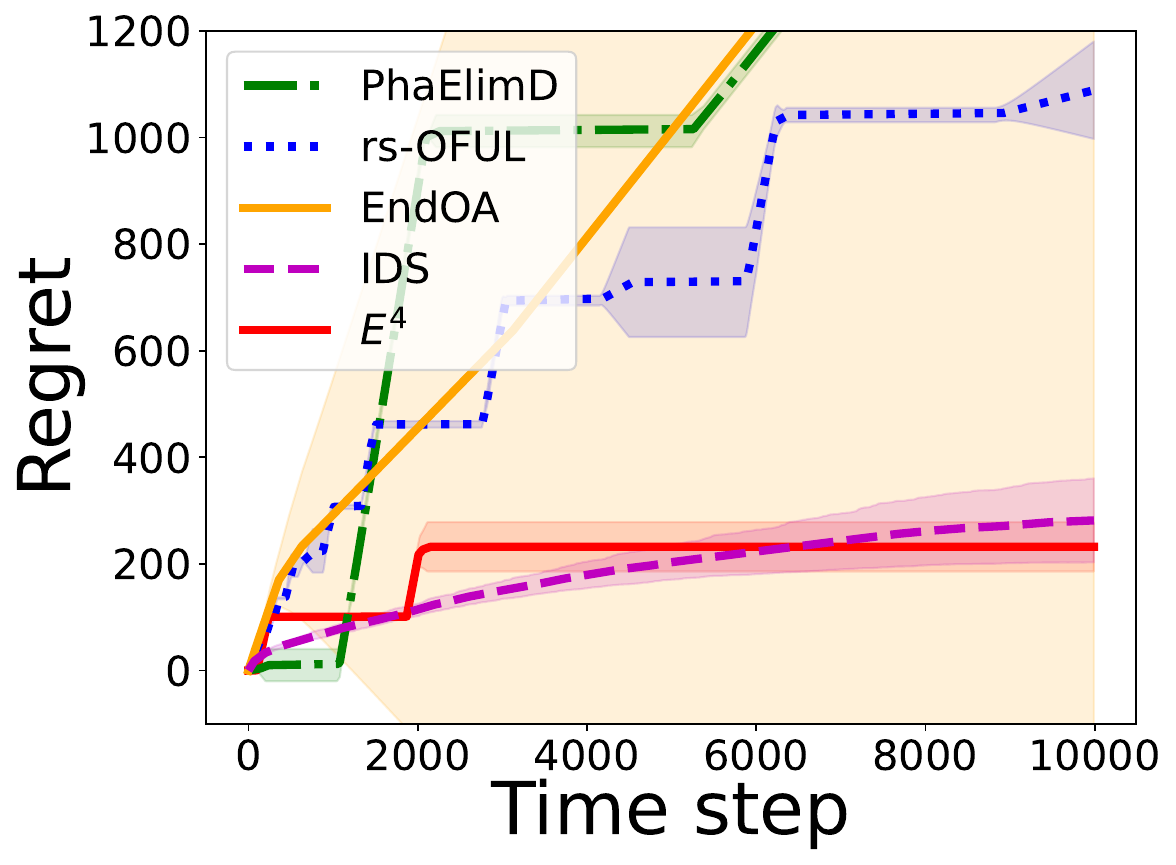}}
    \subfigure[$d=2,K=3,\epsilon=0.2$]{\includegraphics[scale=0.2]{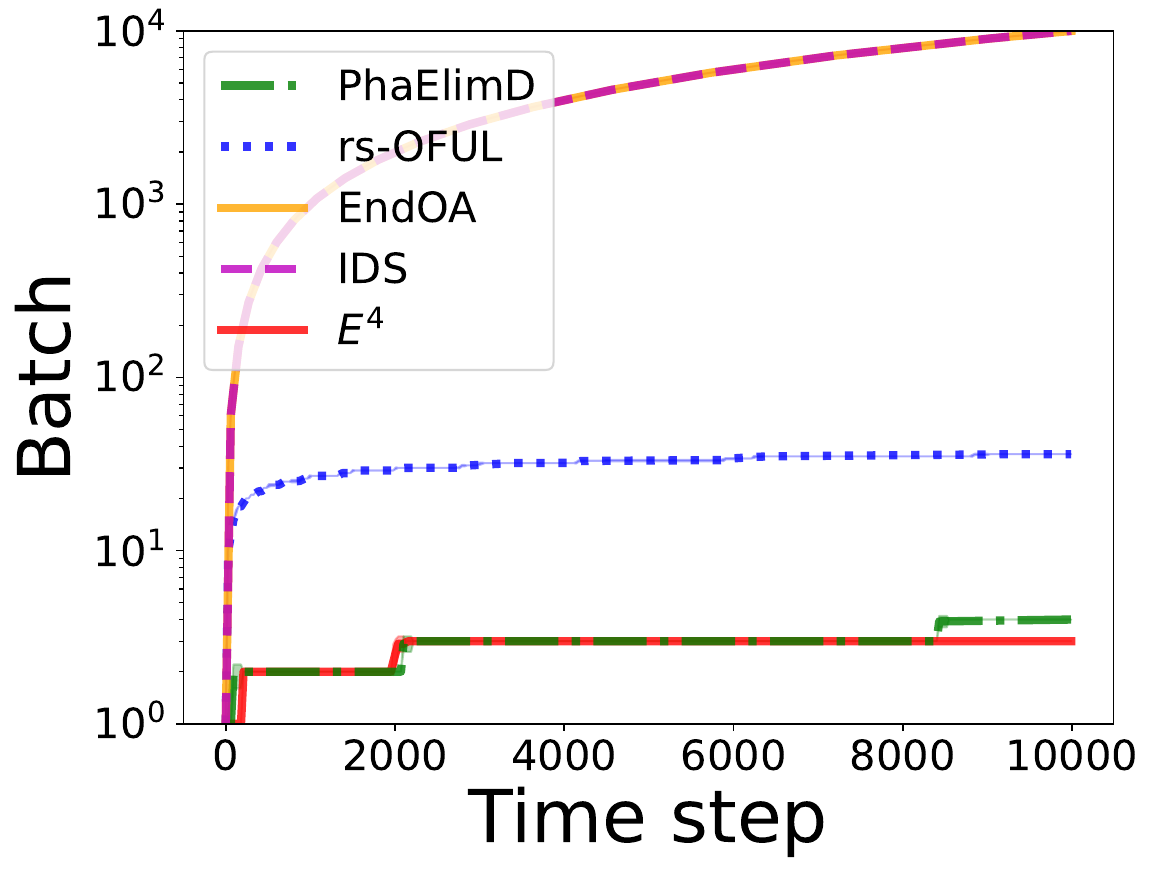}}
    \caption{Regret and Batch Analysis: \textit{End of Optimism} instances ($d=2,K=3$).%
    }
    \label{fig:simulation-end-d-2}
\end{figure*}

\begin{figure*}[!htbp]
    \centering
    \subfigure[$d=3,K=5,\epsilon=0.01$]{\includegraphics[scale=0.2]{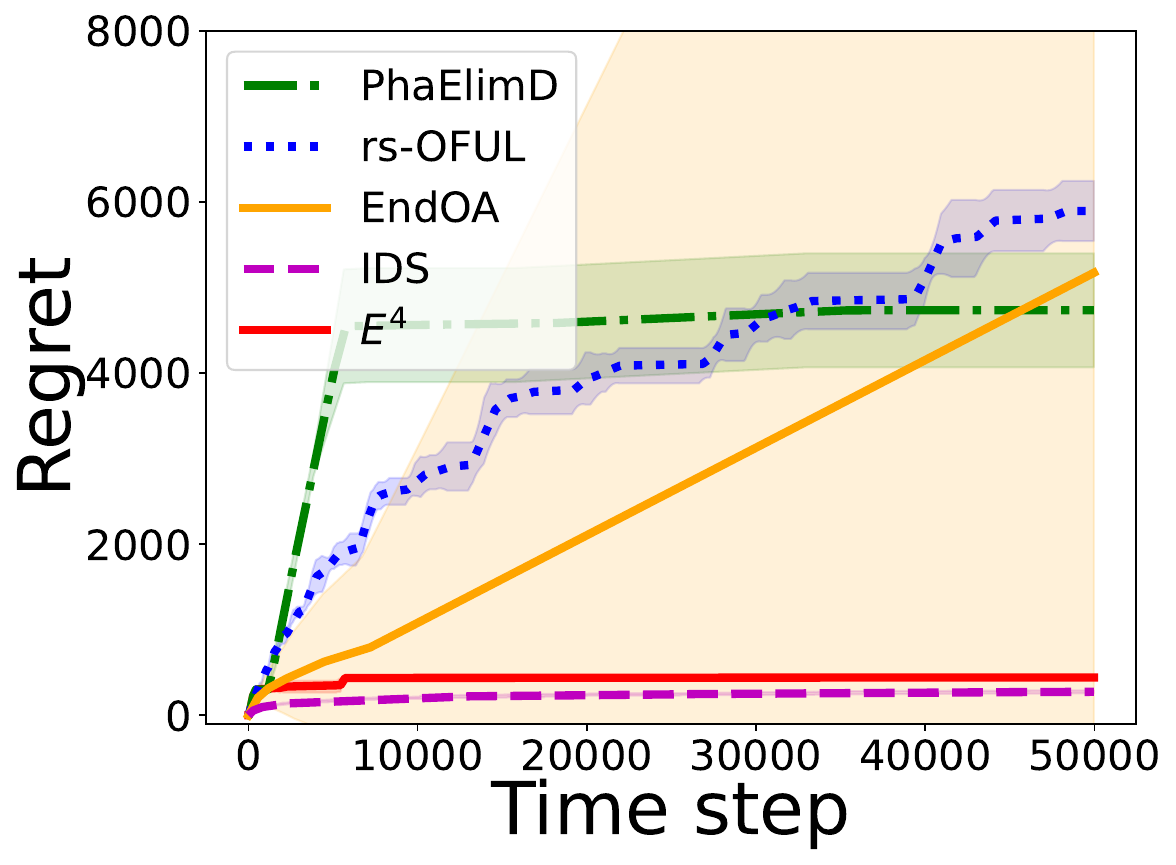}}
    \subfigure[$d=3,K=5,\epsilon=0.01$]{\includegraphics[scale=0.2]{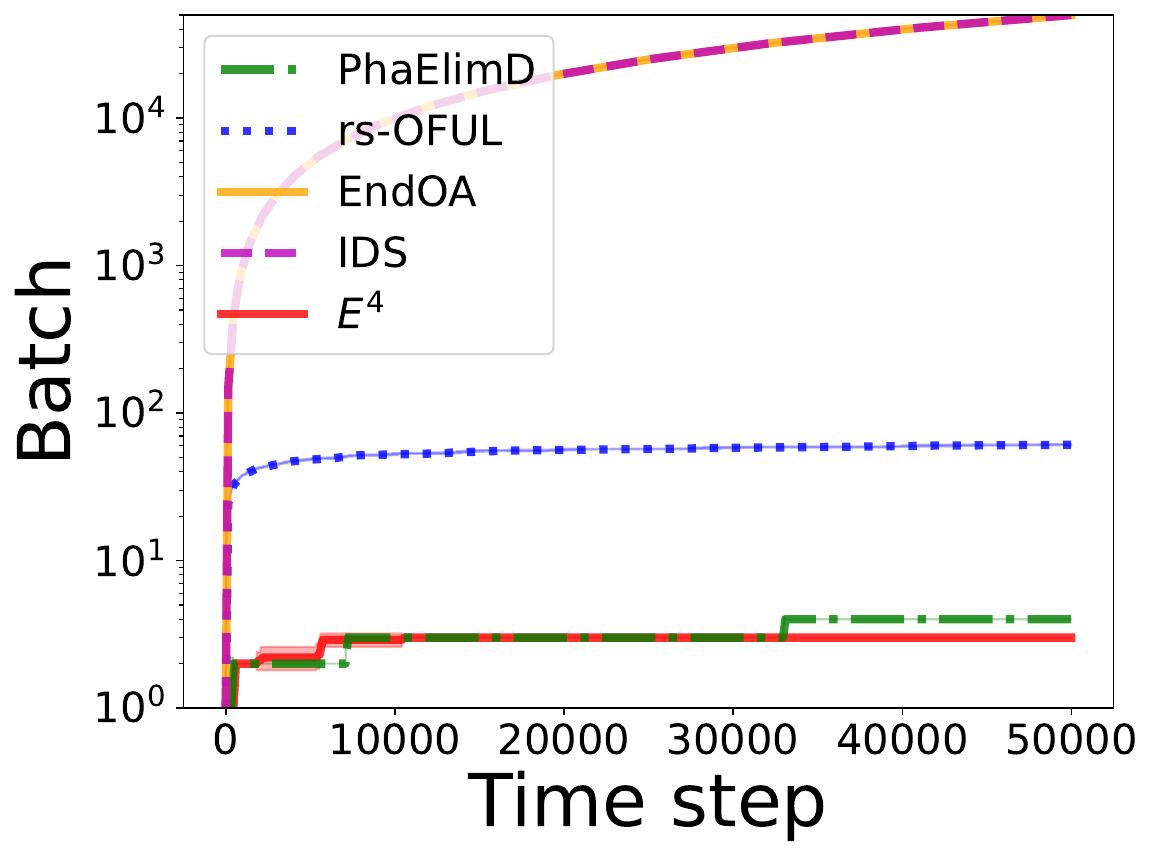}}
    \subfigure[$d=3,K=5,\epsilon=0.2$]{\includegraphics[scale=0.2]{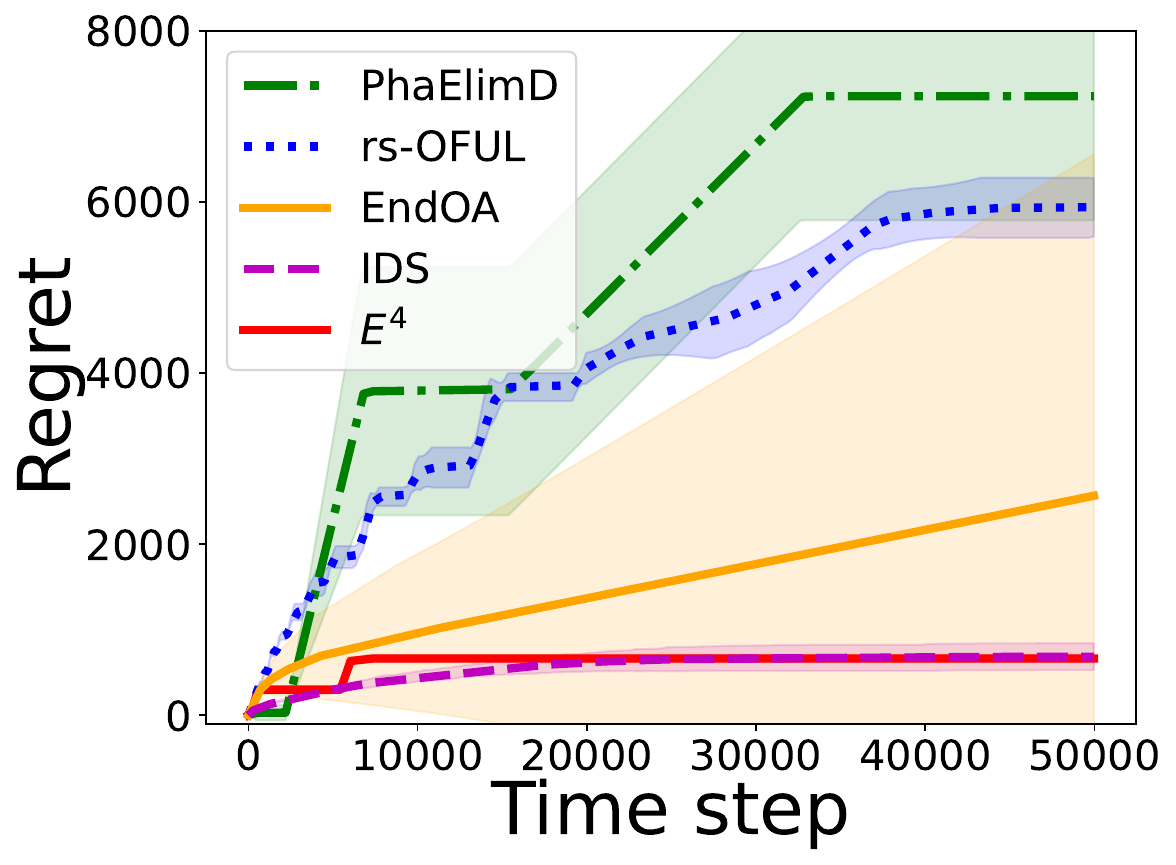}}
    \subfigure[$d=3,K=5,\epsilon=0.2$]{\includegraphics[scale=0.2]{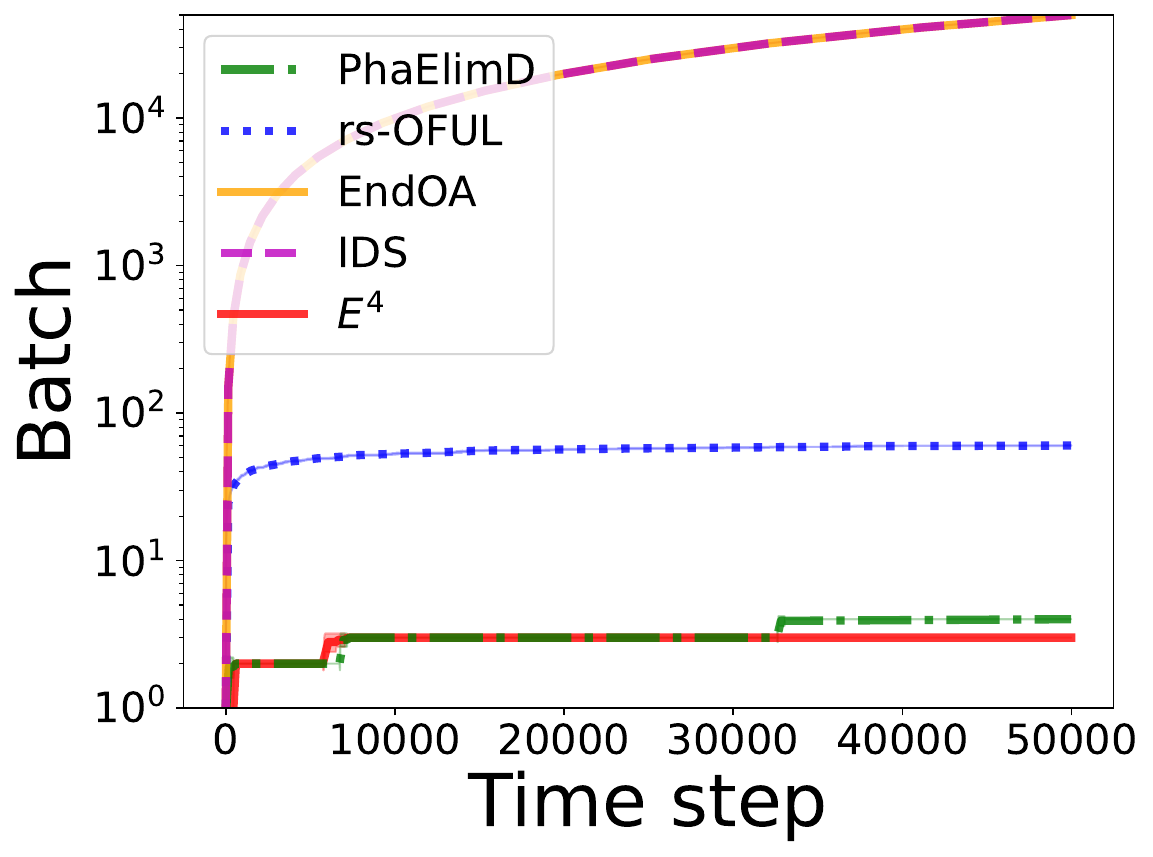}}
    \caption{Regret and Batch Analysis: \textit{End of Optimism} instances ($d=3,K=5$).
    }
    \label{fig:simulation-end-d-3}
\end{figure*}

\begin{table*}[th]
\centering
\caption{Batch Complexity Analysis: \textit{End of Optimism} instances. Note that batch complexity of sequential algorithms like \algendoa\  and \algids\  equals time horizon $T$.}
\label{table:batch-complexity-end} 
\begin{sc}
\begin{tabular}{clccccc}
\toprule
\multicolumn{2}{c}{Instance}    & \algname & \algped & \algoful &\algendoa &\algids\\ 
\midrule
\multirow{2}{*}{$d=2,K=3,T=10000$}

&$\epsilon=0.01$& $3.0\pm 0.0$& $4.0\pm 0.0$ & $36.1\pm 0.3$&-&-\\
&$\epsilon=0.2$& $3.0\pm 0.0$& $4.0\pm 0.0$ & $37.0\pm 0.0$&-&-\\
\multirow{2}{*}{$d=3,K=5,T=50000$}& $\epsilon=0.01$&$3.0\pm 0.0$& $4.0\pm 0.0$ & $61.0\pm 0.5$&-&-\\
&$\epsilon=0.2$& $3.0\pm 0.0$& $4.0\pm 0.0$ & $60.5\pm 0.8$&-&-\\
\multirow{2}{*}{$d=5,K=9,T=100000$}&$\epsilon=0.01$& $3.0\pm 0.0$& $4.0\pm 0.0$ & $102.3\pm  0.9$&-&-\\
&$\epsilon=0.2$& $3.0\pm 0.0$& $4.0\pm 0.0$ &$101.8\pm 0.6$&-&-\\
\bottomrule
\end{tabular}%
\end{sc}
\end{table*}

\subsection{Ablation Study of \algname\ on the End of Optimism Instances}\label{sec:ablation-study-EndOA-instance}
In this section, our focus is solely on batched algorithms, as our primary concern lies in evaluating the performance of batched linear bandits algorithms.
According to \citet{lattimore2017end}, optimism based algorithms would fail to achieve asymptotic optimality on these hard \textit{End of Optimism} instances. And the construction of such instances need sufficiently small $\epsilon$. So it is meaningful to verify this statement by changing the value of $\epsilon$. Note that \algoful\  and \algped\ are both algorithms based on optimism, with the latter one discards low-rewarding arms without considering their information gain.
We use the simplest \textit{End of Optimism} instances (see \Cref{fig:end-of-opt-instance}) given by \eqref{eq:EndOA-instances} where we choose $d=2$, $\epsilon=0.005,0.01,0.05,0.1,0.15,0.2$.
\begin{figure*}[!htbp]
    \centering
    \subfigure[$\epsilon=0.005$]{\includegraphics[scale=0.272]{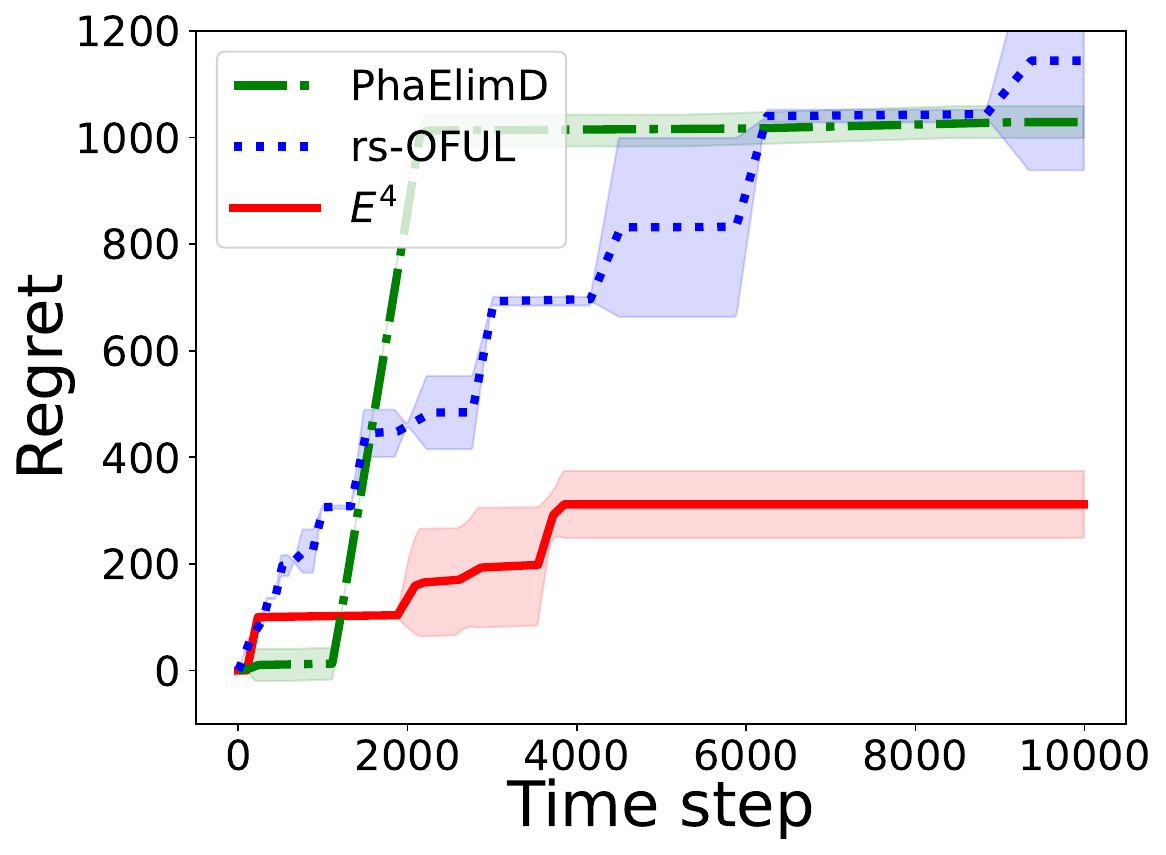}}
    \subfigure[$\epsilon=0.01$]{\includegraphics[scale=0.272]{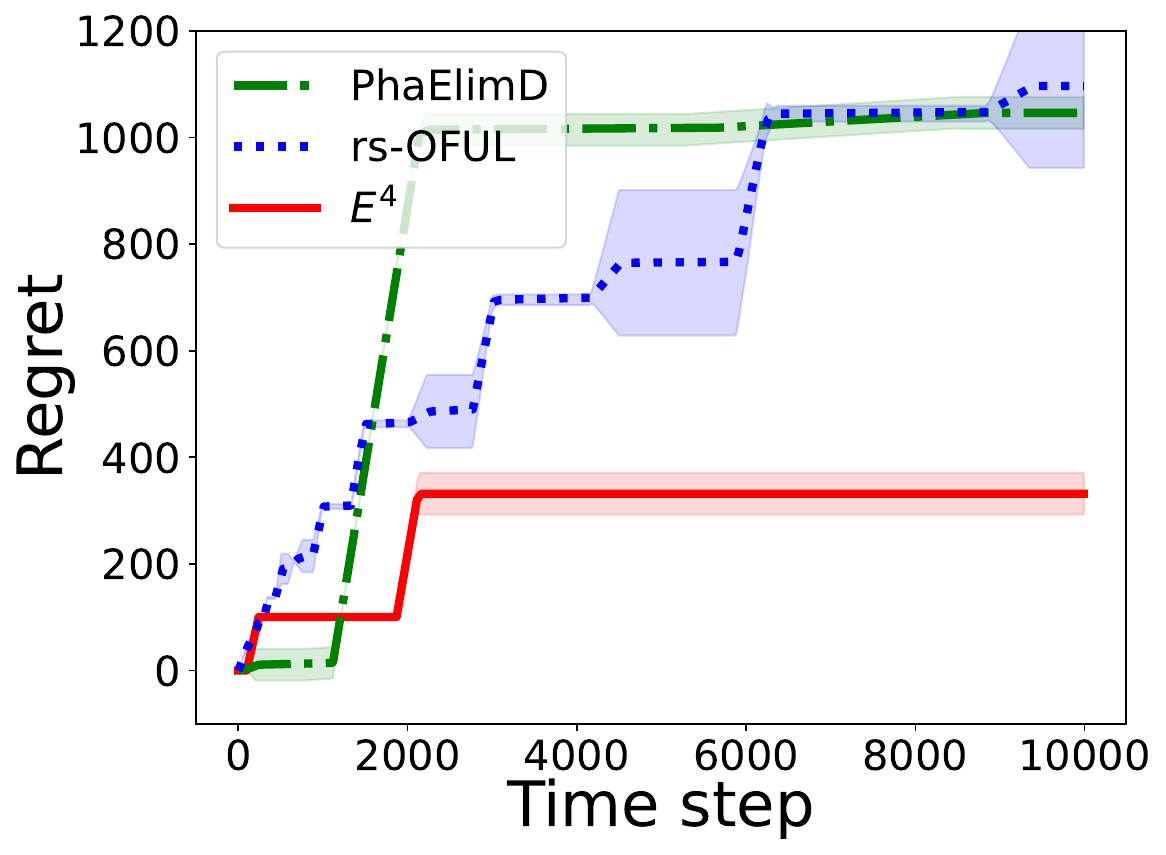}}
    \subfigure[$\epsilon=0.05$]{\includegraphics[scale=0.272]{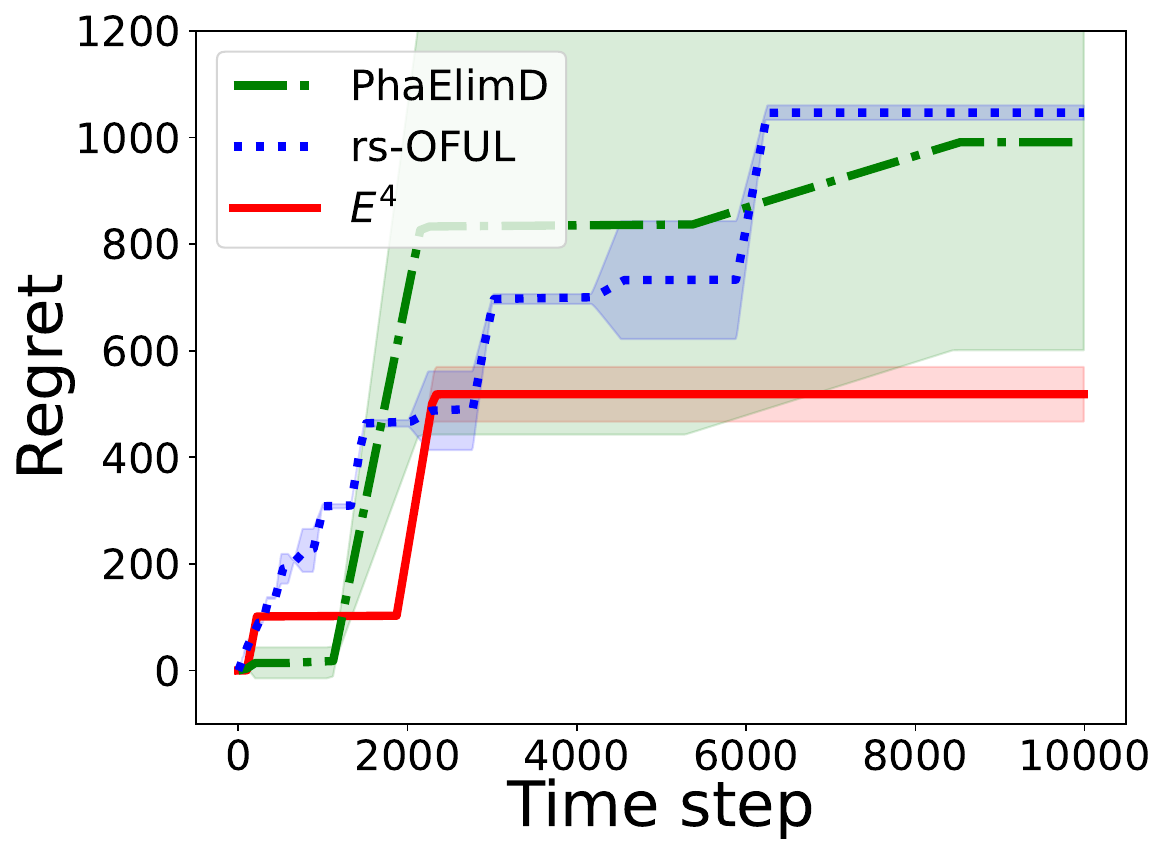}}
    
    \subfigure[$\epsilon=0.1$]{\includegraphics[scale=0.272]{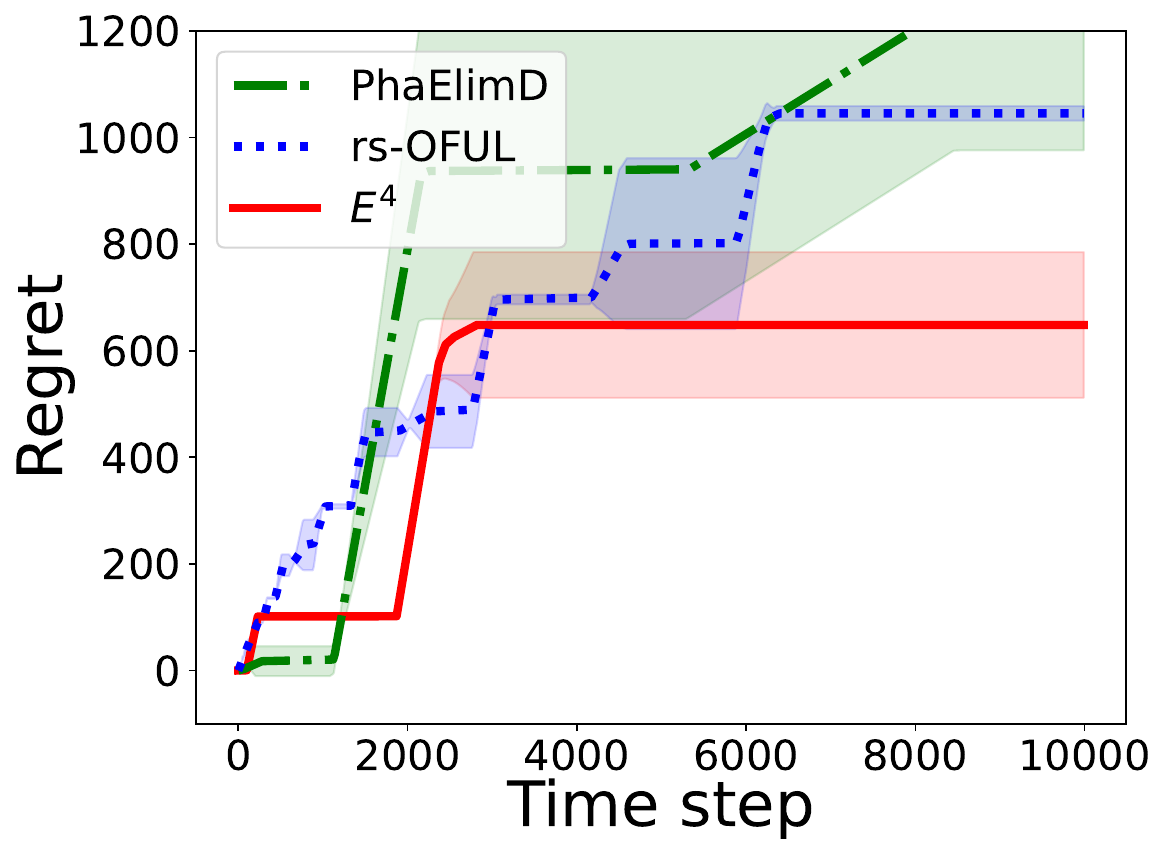}}
    \subfigure[$\epsilon=0.15$]{\includegraphics[scale=0.272]{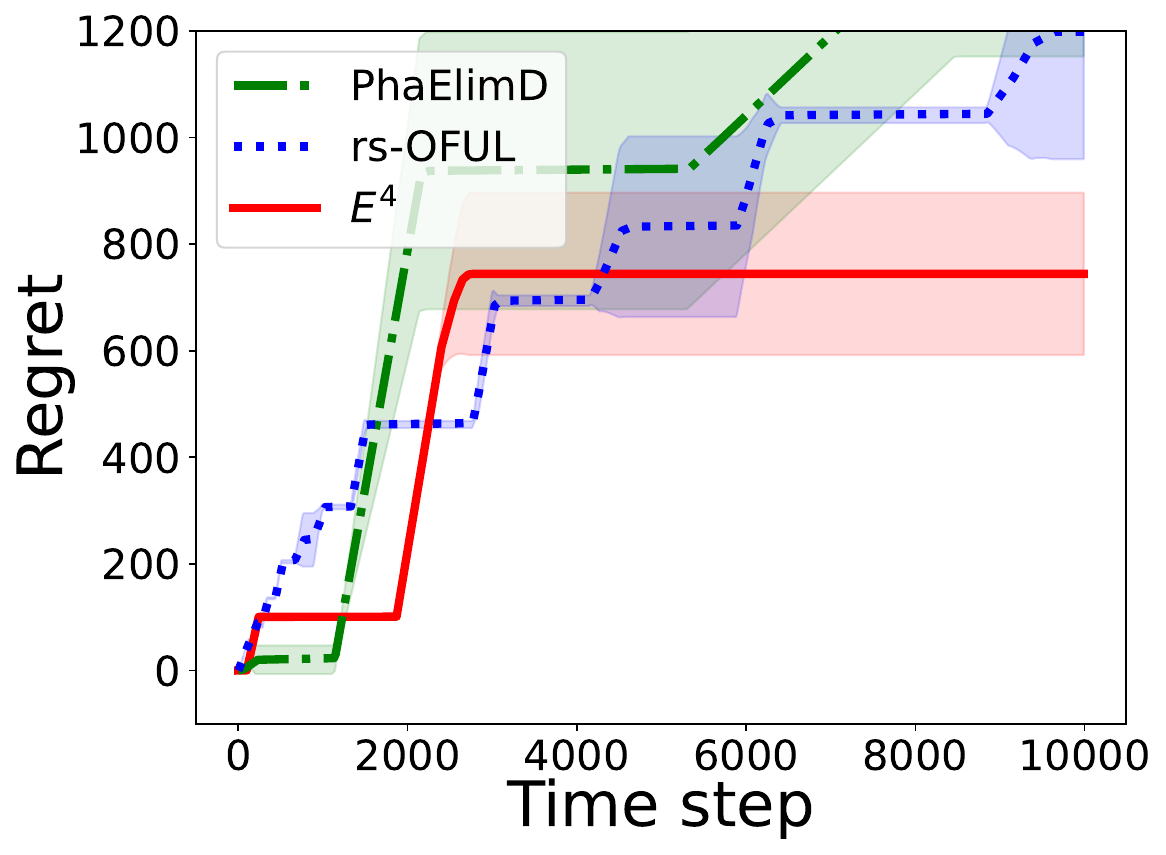}}
     \subfigure[$\epsilon=0.2$]{\includegraphics[scale=0.272]{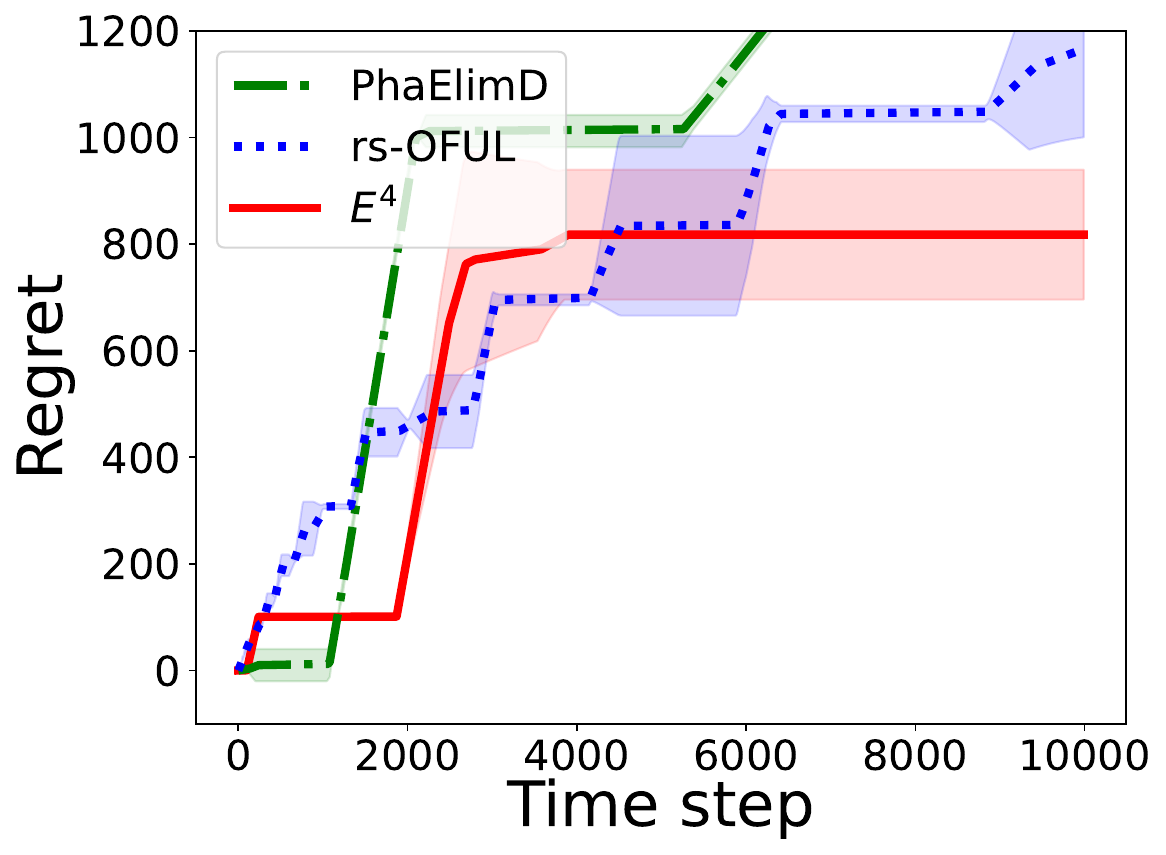}}
    \caption{Ablation study on the parameter $\epsilon$. 
    }\label{fig:research-on-epsilon}
\end{figure*}

As the results in \Cref{fig:research-on-epsilon} show, when $\epsilon$ gets larger, the regret of algorithms gets larger. This is because agents can quickly discard the worst arm $\xb_3$ and suffer the regret from $\xb_2$ in the following rounds. A larger $\epsilon$ leads to a larger regret (worse case). 

However, we can also see that as $\epsilon$ increases, the performance difference between \algname\  and other optimistic algorithms decreases. This is because optimism based  algorithms  can achieve worst-case (instances with large $\epsilon$) optimal regret, but fail to achieve asymptotically optimality under instances with small enough $\epsilon$. In fact, we can see the regret of optimism based algorithms as the minimax optimal regret, and see the regret of \algname\ as the asymptotically optimal regret. Hence these experimental results help to understand the difference between these two types of optimality.

\subsection{Evaluations of \algname\ on Randomly Generated Instances}\label{sec:experiments-random-instances} 
In this section, we empirically evaluate algorithms on randomly generated instances. Notably, we have excluded \algids\ from consideration here due to its consistently poor performance in batch complexity and runtime efficiency as is shown in \Cref{fig:simulation-end-d-5} and \Cref{table:runtime}. For each chosen dimension of features and number of arms, we generate an instance with $\Vert\btheta^*\Vert_2=1$, where arm features are sampled from $U([0,1]^K)$. We set the time horizon to $T=50000$ and conduct 10 independent experiments for each instance. The regret and batch complexity results are illustrated in \Cref{fig:simulation-random}. The outcomes indicate that our algorithm \algname\ exhibits commendable performance even in randomly generated instances, which further demonstrates its advantages.
\begin{figure*}[!htbp]
    \centering
    \subfigure[$d=2,K=3$]{\includegraphics[scale=0.2]{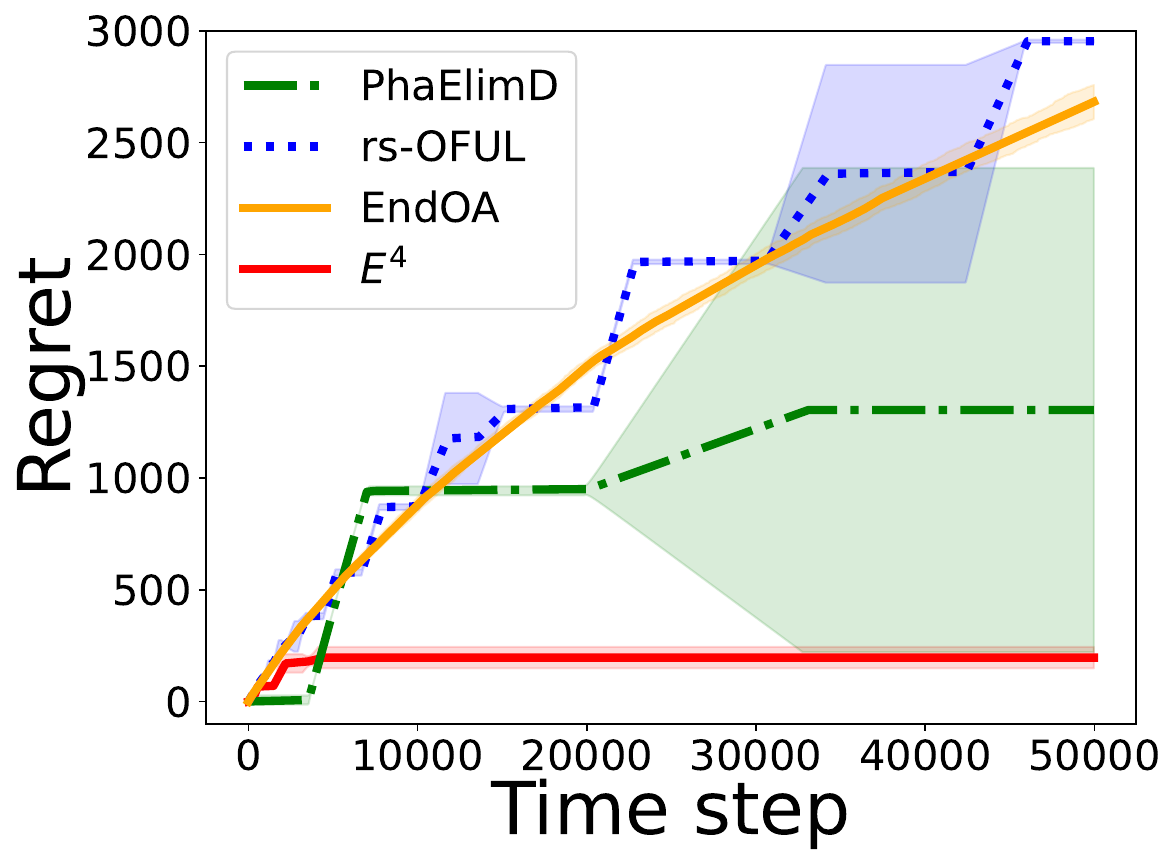}}
    \subfigure[$d=2,K=3$]{\includegraphics[scale=0.2]{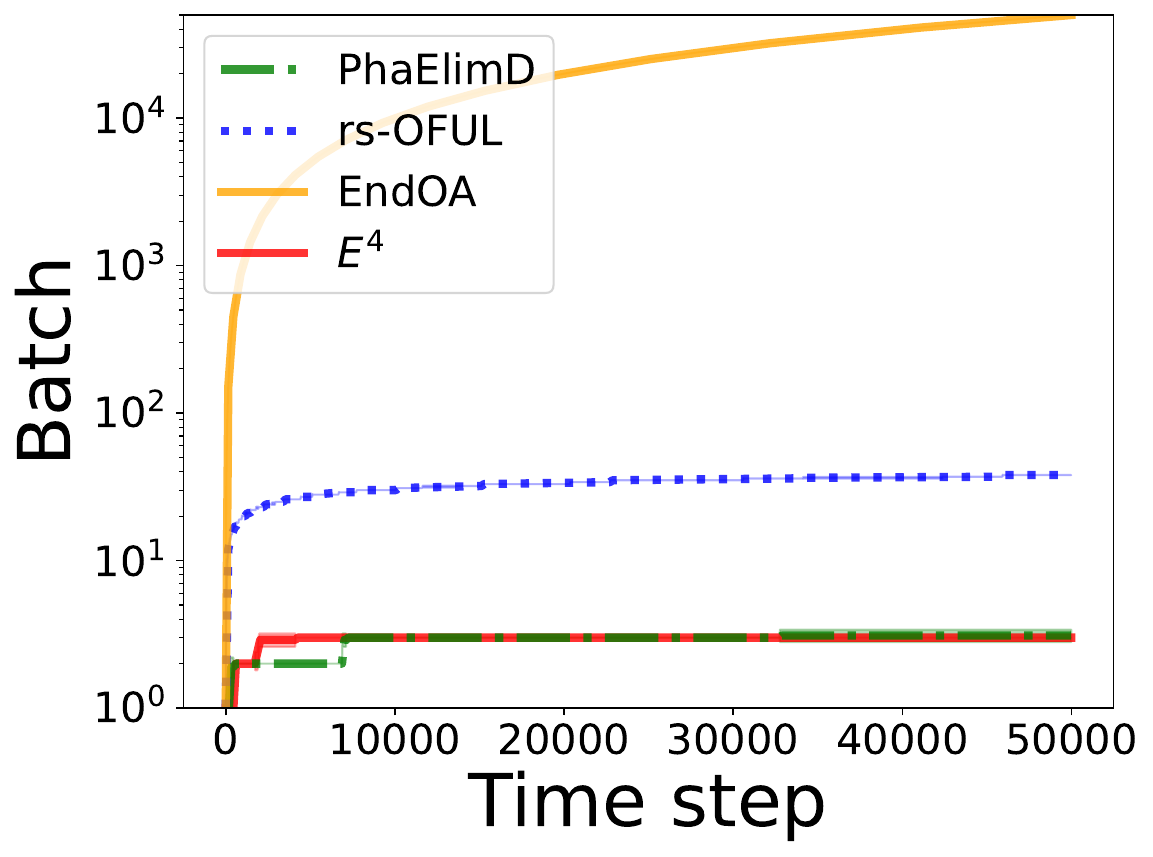}}
    \subfigure[$d=3,K=5$]{\includegraphics[scale=0.2]{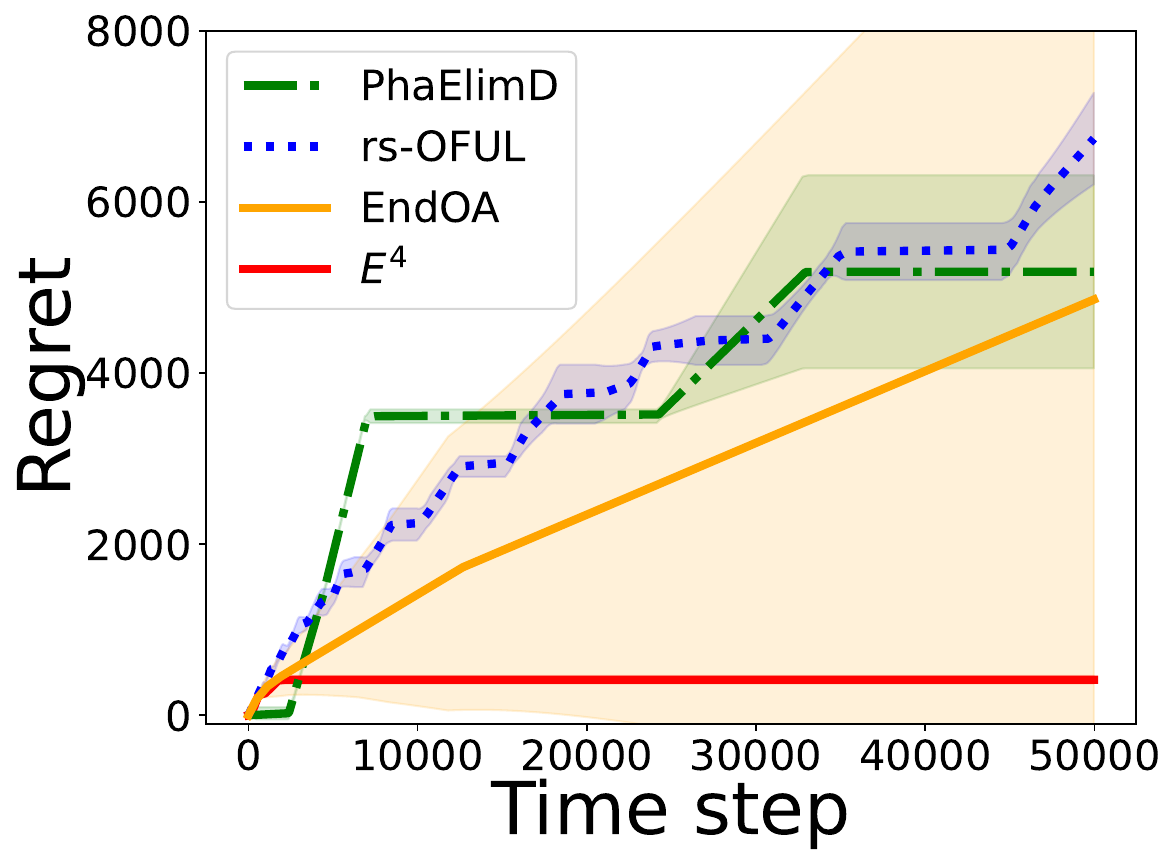}}
    \subfigure[$d=3,K=5$]{\includegraphics[scale=0.2]{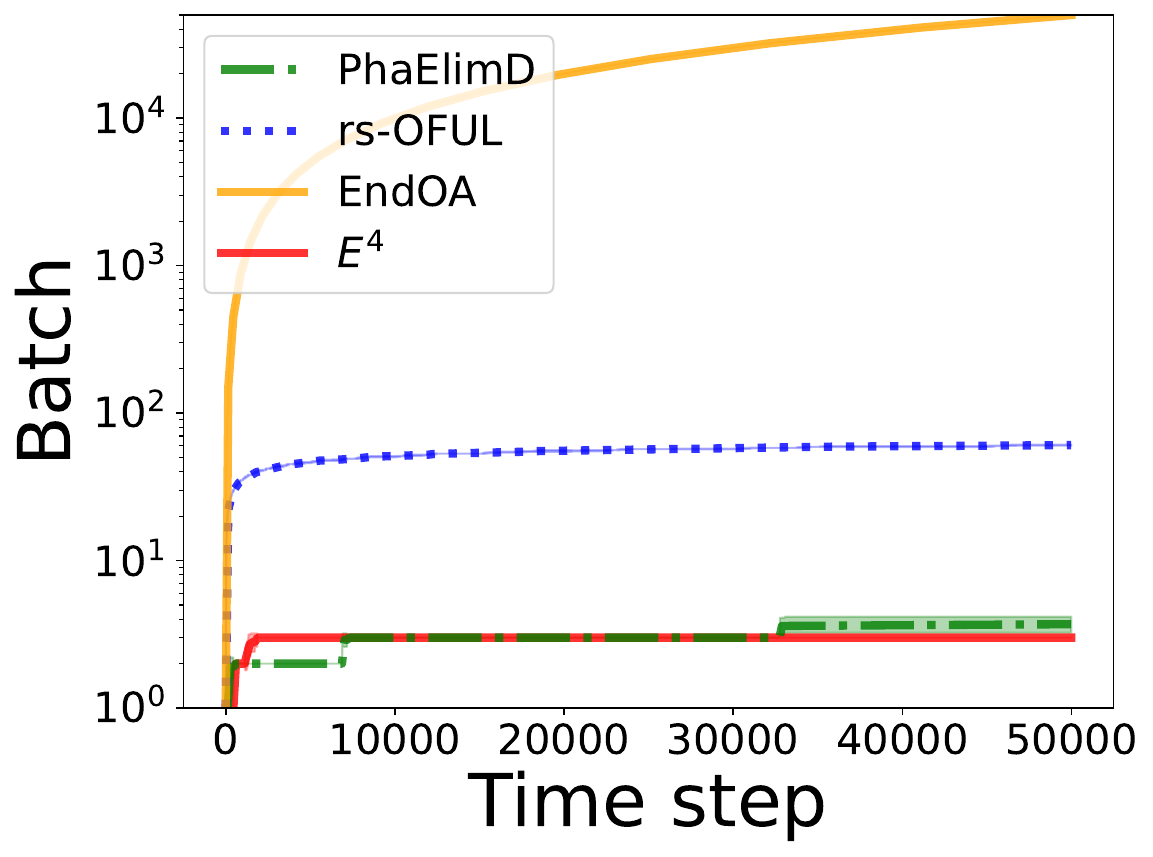}}
    \subfigure[$d=5,K=9$]{\includegraphics[scale=0.2]{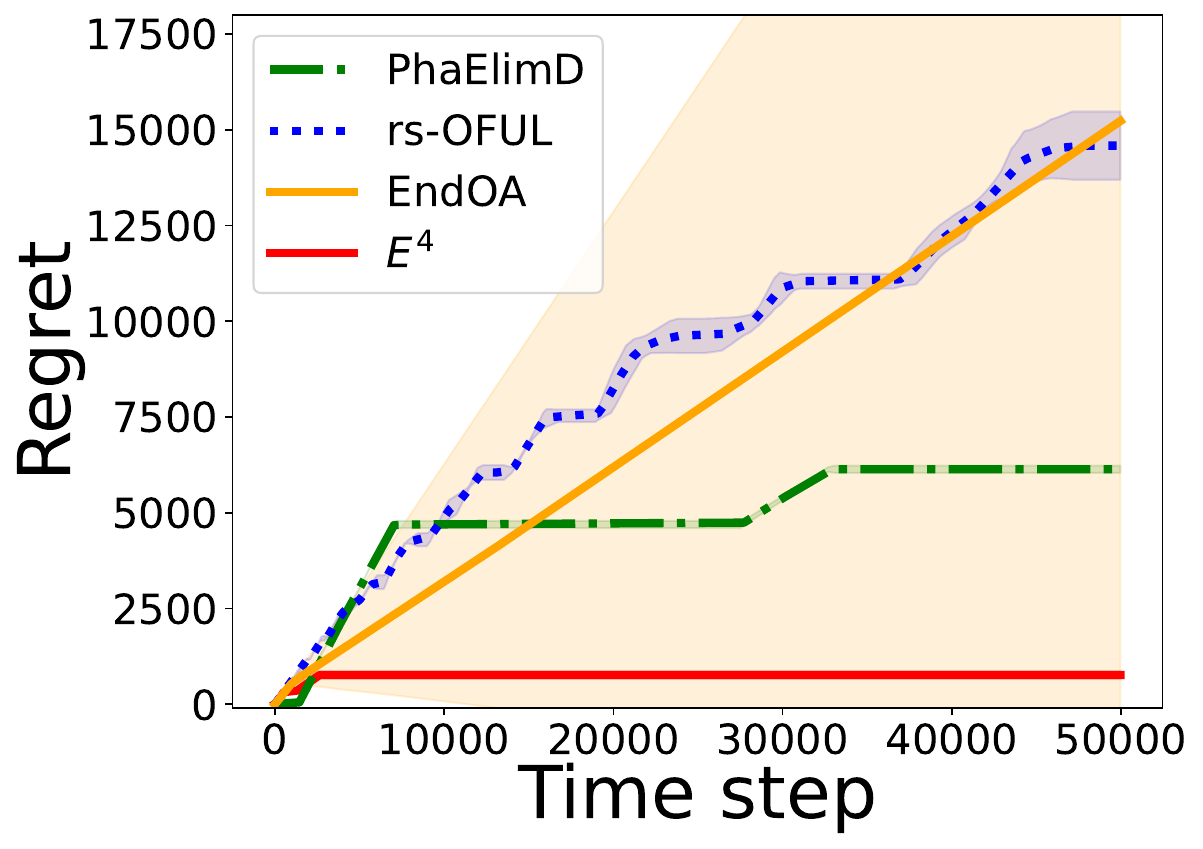}}
    \subfigure[$d=5,K=9$]{\includegraphics[scale=0.2]{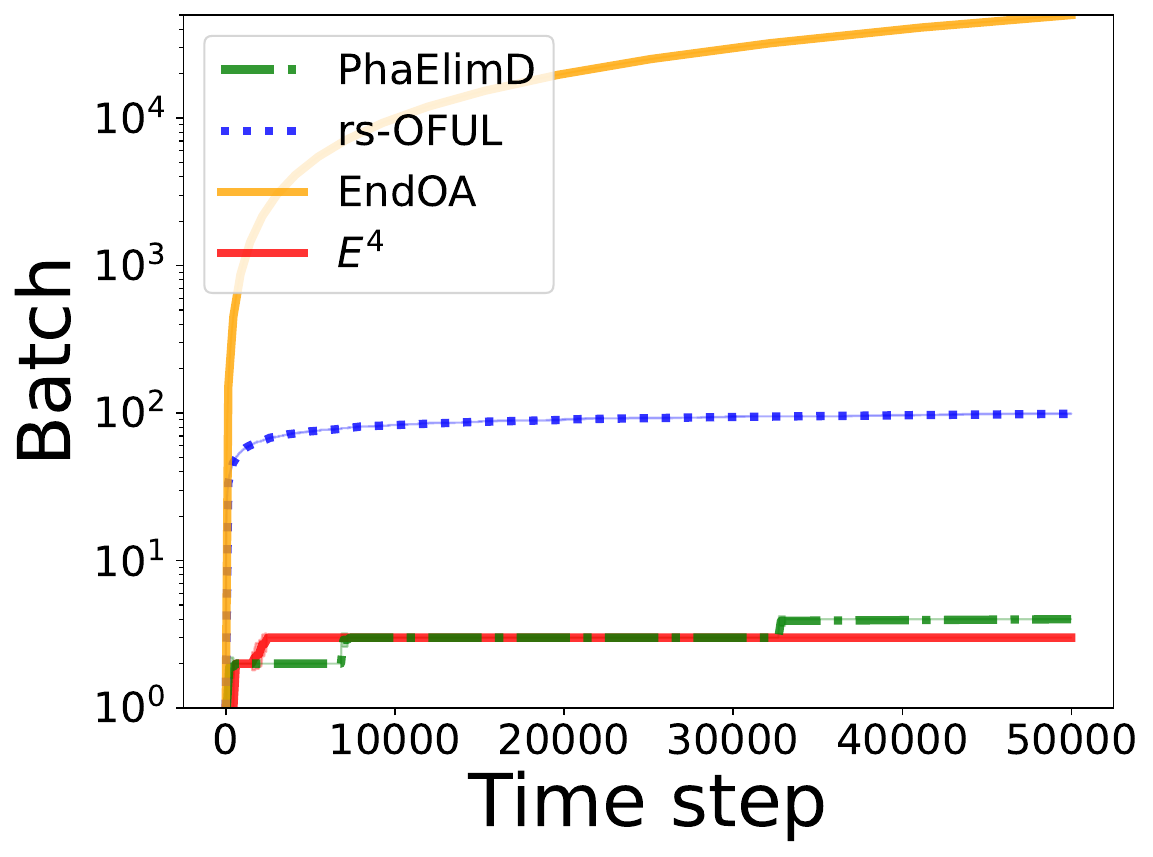}}
    \subfigure[$d=20,K=50$]{\includegraphics[scale=0.2]{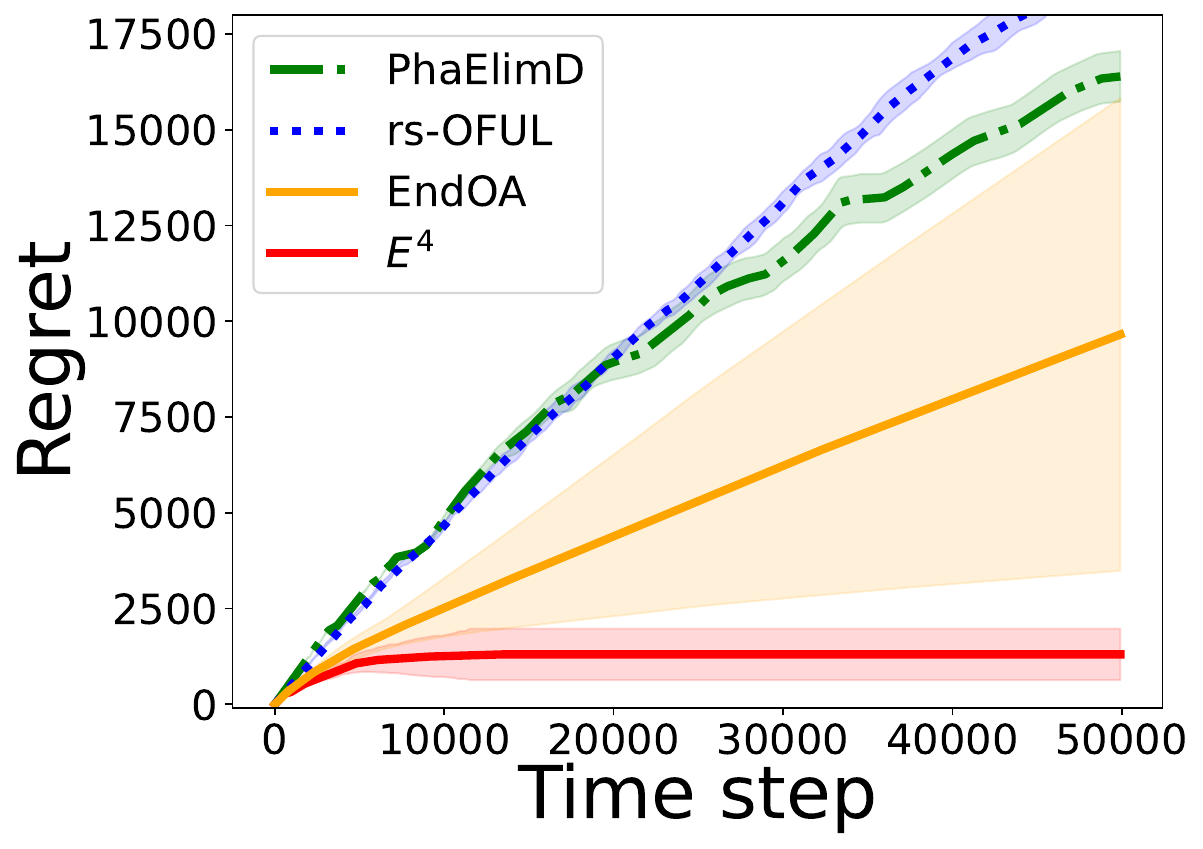}}
    \subfigure[$d=20,K=50$]{\includegraphics[scale=0.2]{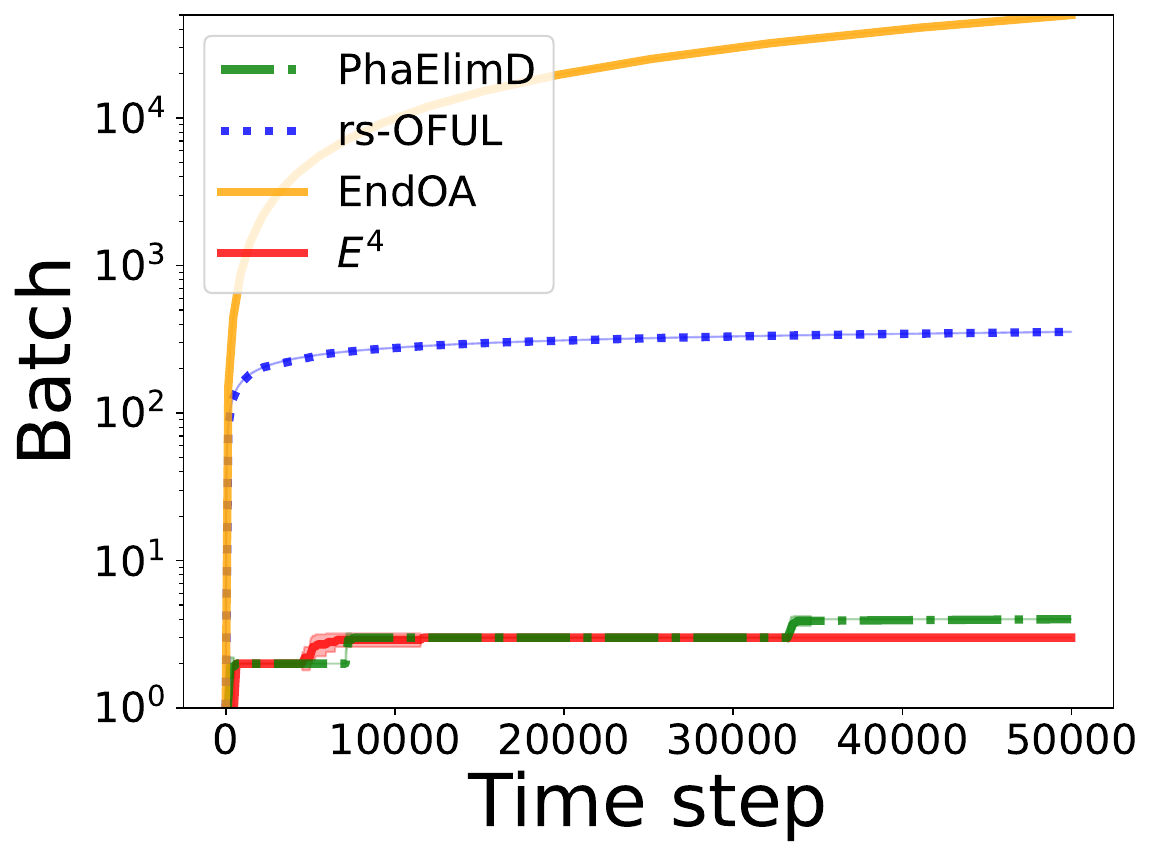}}
    \caption{Regret and Batch  Analysis: Random generated instances ($d=2,3,5,20$).
    }
    \label{fig:simulation-random}
\end{figure*}

\begin{table*}[th]
\caption{Batch Complexity Analysis: Random generated instances. Note that again the batch complexity of the sequential algorithm \algendoa\  equals the time horizon $50000$.}
\label{table:batch-complexity-random} 
\centering
\begin{sc}
\begin{tabular}{cccccc}
\toprule
Instance ($T=50000$)  & \algname & \algped & \algoful&\algendoa \\ 
\midrule
$d=2,K=3$& $3.0\pm 0.0$& $3.0\pm 0.0$ &$38.0\pm 0.0$&- \\
$d=3,K=5$& $3.0\pm 0.0$& $3.7\pm 0.5$ &$60.7\pm 0.5$ &-\\
$d=5,K=9$& $3.0\pm 0.0$& $4.0\pm 0.0$ & $98.8\pm 0.7$&-\\
$d=20,K=50$& $3.0\pm 0.0$& $4.0\pm 0.0$ &$355.5\pm 0.7$&- \\
\bottomrule
\end{tabular}
\end{sc}
\end{table*}

\section{Lower Bounds on the Batch Complexity}\label{sec:proof_lower_bound}
The following lemma implies that any algorithm cannot find the best arm correctly with a high enough probability using only one exploration stage of $o(\log T)$ steps.
\begin{lemma}\label{lemma:o-logT-not-enough}
Consider any linear bandit instances with two arms. Suppose a best-arm-identification algorithm has only $o(\log T)$ steps of exploration, let $\tilde \xb$ be its output, then there are infinite many $T$ such that $\PP(\tilde \xb\neq \xb^*)> 1/\sqrt{T}$.

\end{lemma}
\subsection{Algorithms with Deterministic Number of Batches}
We first present the following theorem that states algorithms performing at most 2 batches are not asymptotically optimal.
\begin{theorem}\label{them:batch-lower-bound-fixed-batch}
    Algorithms that only perform $1$ or $2$ batches are not asymptotically optimal.
\end{theorem}
\begin{proof}
An algorithm is called asymptotically optimal if its regret satisfies $\limsup_{T\rightarrow\infty}R_T/\log T\leq c^*(\btheta^*)$ for all bandit instances. To prove the desired result, we just need to consider instances given by given by arm set $\cX=\{(1,0),(0,1)\}$, parameter $\btheta^*=(\alpha,0)$ or $(0,\alpha)$, where $\alpha\in(0,1)$.  Here the gap $\Delta=\alpha$. 
We aim to prove algorithms that only perform $1$ or $2$ batches fail to achieve asymptotic optimality on  these instances.

First, no consistent algorithm can achieve sub-linear regret with only $1$ batch. Algorithms with $1$ batch must decide how many times to pull each arm at the beginning, then pull arms for $T$ times in total according to the decision. With no information of unknown instance, it must suffer $O(T)$ regret on some instance.

Then we prove that no algorithm can achieve asymptotic optimality with $2$ batches. Assume certain  algorithm with policy $\pi$ can achieve this, meaning its regret satisfies
\begin{align}\label{eq:regret-in-lower-bound}
    \lim_{T\rightarrow\infty}\frac{R_T}{\log T}=c^*(\btheta^*)=\sum_{\xb\in\cX^-} \frac{2}{\Delta_{\xb}}=\frac{2}{\alpha}, 
\end{align}
where the second equality is the explicit form of the $c^*$ on multi-armed bandits setting. The detail can be seen from the  {Example 3} in \citet{lattimore2017end}. \eqref{eq:regret-in-lower-bound} implies the asymptotically optimal algorithm pulls the optimal arm for  $\Theta(T)$ times in expectation, and pulls the suboptimal arm for $\frac{2}{\alpha^2}\log T+o(\log T)$ times in expectation.
Since the algorithm has only $2$ batches, let $t_1$ and $t_2$ be the expected batch sizes for the first and second batch, where $t_1+t_2=T$. Without loss of generality, we can assume the agent use a uniform exploration (pull all arms equal times) in the first batch.

Note that we consider the problem in an asymptotic setting, where $T\rightarrow\infty$. The statements above show that an optimal algorithm has $\frac{2}{\alpha}\log T+o(\log T)$ regret. Then we claim $t_1=O(\log T)$.  Here we can see that if $t_1$ has an order higher than $O(\log T)$, the agent must pull one arm for some times with the order larger than $O(\log T)$. On some instances, this leads regret with the order larger than $O(\log T)$, which is an contradiction with asymptotic optimality.
We see $t_2=T-t_1=T-o(T)$. This means in the second batch, the algorithm must pull one arm for at least $T/2-o(T)$ times.

Then we prove that $t_1$ cannot be $o(\log T)$ by contradiction. Suppose an algorithm (denoted as Algorithm A) is a 2-batch asymptotically optimal algorithm with the initial batch size being $o(\log T)$. Another algorithm for best-arm identification, referred to as Algorithm B, can be derived from Algorithm A. This derivation is straightforward: it begins with the exploration phase using the first batch of Algorithm A, collects the data, updates the policy, and then diverges from Algorithm A's approach by outputting the arm, denoted as $\tilde \xb$, that is to be pulled more frequently in the second batch of Algorithm A as the best-arm. Since the second batch of Algorithm A has size $\Theta(T)$, it must pull $\tilde \xb$ for $\Theta(T)$ times. Given that Algorithm A is asymptotically optimal, the probability of selecting an incorrect arm is $\PP(\tilde \xb \neq \xb^*) \leq 1/\sqrt{T}$ for sufficiently large $T$. Otherwise, the regret of Algorithm A would be at least $\Theta(T) \times \alpha \times 1/\sqrt{T}=\Theta(\sqrt{T})$, for infinite many $T$, contradicting its asymptotic optimality, which is characterized by a regret of $O(\log T)$ order. However, viewing Algorithm B as a best-arm identification algorithm introduces a new issue: achieving the bound $\PP(\tilde \xb \neq \xb^*) \leq 1/\sqrt{T}$ for sufficiently large $T$ necessitates an exploration process involving more than $\Theta(\log T)$ samples, as indicated by \Cref{lemma:o-logT-not-enough}, which is a contradiction. This implies that for Algorithm A to be asymptotically optimal, it cannot limit its initial batch to only $o(\log T)$ samples.
So $t_1$ cannot be $o(\log T)$. 
Combining with $t_1=O(\log T)$ we know $t_1=\Theta(\log T)$.

But the main problem arises here: how many steps do we need to explore in the first batch?
In fact, we have no information about the unknown instance at the beginning. In the first batch, we may choose to pull each arm for $\frac{2}{\beta^2}\log T+o(\log T)$ times, where $\beta>0$ is a predetermined constant parameter. Then for instances with $\Delta=\alpha>\beta$ and sufficiently large $T$, the regret of batch $1$ is $\alpha\cdot \frac{2}{\beta^2}\log T+o(\log T)>\frac{2}{\alpha}\log T+o(\log T)$, which implies
the algorithm is suboptimal. This shows algorithms with only $2$ batches cannot achieve asymptotic optimality, which ends the proof. 
\end{proof}

\subsection{Proof of \Cref{them:batch-lower-bound-adaptive-batch}}

Firstly, any consistent algorithm must have at least $2$ batches.  We still consider instances given by given by arm set $\cX=\{(1,0),(0,1)\}$, parameter $\btheta^*=(\alpha,0)$ or $(0,\alpha)$, where $\alpha\in(0,1)$. 
Then we aim to prove any asymptotically optimal must have $3$ batches in expectation as $T\rightarrow$ on all these instances, which leads to the desired conclusion. Now we prove  by contradiction.
Suppose that an asymptotically optimal algorithm has $3-\delta$ batches in expectation on some instance among these, for some constant $\delta\in(0,1)$. 
Then we discuss the behavior of this algorithm on this certain bandit instance.
Let $\cM$ be the event that the algorithm has only $2$ batches, and $p=\PP(\cM)$. Therefore the algorithm has at least $3$ batches with probability $1-p$.
We have that the expected batch complexity of this algorithm
\begin{align*}
    3-\delta \geq 2\cdot p+3\cdot(1-p)=3-p,
\end{align*}which implies $p\geq \delta$.

We use the event $\cM$ to decompose the regret:
\begin{align*}
    R_T&=\EE\bigg[\sum_{t=1}^T\la {\xb}^*-{\xb}_t,\btheta^*\ra\bigg\vert \cM\Bigg]\cdot p+\EE\bigg[\sum_{t=1}^T\la {\xb}^*-{\xb}_t,\btheta^*\ra\bigg\vert \cM^c\Bigg]\cdot (1-p)\\
    &:=A_T\cdot  p+ B_T\cdot(1-p)
    \\
    &\geq A_T\cdot \delta,
\end{align*}
here we use $A_T=\EE\big[\sum_{t=1}^T\la {\xb}^*-{\xb}_t,\btheta^*\ra\vert \cM\big]$ to denote the regret given $\cM$, and $B_T=\EE\big[\sum_{t=1}^T\la {\xb}^*-{\xb}_t,\btheta^*\ra\vert \cM^c\big]$ to denote the regret given $\cM^c$. Since we suppose the algorithm is asymptotically optimal, i.e.,
\begin{align*}
    \lim_{T\rightarrow\infty}\frac{A_T\cdot \delta}{\log T}\leq \lim_{T\rightarrow\infty}\frac{R_T}{\log T}= c^*,
\end{align*}which implies $A_T\leq c^*/\delta\log T+o(\log T)$.
Given $\cM$, we denote $T=t_1+t_2$, where $t_1$ and $t_2$ are the expected batch sizes for two batches, respectively.
Similar to the proof of \Cref{them:batch-lower-bound-fixed-batch}, we have $t_1=O(\log T)$ and $t_2=\Theta(T)$, otherwise the order of the regret  would be larger than $\Omega(\log T)$.
We can also use the same arguments as the proof of \Cref{them:batch-lower-bound-fixed-batch} to prove that $t_1$ cannot be $o(\log T)$. Otherwise, $A_T=\Omega(\sqrt{T})$. Combining these results we know $t_1=\Theta(\log T)$.

The behavior of the algorithm in the first batch is not influenced by bandit instances. The algorithm must pull $\Theta(\log T)$ arms in the first batch. Then  we just analyze the regret in the first batch. And the algorithm faces the same problem as the proof of \Cref{them:batch-lower-bound-fixed-batch}: how many steps does it need to explore in the first batch? Use exactly the same way
we can find an instance where the algorithm is suboptimal, which is a contradiction. Actually, no matter how the agent allocates these $\Theta(\log T)$ steps of exploration, there exists an instance that makes the algorithm suboptimal. This ends the proof.

\section{Proofs of Main Theorems in \Cref{section-theoretical-analysis}}\label{sec:proof_upper_bound}
The two lemmas outlined below offer a theoretical basis for the exploration of D-optimal design and the application of Chernoff's  stopping rule mentioned in \Cref{section-algorithm}.

In practice, we use the Frank–Wolfe algorithm to solve the optimization problem in \eqref{eq:D-optimal-solving} to get a near-optimal solution $\tilde\bpi$, which  only exerts an impact on our theoretical analysis up to a constant order. In this way, the agent pulls each arm $\xb\in\cX$ for $\lceil 2\tilde \pi_\xb g(\tilde \bpi)M/d\rceil$ times.  Besides,  the number of iterations for this algorithm can be $O(d\log\log d)$, which shows the efficiency. See Section 21.2 in \citet{lattimore2020bandit} for more details.
\begin{lemma}[D-optimal design concentration \citep{lattimore2020bandit}]\label{lemma:Dconcentration}
Let $\cA$ be the active arm set. If the agent performs a D-optimal design exploration  on $\cA$ with rate $M$ defined in \Cref{definition:D-optimal-design}, we can get the  concentration results for related least squares estimators: for any $\xb\in\cX$, with probability at least $1-1/(KT^2)$,
   \begin{align*}
       \vert\la {\xb},\hat\btheta-\btheta^*\ra\vert\leq \sqrt{\frac{d\log(KT^2)}{M}}.
   \end{align*}
\end{lemma}
\begin{lemma}[\cite{jedra2020optimal}]\label{lemma:stopping-rule}
    Let $\delta\in (0,1)$, $u>0$.  Regardless of the sampling strategy,  the Chernoff's Stopping Rule 
    defined as $$Z(t)\geq \beta(t,\delta) \text{ and } \sum_{s=1}^t{\xb}_s{\xb}_s^\top \geq c\Ib_d$$
    with $c=\max_{\xb\in\cX}\Vert \xb\Vert_2^2$ and 
    \begin{align*}
        Z(t)&=\min_{\bbb\neq \hat \xb^*}\frac{\hat\Delta_\bbb^2}{2\Vert \hat \xb^*-\bbb\Vert^2_{{\Hb}_t^{-1}}},\\
        \beta(t,\delta)&=(1+u)\log\bigg(\frac{u^{-\frac d2}t^{\frac d2}}{\delta}\bigg).
    \end{align*}
    ensures that $\PP\big(\tau<\infty, \la{\btheta^*}, {\xb}^*-\hat {\xb}^*_\tau\ra>0\big)\leq \delta$, where $\hat{\bx}^*_\tau$ is the estimated best arm at time step $\tau$. This shows that if the stopping rule holds, we can find the best arm correctly with probability of at least $1-\delta$.
\end{lemma}
Note that we need different concentrations for each batches, with each just based on  one batch of independent data. This makes the proof much easier.
The subsequent lemmas are crucial as they assist in demonstrating  the asymptotic properties of \Cref{algorithm:optimal-algorithm}.
\Cref{lemma:basic-concentration-E-1} gives a concentration result in case that the agents use at least a D-optimal design exploration with rate $(\log T)^{1/2}$.
\begin{lemma} \label{lemma:basic-concentration-E-1}
Suppose in one batch of \Cref{algorithm:optimal-algorithm}, the agent conducts at least a D-optimal design exploration with rate $(\log T)^{1/2}$ given in \Cref{definition:D-optimal-design}. This means (1) $t_1=\Theta\big( (\log T)^{1/2}\big)$ pulls for D-optimal design exploration, (2) possibly some other pulls.
Let $t\geq t_1$  be the total number of pulls in this batch,
$\hat\btheta_{t_1},\hat\btheta_t$ be the least squares estimators of $\btheta^*$ after the D-optimal design exploration and the whole batch of exploration,  $\mu_{\xb}=\la\btheta^*,\xb\ra$, $\hat\mu_{\xb}(t)=\la \hat\btheta_{t},\xb\ra$ for $\xb\in\cX$, $\epsilon=1/\log\log T$. Define an event
\begin{align}\label{eq:basic-concentration-E-1}
    \cE_1&=\{\forall t\geq t_1,\forall \xb\in\cX,\vert \hat\mu_{\xb}(t)-\mu_{\xb}\vert\leq \epsilon\}.
\end{align} 
Then for sufficiently large $T$, we have
    \begin{align*}
        \PP(\cE_1^c)&\leq \frac{1}{(\log T)^2}.
    \end{align*}
\end{lemma}

 Next \Cref{lemma:concrete-concentration-first-two-batches-E-2} describes the exact concentration result for the first two batches in our \Cref{algorithm:optimal-algorithm}. Note that in the estimation stage of the $\ell$-th batch, we only use data collected in the exploration stage of this batch to calculate the least squares estimators.
\begin{lemma}\label{lemma:concrete-concentration-first-two-batches-E-2}
   Let $\epsilon=1/\log\log T$, $b_1$ and $b_2$ be the batch size of the first two batches in \Cref{algorithm:optimal-algorithm}, respectively. For the first two batches, we use $\hat\mu_{\xb}(b_1)=\la \hat\btheta_{b_1},\xb\ra$ and $\hat\mu_{\xb}(b_2)=\la \hat\btheta_{b_2},\xb\ra$ to denote the estimated expected rewards for $\xb\in\cX$ using data collected in the first and the second batches, respectively.
   Define 
    \begin{align}\label{eq:concentration-first-two-batches-E-2}
        \cE_2=\{
                \forall \xb\in\cX,\vert \hat\mu_{\xb}(b_1)-\mu_{\xb}\vert\leq \epsilon,\vert \hat\mu_{\xb}(b_2)-\mu_{\xb}\vert\leq \epsilon
        \}.
    \end{align}
   Then for sufficiently large $T$, we have 
    \begin{align*}
        \PP(\cE_2^c)&\leq \frac{2}{(\log T)^2}.
    \end{align*} 
\end{lemma}
Now we talk about some properties of the optimal allocation solved from \Cref{definition-tracking-proportion}.
Without loss of generality, we can always assume the allocation for $\hat \xb^*$ solved from \Cref{definition-tracking-proportion} satisfies
\begin{align}\label{eq:proportion-for-the-best-arm}
    w_{\hat \xb^*}= (\log T)^\gamma /\alpha.
\end{align}
In fact, if we have one solution $\wb(\hat \Delta)\in [0,\infty)^K$ to the program in \Cref{definition-tracking-proportion}, then we can increase the value $w_{\hat \xb^*}$ to $w_{\hat\xb^*}'\geq w_{\hat \xb^*}$ to get another solution. This is because the increase of $w_{\hat  {\xb}^*}$ doesn't affect the objective function $\sum_{\xb\neq \hat \xb^*} w_{\xb}(\hat \Delta_{\xb}-4\epsilon)$. Moreover, this increase can minimize $\Vert \xb- \hat \xb^*\Vert^2_{{\Hb}^{-1}_w}$ for  $\xb\neq \hat \xb^*$, which  makes the constraint still true.

In the following  analysis we use 
$\{w_{\xb}\}_{\xb\in\cX}$ to denote $\wb(\hat\Delta)$ solved using  estimated gaps. 
Next definition gives the allocation solved from the true parameters, which we will use in our theoretic analysis. It's easy to see \Cref{definition-tracking-proportion} is an approximate  version of \Cref{definition-real-proportion}. Using the same statements as above, we can assume $w^*_{{\xb}^*}=\infty$. 

\begin{definition}\label{definition-real-proportion}
    Let  $\Delta=\{\Delta_{\xb}\}_{\xb\in\cX}\in[0,\infty)^K$ be the vector of gaps,
    ${\xb}^*=\argmax_{\xb\in\cX}\la\btheta^*,\xb\ra$ be the best arm.
    We define $\wb(\Delta)=\{w^*_{\xb}\}_{\xb\in\cX}\in[0,\infty]^K$ to be the solution to the optimisation problem
    \begin{align*}
    \min_{\wb\in [0,\infty]^K}&\sum_{  \xb\in\cX^-} w_{\xb}\Delta_{\xb} \\ \text{s.t.} \quad\Vert  {\xb}-{\xb}^*\Vert^2_{{\Hb}^{-1}_\wb}&\leq\frac{\Delta_{\xb}^2}{2}, \quad\forall   \xb\in\cX^-,\\
\end{align*}where $\Hb_\wb=\sum_{\xb\in\cX}w_{\xb}\cdot \xb\xb^\top$.
\end{definition}
The following lemma shows the solution from \Cref{definition-tracking-proportion} converges to the solution from \Cref{definition-real-proportion} as $T\rightarrow\infty$.
\begin{lemma}\label{lemma:convergence-proportion-w}
    We use $\wb(\hat\Delta)=\{w_{\xb}\}_{\xb\in\cX}$ from \Cref{definition-tracking-proportion} and $\wb(\Delta)=\{w^*_{\xb}\}_{\xb\in\cX}$ from \Cref{definition-real-proportion}. Then if $\cE_2$ holds, 
    \begin{align*}
        \lim_{T\rightarrow\infty}w_{\xb}=w^*_{\xb},\quad \forall \xb\in\cX.
    \end{align*}
\end{lemma}
The following lemmas demonstrate that for sufficiently large $T$ the
stopping rule in the second batch of \Cref{algorithm:optimal-algorithm} holds with probability larger than $1-1/(\log T)^2$, and \textbf{while} loop  breaks after the third batch with probability larger than $1-1/T$.
\begin{lemma}\label{lemma:break-of-the-second-batch}
   If $\cE_2$ holds, then  for sufficiently large $T$, the stopping condition \eqref{eq:stopping-rule} holds in the second batch.
   
    Namely, \Cref{algorithm:optimal-algorithm} skip the \textbf{while} loop after the second batch and enter the \textit{Exploitation} stage with probability of at least  $1-2/(\log T)^2$, since $\PP(\cE_2^c)\leq 2/(\log T)^2$.
\end{lemma}

\begin{lemma}\label{lemma-break-thrid-batch}
    If the stopping rule \eqref{eq:stopping-rule} doesn't hold in the second batch, for sufficiently large $T$, the algorithm would eliminate all suboptimal arms in the third batch with probability of at least $1-1/T$.
\end{lemma}
Let $Regret_{2}$ be the regret for batch $\ell=2$ of \Cref{algorithm:optimal-algorithm}. The following lemma implies that $Regret_{2}$ matches the leading term in the asymptotic lower bound in \Cref{lemma:asym-lower-bound}.
\begin{lemma}\label{regretII}
    Regret of the second batch can be bounded by: 
    \begin{align*}
        \lim_{T\rightarrow\infty} \frac{Regret_{2}}{\log T}\leq c^*,
    \end{align*}
   where $c^*=c^*(\btheta^*)$ defined in \Cref{lemma:asym-lower-bound}. 
\end{lemma}
\subsection{Proof of \Cref{theorem:mainthem-optimal-algorithm}}
Now we prove the results in \Cref{theorem:mainthem-optimal-algorithm}. Note that under the  choice of exploration rate $\cT_1$, we want to show that \algname\ achieves the optimal regret and the batch complexity in both the finite-time setting and the asymptotic setting.

\subsubsection{The Finite-time Setting}
We first prove the batch complexity and the worst case regret bound in the finite-time setting.

\textbf{Batch complexity:}
In this part, we analyze the upper bound of batch complexity. It is sufficient to count the number of \textbf{while} loops starting  from  the third batch in \Cref{algorithm:optimal-algorithm}  before \textit{Exploitation} in  Line \ref{algline:exploitation}. For $\ell\geq 4$, the $\ell$-th batch has size $\Theta\big(T^{1-\frac{1}{2^{\ell-3}}}\big)$, so the $(\log\log T+3)$-th batch has size $\Theta(T)$. Suppose the $(\log\log T+3)$-th batch size is larger than $T/M$, where $M>0$ is a constant. By definition, the batch size increases as $t$ increases, so the sizes of batches after this batch is larger than $T/M$.  
Therefore, the algorithm would have at most $\log\log T+3+M$ batches.
Then we prove the  batch complexity upper bound $O(\log\log T)$.

\textbf{Minimax  optimality:}
First, we assume the agent doesn't eliminate any arms in the first two batches, because it's easier the other way as we will show at the end of this part of the proof. Since the number of time steps for the first two batches is at most $O\big((\log T)^{1+\gamma}\big)$, the regret of the first two batches is at most $O\big((\log T)^{1+\gamma}\big)$, due to the bounded rewards assumption. Besides, the property of Chernoff's Stopping Rule (\Cref{lemma:stopping-rule}) means the agent finds the best arm correctly with probability larger than $1-1/T$, then the regret of Line \ref{algline:exploitation} \textit{Exploitation} stage can be bounded by a small constant.  Hence we don't need to consider the regret in Line \ref{algline:exploitation} \textit{Exploitation} stage.
  
In \Cref{algorithm:optimal-algorithm}, we let $\cA$ be the current active arm set. Note that $\cA$ may shrink after some batches due to the \textit{Elimination} stages. We define an arm $\xb$ to be active in batch $\ell$, if it is in the active arm set $\cA$ of this batch.
 The same as \cite{esfandiari2021regret} and \cite{lattimore2020bandit}, we define the following \textit{good event} $\cE_0$ \eqref{eq:good-event-in-elimination-algorithm} that we will use in the following proof constantly: 
    \begin{align}\label{eq:good-event-in-elimination-algorithm}
        \cE_0=\{\text{For any arm }\xb \text{ that is active in the beginning of batch } \ell, \\
        \text{at the end of this batch we have } \vert \la\xb,\hat\btheta-\btheta^*\ra\vert\leq \notag\varepsilon_\ell.\}
    \end{align}

    By \Cref{definition:D-optimal-design}, for any $\xb\in\cX$, after the $\ell$-th batch, with probability at least $1-1/(KT^2)$, we have
    \begin{align*}
        \vert\la \xb,\hat\btheta-\btheta^*\ra\vert\leq \varepsilon_\ell.
    \end{align*}
    Since   there are $K$ arms and at most $T$ batches, by union bound we know the good event $\cE_0$ \eqref{eq:good-event-in-elimination-algorithm} happens with probability at least $1-1/T$. %
    Consequently, we know
    \begin{align*}
        R_T&=\EE\Bigg[\max_{\xb\in\cX}\sum_{t=1}^T\la {\xb}-{\xb}_t,\btheta^*\ra\bigg\vert \cE_0^c \Bigg]\cdot\PP(\cE_0^c)+\EE\Bigg[\max_{\xb\in\cX}\sum_{t=1}^T\la {\xb}-{\xb}_t,\btheta^*\ra\bigg\vert \cE_0 \Bigg]\cdot\PP(\cE_0)\\
        &\leq 2LT\PP(\cE_0^c)+\EE\Bigg[\max_{\xb\in\cX}\sum_{t=1}^T\la {\xb}-{\xb}_t,\btheta^*\ra\bigg\vert \cE_0 \Bigg]\cdot\PP(\cE_0)\\
        &= O(1)+\EE\Bigg[\max_{\xb\in\cX}\sum_{t=1}^T\la {\xb}-{\xb}_t,\btheta^*\ra\bigg\vert \cE_0 \Bigg]\cdot\PP(\cE_0).
    \end{align*}
    So we can just bound the regret when $\cE_0$ happen. Under $\cE_0$, at the end of each batch $\ell\geq 3$, we know for each arm,
    \begin{align*}
        \vert \la {\xb},\hat\btheta-\btheta^*\ra\vert\leq \varepsilon_\ell.
    \end{align*}

    Denote the best arm and estimated best arm {(ties broken arbitrarily)} by:
    \begin{align*}
        {\xb}^*&=\argmax_{\xb\in\cA} \la {\xb},\btheta^*\ra,\\
       \hat \xb^*&=\argmax_{\xb\in\cA} \la {\xb},\hat\btheta\ra.
    \end{align*}
    Then,
    \begin{align*}
        \max_{y\in\cA} \la \hat\btheta, \yb-\xb^*\ra&= \la \hat\btheta, \hat \xb^*-{\xb}^*\ra\\
        &=  \la \hat\btheta-\btheta^*+\btheta^*, \hat \xb^*-{\xb}^*\ra\\
        &= \la  \hat\btheta-\btheta^*, \hat \xb^*\ra -  \la  \hat\btheta-\btheta^*,  {\xb}^*\ra + \la \btheta^*,\hat \xb^*-{\xb}^*\ra\\
        &\leq \vert  \la  \hat\btheta-\btheta^*, \hat \xb^*\ra\vert +  \vert \la  \hat\btheta-\btheta^*,  {\xb}^*\ra\vert  + \la \btheta^*,\hat \xb^*-{\xb}^*\ra\\
        &\leq 2\varepsilon_\ell+\la \btheta^*,\hat \xb^*-{\xb}^*\ra\\
        &\leq 2\varepsilon_\ell.
    \end{align*}    
   Compared this with the elimination rule in \Cref{algorithm:optimal-algorithm} for $\ell\geq 3$, we know the best arm won't be eliminated, namely,
    \begin{align*}
         {\xb}^*\in \cA_{\ell}, \quad\forall \ell\geq 3,
    \end{align*}here we use $\cA_\ell$ to denote the active action set in the $\ell$-th batch.
    Correspondingly, for each suboptimal arm $\xb$, define $\ell_{\xb}=\min\{\ell:4\varepsilon_\ell<\Delta_{\xb} \}$ to be the first phase where the suboptimality gap of arm $\xb$  bigger than $4\varepsilon_\ell$. Then if $\xb\neq {\xb}^*$ is not eliminated in the first $\ell-1$ batches, in the $\ell$-th batch,
    \begin{align*}
         \max_{y\in\cA} \la \hat\btheta, \yb-\xb\ra&\geq \la \hat\btheta, {\xb}^*-\xb\ra\\
         &= \la \hat\btheta-\btheta^*+\btheta^*, {\xb}^*-\xb\ra\\
         &= \la \hat\btheta-\btheta^*, {\xb}^*-\xb\ra+\la \btheta^*, {\xb}^*-\xb\ra\\
         &\geq -\vert  \la  \hat\btheta-\btheta^*, \hat \xb^*\ra\vert -  \vert \la  \hat\btheta-\btheta^*,x \ra\vert+\Delta_{\xb}\\
         &> 2\varepsilon_\ell.
    \end{align*}
     Compared this with the elimination rule in \Cref{algorithm:optimal-algorithm} for $\ell\geq 3$, we know $\xb\neq {\xb}^*$ would be eliminated in the first $\ell$ batches, namely,
    \begin{align}
        x\notin\cA_{\ell_{\xb}+1},\quad\forall \xb\neq {\xb}^*,\label{eq:eliminate}
    \end{align}where $\cA_{\ell_{\xb}+1}$ is the active action set in the $\ell_{\xb}+1$ batch.
   Thus when the good event $\cE_0$ \eqref{eq:good-event-in-elimination-algorithm} happens, arms active in phase $\ell$ have the property: $\Delta_{\xb}\leq 4\varepsilon_{\ell-1}$.

    When $T_3=(\log T)^{1+\gamma} ,T_\ell=T^{1-\frac{1}{2^{\ell-3}}} ,\ell\geq 4$ and $ \varepsilon_\ell=\sqrt{d\log(KT^2)/T_\ell}$,  \Cref{algorithm:optimal-algorithm} has at most $B=O(\log\log T)$ batches. We use $Regret_\ell$ to denote the regret of the $\ell$-th batch, then the regret  for batches $\ell\geq 5$ under $\cE_0$ can be bounded as
    \begin{align*}
      \sum_{\ell=5}^B  Regret_{\ell}&\leq\sum_{\ell=5}^B4T_\ell\varepsilon_{\ell-1}\\
      &\leq \sum_{\ell=5}^B4C T^{1-\frac{1}{2^{\ell-3}}} \sqrt{d\log(KT^2)/T^{1-\frac{1}{2^{\ell-4}}}} \\
        &=  \sum_{\ell=5}^B 4C\sqrt{dT\log(KT^2)}\\
        &= O\big(\log\log T\cdot\sqrt{dT\log(KT)}\big).
    \end{align*}
    As a result, the regret  for batches $\ell\geq 3$ under $\cE_0$ can be bounded  as %
    \begin{align*}
      \sum_{\ell=3}^B  Regret_{\ell}&\leq\sum_{\ell=3}^B4T_\ell\varepsilon_{\ell-1}\\
      &=O(T_3)+O(T_4) +\sum_{\ell=5}^B4T_\ell\varepsilon_{\ell-1}\\
        &=  O\big((\log T)^{1+\gamma}\big)+O\big(\sqrt{T}\big)+\sum_{\ell=5}^B 4C\sqrt{dT\log(KT^2)}\\
        &= O\big(\log\log T\cdot\sqrt{dT\log(KT)}\big).
    \end{align*}
All of the proofs above suppose the algorithm doesn't eliminate any arms in the first two batches, i.e. the stopping rule in the second batch doesn't hold. Now we suppose the algorithm eliminates all suboptimal arms after the second batch. The property of Chernoff's Stopping Rule \Cref{lemma:stopping-rule} means the agent finds the optimal arm with probability at least $1-1/T$, then the regret of Line \ref{algline:exploitation} \textit{Exploitation} after the second batch can be bounded by a constant.  We can combine these to end the proof.

\subsubsection{The Asymptotic Setting}
Now we investigate the batch complexity and the regret bound of \algname\ when $T\rightarrow \infty$.

\textbf{Batch complexity:} By \Cref{lemma:break-of-the-second-batch}, we know for sufficiently large $T$,
with probability $1-2/(\log T)^2$, the stopping rule \eqref{eq:stopping-rule} in the second batch of \Cref{algorithm:optimal-algorithm}  holds.
This  means the algorithm identifies the best arm  and goes into the Line \ref{algline:exploitation}, the \textit{Exploitation} stage. According to the last parts of proof, with time grid selections $\cT_1$ (and even $\cT_2$), the algorithm has at most $O(\log T)$ batches in total. So we can calculate the expected batch complexity by
\begin{align}\label{eq:batch-complexity-is-3}
\lim_{T\rightarrow\infty} 3\cdot  \big(1-2/(\log T)^2\big)+O\big(\log T/(\log T)^2\big)= 3,
\end{align} 
because the algorithm just has $3$ batches when the stopping rule holds, and this event happens with probability at least $1-2/(\log T)^2$.
Thus we complete the proof of asymptotic batch complexity.

\textbf{Asymptotic optimality:}
\Cref{lemma:break-of-the-second-batch,lemma-break-thrid-batch} demonstrate that for sufficiently large $T$ the stopping rule in the second batch of \Cref{algorithm:optimal-algorithm} holds with probability larger than $1-2/(\log T)^2$, and the \textbf{while} loop starting from batch $\ell=3$ breaks after the third batch with probability larger than $1-1/T$.

Decompose accumulative regret by batches as:
\begin{align*}
    R_T&=Regret_1+Regret_{2}+Regret_{3}+Regret_{else},
\end{align*}
where $Regret_\ell$ represents the regret in the $\ell$-th batch ($\ell=1,2,3$), and $Regret_{else}$ represents the regret after the third batch.

We research on the asymptotic setting, so all of the statements below assumes that $T$ is sufficiently large. 
With the assumption of bounded rewards, regret in the first batch is negligible compared with  $\Theta(\log T)$ order, namely, $Regret_1=o(\log T)$:
\begin{align*}
    \lim_{T\rightarrow\infty}\frac{Regret_1}{\log T}=\frac{(\log T)^{1/2}}{\log T}=0.
\end{align*}
Note that the stopping rule \eqref{eq:stopping-rule} in the second batch holds with probability at least $1-2/(\log T)^2$. But the batch $\ell=3$ in the \textbf{while} loop is encountered only when the stopping rule fails to hold. Given that the batch $3$ has a size of $O\big((\log T)^{1+\gamma}\big)$, the regret in this batch is considered negligible:
\begin{align*}
   \lim_{T\rightarrow\infty} Regret_3=\lim_{T\rightarrow\infty} \frac{2}{(\log T)^2}\cdot O\big((\log T)^{1+\gamma}\big)=0.
\end{align*}
Similarly, note that the agent eliminates all suboptimal arms after the third batch  with probability $1-1/T$, according to \Cref{lemma-break-thrid-batch}, so the regret after this batch $Regret_{else}$ is negligible, too:
\begin{align*}
    \lim_{T\rightarrow\infty}\frac{Regret_{else}}{\log T}=\lim_{T\rightarrow\infty}\frac{ \frac{1}{T}\cdot O(T)}{\log T}=0.
\end{align*}

 The only term of care about is the regret of the second batch of \Cref{algorithm:optimal-algorithm}. Therefore, it is sufficient to use \Cref{regretII} to conclude the proof of asymptotically optimal regret, which shows
 \begin{align*}
        \lim_{T\rightarrow\infty} \frac{Regret_{2}}{\log T}\leq c^*.
    \end{align*}
Now we can conclude our prove of the main \Cref{theorem:mainthem-optimal-algorithm} by combining all these results about different parts of regret:
\begin{align*}
    \limsup_{T\rightarrow\infty}\frac{R_T}{\log T}&=\limsup_{T\rightarrow\infty}\frac{Regret_1+Regret_{2}+Regret_{3}+Regret_{else}}{\log T}\\
&=\limsup_{T\rightarrow\infty}\frac{Regret_{2}}{\log T}\\
    &\leq c^*,
\end{align*}which matches the lower bound in \Cref{lemma:asym-lower-bound}.

\subsection{Proof of \Cref{theorem:instance-optimal-algorithm}}
\noindent
\textbf{Finite-time instance-dependent batch complexity and regret bound:}
This part of proof is similar to the analysis in the proof of \Cref{theorem:mainthem-optimal-algorithm}. We  use the same good event  $\cE_0$ in \eqref{eq:good-event-in-elimination-algorithm}. Here we choose $T_\ell=\Theta(d\log(KT^2)\cdot2^{\ell-3})$, $ \varepsilon_\ell=\sqrt{1/2^{\ell-3}},\ell\geq 4$. Then the $\log T$-th batch in \Cref{algorithm:optimal-algorithm} has size $\Theta(d\log(KT^2)T)$. Therefore, the \textbf{while} loop in the algorithm would break with batches at most   $O(\log T)$. 

In this part we can also suppose the agent doesn't eliminate any arms in the first two batches. The reason for this is the same as  the  proof of \Cref{theorem:mainthem-optimal-algorithm}. Similarly, the regret of the first three batches is at most $O\big((\log T)^{1+\gamma}\big)$ and the regret of Line \ref{algline:exploitation} \textit{Exploitation} stage is $O(1)$, which is negligible.
To analyze the instance-dependent regret, we  focus on batches for $\ell\geq 4$. Define $\Delta_{\min}=\min_{\xb\neq {\xb}^*}\Delta_{\xb}$. Recall that under the good event $\cE_0$ \eqref{eq:good-event-in-elimination-algorithm}, the arm $\xb\neq {\xb}^*$ would be eliminated after the $\ell_{\xb}$ batch. We can explicitly calculate  when would all suboptimal arms be eliminated. Solving 
\begin{align}\label{eq:instance-elimination}
    \Delta_{\min}>4\varepsilon_\ell
\end{align} we get
\begin{align}\label{eq:instance-elimination-result}
    \ell>\log\bigg(\frac{1}{\Delta_{\min}^2}\bigg)+7.
\end{align}
Define $L_0=\Big\lceil \log\Big(\frac{1}{\Delta_{\min}^2}\Big)\Big\rceil+7=O\big(\log(\frac{1}{\Delta_{\min}})\big)$. We just need to consider regret in the first $L_0$ batches. Because after that the only active action is the optimal one.

Absolute the same as the proof of \Cref{theorem:mainthem-optimal-algorithm}, we have 
\begin{align}
   \sum_{\ell=4}^{L_0} Regret_{\ell}&\leq\sum_{\ell=4}^{L_0}4T_\ell\varepsilon_{\ell-1}\notag\\
    &\leq \sum_{\ell=4}^{L_0}4C\cdot d\log(KT^2)\cdot 2^{\ell-3} \sqrt{1/2^{\ell-4}}\notag \\
    &= \sum_{\ell=4}^{L_0}8C\cdot d\log(KT^2)2^{(\ell-4)/2} \notag\\
    &\leq 8C \cdot d\log(KT^2)2^{(L_0-3)/2}\label{eq:T-times-eps} \\
    &= O(d\log(KT)/\Delta_{\min}).\notag
\end{align} 

Combining with the regret of the first three batches, we have:
\begin{align}\label{eq:instance-dependent-bound}
    R_T= O\big((\log T)^{1+\gamma}+d\log(KT)/\Delta_{\min}\big).
\end{align} 
Since the good event $\cE_0$ \eqref{eq:good-event-in-elimination-algorithm} happens with probability at least $1-1/T$, the expected batch complexity of the algorithm is
\begin{align*}
    O\bigg(\log\bigg(\frac{1}{\Delta_{\min}}\bigg)+(\log T)/T \bigg)=O\bigg(\log\bigg(\frac{1}{\Delta_{\min}}\bigg)\bigg).
\end{align*}

~\\
\noindent
\textbf{Minimax  optimality:}
We can also prove $R_T=\tilde O(\sqrt{dT})$, which implies the algorithm with $\{T_\ell\}_{\ell=1}^\infty=\cT_2$ is nearly minimax optimal.

When $\Delta_{\min}> \sqrt{d/T}$, \eqref{eq:instance-dependent-bound} is a tighter bound compared to $\tilde O(\sqrt{dT})$. So we just need to consider the case $\Delta_{\min}\leq \sqrt{d/T}$.
Here the same as the methods in \eqref{eq:instance-elimination}-\eqref{eq:instance-elimination-result}, we can prove that under $\cE_0$, arms with gap larger than $\sqrt{d/T}$ would be eliminated in the first $L'=\big\lceil\log(T/d)\big\rceil+5$ batches.
However, the regret caused by arms with gaps less than $\sqrt{d/T}$ is at most 
$T\cdot \sqrt{d/T}=\sqrt{dT}$. Note that the regret for batches $\ell= 4,\ldots,L'$ in the \textbf{while} loop  can be bounded by
\begin{align}
    \sum_{\ell=4}^{L'} Regret_{\ell}
    &\leq \sum_{\ell=4}^{L'}4T_\ell \varepsilon_{\ell-1}\notag\\
    &\leq 8C\cdot d\log(KT^2)2^{(L'-3)/2}\label{eq:T-times-varepsilon}\\
    &= O\big(\sqrt{dT}\cdot\log(KT)\big),\notag
\end{align}where we get \eqref{eq:T-times-varepsilon} the same as \eqref{eq:T-times-eps}.
As a result, the total regret can be bounded by
\begin{align}
    R_T&\leq \sum_{\ell=1}^{3}Regret_{\ell}+\sum_{\ell=4}^{L'} Regret_{\ell}+\sqrt{dT}\\
    &=O\big((\log T)^{1+\gamma}\big)+\sum_{\ell=4}^{L'} Regret_{\ell}+\sqrt{dT}\notag\\
    &= O\big(\sqrt{dT}\cdot\log(KT)\big).\notag
\end{align}

~\\
\noindent
\textbf{Asymptotic regret and batch complexity:}
When  $\{T_\ell\}_{\ell=1}^\infty$ is set to either $\cT_1$ or $\cT_2$, \Cref{algorithm:optimal-algorithm} yields identical performances   within the initial three batches. Thus the corresponding portion of the proof aligns with that of \Cref{theorem:mainthem-optimal-algorithm}. Specifically, according to \Cref{lemma:break-of-the-second-batch,lemma-break-thrid-batch}, for a sufficiently large $T$, the probability that the algorithm avoids batch $\ell=3$ exploration at least $1-2/(\log T)^2$. Similarly, the probability that it avoids batch $\ell=4$ is at least $1-1/T$. By utilizing the same rationale, we can  complete this part of the proof.

\section{Proofs of Technical Lemmas}
\subsection{Proof of \Cref{lemma:o-logT-not-enough}}

We prove the statement by contradiction. Suppose there is an algorithm $\pi$ such that for sufficient large $T$
\begin{align*}
     \PP(\tilde \xb\neq {\xb}^*)\leq1/\sqrt{T}.
\end{align*}
The algorithm terminates after $b=o(\log T)$ steps and output the estimated best arm $\tilde \xb$. By the assumption above we know $\PP(\tilde \xb \neq {\xb}^*)\leq1/\sqrt{T}$, so this is a $\delta$-PAC algorithm with $\delta=1/\sqrt T$.

However, by the Theorem 1 in \citet{jedra2020optimal}, for a $(1/\sqrt T)$-PAC strategy, any bandit instance $\btheta^*$ and sufficiently large $T$,
\begin{align*}
    b\geq T^*(\btheta^*)\log\big(\sqrt{T}\big)=\Theta(\log T),
\end{align*}which is contradictory with $b=o(\log T)$. Therefore, the initial assumption is incorrect. We can now conclude the proof of $\PP(\tilde\xb\neq {\xb}^*)>1/\sqrt{T}$ for infinite many $T$.

\subsection{Proof of \Cref{lemma:concrete-concentration-first-two-batches-E-2}}

  Here we can use \Cref{lemma:basic-concentration-E-1} directly. 
  
  In the first batch of \Cref{algorithm:optimal-algorithm}, the agent just use a D-optimal design exploration given in \Cref{definition:D-optimal-design} with $b_1=t_1=\Theta\big( (\log T)^{1/2}\big)$ pulls. In the second batch of \Cref{algorithm:optimal-algorithm}, the agent use a D-optimal design exploration given in \Cref{definition:D-optimal-design} with $b_1=t_1=\Theta\big( (\log T)^{1/2}\big)$ pulls, and pull arms $\xb\in\cX$ for another $\min\big\{w_{\xb} \cdot \alpha\log T,(\log T)^{1+\gamma}\big\}$ times. The behavior of the first two batches  satisfies the conditions of \Cref{lemma:basic-concentration-E-1}. Therefore, by the conclusion of \Cref{lemma:basic-concentration-E-1},
  \begin{align*}
      \PP(\forall \xb\in\cX,\vert \hat\mu_{\xb}(b_1)-\mu_{\xb}\vert\leq \epsilon)\geq 1- \frac{1}{(\log T)^2},\\
      \PP(\forall \xb\in\cX,\vert \hat\mu_{\xb}(b_2)-\mu_{\xb}\vert\leq \epsilon)\geq 1-\frac{1}{(\log T)^2}.
  \end{align*}
Hence by union bound we have
\begin{align*}
    \PP(\cE_2)= \PP(\forall \xb\in\cX,\vert \hat\mu_{\xb}(b_1)-\mu_{\xb}\vert\leq \epsilon,\vert \hat\mu_{\xb}(b_2)-\mu_{\xb}\vert\leq \epsilon)&\geq 1-\frac{2}{(\log T)^2},
\end{align*}as needed.

\subsection{Proof of \Cref{lemma:convergence-proportion-w}}

    First, if $\cE_2$ holds, for sufficiently large $T$, $\epsilon=1/\log\log T<\Delta_{\min}/2$, so the estimated best arm is really the best arm, namely, $\hat \xb^*={\xb}^*$. 
    Then 
    \begin{align*}
    \lim_{T\rightarrow\infty}w_{{\xb}^*}=\lim_{T\rightarrow\infty} (\log T)^{\gamma}/\alpha=\infty=w^*_{{\xb}^*}.
    \end{align*}Besides, as $T\rightarrow\infty$, $(\log T)^\gamma/\alpha\rightarrow\infty$, so the last restricted condition $w_{\hat \xb^* }\leq(\log T)^\gamma/\alpha$ in the optimisation problem in \Cref{definition-tracking-proportion} has no effect.
    In addition, by the concentration result of $\cE_2$, we know 
    \begin{align*}
        \lim_{T\rightarrow\infty}\hat\Delta_{\xb}-4\epsilon=\Delta_{\xb},\quad\forall \xb\in\cX.
    \end{align*}Then we can compare the two programs in \Cref{definition-tracking-proportion} and \Cref{definition-real-proportion} line by line.
 By the continuity of the program we know 
    \begin{align*}
        \lim_{T\rightarrow\infty}w_{\xb}=w^*_{\xb},\quad \forall \xb\in\cX.
    \end{align*}
    
\subsection{Proof of \Cref{lemma:break-of-the-second-batch}}
If $\cE_2$ holds,
\begin{align}
    \vert \hat\mu_{\xb}(b_1)-\mu_{\xb}\vert\leq \epsilon,\vert \hat\mu_{\xb}(b_2)-\mu_{\xb}\vert\leq \epsilon,\quad \forall \xb\in\cX.\label{eq:proof-E-2}
\end{align}
 For sufficiently large $T$, $\epsilon=1/\log\log T<\Delta_{\min}/2$.
 Hence by \eqref{eq:proof-E-2}, for sufficiently large $T$ and $\xb\neq \xb^*$,
 \begin{align*}
     \hat\mu_{\xb^*}(b_1)-\hat\mu_{\xb}(b_1)
     &=\hat\mu_{\xb^*}(b_1)-\mu_{\xb^*}-(\hat\mu_{\xb}(b_1)-\mu_{\xb})+\Delta_{\xb}\\
     &\geq -\vert \hat\mu_{\xb^*}(b_1)-\mu_{\xb^*}\vert -\vert\hat\mu_{\xb}(b_1)-\mu_{\xb}\vert +\Delta_{\xb}\\
     &\geq -2\epsilon+\Delta_{\xb}>0, 
 \end{align*}which implies the estimated best arm is really the best arm, namely, $\hat \xb^*={\xb}^*$.
And for  suboptimal arm $\xb\neq\xb^*$, we have
\begin{align}\label{eq:estimated-gap}
    \vert\hat\Delta_{\xb}-\Delta_{\xb}\vert
    =\vert (\hat \mu_{{\xb}^*}-\hat\mu_{\xb})-(\mu_{{\xb}^*}-\mu_{\xb})\vert
    \leq \vert \hat \mu_{{\xb}^*}-\mu_{{\xb}^*}\vert + \vert \hat\mu_{\xb}-\mu_{\xb}\vert
    \leq 2\epsilon.
\end{align}

The second stopping rule $\sum_{s=1}^t{\xb}_s{\xb}_s^\top \geq c\Ib_d$ holds easily for sufficiently large $T$, with the  D-optimal design exploration in the second batch, because features of arms used in  D-optimal design exploration spans $\RR^d$ \citep{lattimore2020bandit}.

Now, all that's left to do is to calculate the stopping statistic $Z(b_2)$ defined in \eqref{eq:stopping-statistic-Z-t} and to show it is larger than the  threshold value $\beta(b_2,1/T)$ defined in \eqref{eq:beta-definition} as $T\rightarrow\infty$. 

After  the first batch, the agent calculates the allocation $\wb(\hat\Delta)=\{w_{\xb}\}_{\xb\in\cX}$  according to \Cref{definition-tracking-proportion},
where we use $\hat\Delta$ to denote estimated gaps based on data in the first batch. Correspondingly, we define $\wb(\Delta)=\{w^*_{\xb}\}_{\xb\in\cX}$ to denote the proportion solved from \Cref{definition-real-proportion} using the true instance parameter $\btheta^*$. Under the concentration event $\cE_2$ and  \Cref{lemma:convergence-proportion-w}, we know $\lim_{T\rightarrow\infty}\hat\Delta_{\xb}= \Delta_{\xb}$, $\lim_{T\rightarrow\infty}w_{\xb}=w^*_{\xb},\forall \xb\in\cX$.

 Note that when $\ell=2$, since $w^*_{\xb}\text{ for }\xb\in\cX^-$ is independent of $T$,
\begin{align*}
\lim_{T\rightarrow\infty}\frac{w_{\xb}\cdot\alpha\log T}{(\log T)^{1+\gamma}}=
\lim_{T\rightarrow\infty}\frac{w^*_{\xb}\cdot\alpha\log T}{(\log T)^{1+\gamma}}=0, \xb\in\cX^-.
\end{align*}
 Thus for sufficiently large $T$,
 \begin{align*}
     w_{\xb}\cdot\alpha\log T\leq (\log T)^{1+\gamma}, \xb\in\cX^-.
 \end{align*}
 
 Besides, by \eqref{eq:proportion-for-the-best-arm} we know $w_{{\xb}^*}= (\log T)^\gamma/\alpha$. It follows that
 \begin{align*}
     w_{{\xb}^*}\cdot\alpha\log T= (\log T)^{1+\gamma}.
 \end{align*}
 
Therefore, the exploration rule in the second batch should be
 D-optimal design exploration with rate $T_2$, then pull each arm $\xb\in\cX$ for $ w_{\xb}\cdot\alpha\log T$ times. Here we use $b_2$ to denote the batch size of the second batch, $a_2=\sum_{\xb\in\cA}f(\xb,\cA,T_2)=\Theta\big((\log T)^{1/2}\big)$ to denote the total number of D-optimal design exploration steps in this batch given by \Cref{definition:D-optimal-design}. Consequently, we know
\begin{align}\label{eq:b2-the-second-batch-size}
    b_2=a_2+(\log T)^{1+\gamma}+\sum_{\xb\in\cX^-}w_{\xb}\cdot\alpha\log T.
\end{align}
Likewise, we can split the data $\{{\xb}_s\}_{s=1}^{b_2}$ into $\{{\xb}_s\}_{s=1}^{a_2}$ and $\{{\xb}_s\}_{s=a_2+1}^{b_2}$, with the former one denoting the arms in D-optimal design exploration and the latter one denoting the remaining arms according to the allocation $\wb$ in the second batch.

By the constraint condition in the convex program of \Cref{definition-tracking-proportion}, we have
\begin{align*}
    \Vert \xb- {\xb}^*\Vert^2_{{\Hb}^{-1}_w}=\Vert \xb-\hat \xb^*\Vert^2_{{\Hb}^{-1}_w}\leq\frac{(\hat\Delta_{\xb}-4\epsilon)^2}{2}, \forall \xb\in \cX^-.
\end{align*}

Then we get
\begin{align}
    \frac{(\hat\Delta_{\xb}-4\epsilon)^2/2}{\Vert \xb-  {\xb}^*\Vert^2_{{\Hb}^{-1}_w}}=\frac{(\hat\Delta_{\xb}-4\epsilon)^2/2}{\Vert \xb- \hat \xb^*\Vert^2_{{\Hb}^{-1}_w}}\geq 1, \forall \xb\in \cX^-.\label{eq:constant-lower-bound-for-Z}
\end{align}

In the second batch of exploration, the agent collects data and estimates the parameters. To avoid confusion, here we use $\bar\btheta$,  $\bar\Delta$ and $\bar {\xb}^*$ to denote related estimators based on data in the second batch, and they have the same concentration properties with  estimators before. Namely, $\lim_{T\rightarrow\infty}\bar\Delta=\Delta$ and $\bar {\xb}^*={\xb}^*$. And similar to \eqref{eq:estimated-gap}, we have
\begin{align*}
    \vert \bar \Delta_{\xb}-\Delta_{\xb}\vert \leq 2\epsilon, \forall \xb\in \cX^-.
\end{align*}
Further, combining this with \eqref{eq:estimated-gap}, we have for sufficiently large $T$,
\begin{align}\label{eq:4-epsilon}
     \bar \Delta_{\xb} \geq \hat \Delta_{\xb}-4\epsilon\geq 0, \forall \xb\in \cX^-,
\end{align}where the second inequality is because $\lim_{T\rightarrow\infty}\hat\Delta_{\xb}=\Delta_{\xb}>0,\forall \xb\in\cX^-$ and $\lim_{T\rightarrow\infty}\epsilon=\lim_{T\rightarrow\infty}1/\log\log T=0$.
We can calculate the stopping statistic given by \eqref{eq:stopping-statistic-Z-t}:
\begin{align*}
    Z(b_2)&=\min_{\xb\neq  \bar {\xb}^*}\frac{\bar\Delta_{\xb}^2}{2\Vert \xb-\bar {\xb}^*\Vert_{\Hb^{-1}_{b_2}}^2}=\min_{\xb\neq  {\xb}^*}\frac{\bar\Delta_{\xb}^2}{2\Vert \xb-{\xb}^*\Vert_{\Hb^{-1}_{b_2}}^2}.
\end{align*}
Here we get
\begin{align}
    \frac{Z(b_2)}{\alpha\log T}
    &=\frac{\min_{\xb\neq  {\xb}^*}\frac{\bar\Delta_{\xb}^2}{2\Vert \xb-{\xb}^*\Vert_{\Hb^{-1}_{b_2}}^2}}{\alpha\log T}\notag\\
    &\geq \min_{\xb\neq  {\xb}^*}\frac{(\hat \Delta_{\xb}-4\epsilon)^2}{2\cdot \alpha\log T \Vert \xb-{\xb}^*\Vert_{\Hb^{-1}_{b_2}}^2}\label{eq:4-epsilon-application}\\
    &=\min_{\xb\neq  {\xb}^*}\frac{(\hat \Delta_{\xb}-4\epsilon)^2}{2\big\Vert \xb-{\xb}^*\big\Vert_{{\big(\Hb_{b_2}/(\alpha\log T)\big)}^{-1}}^2}\notag\\
    &\geq \min_{\xb\neq  {\xb}^*}\frac{(\hat \Delta_{\xb}-4\epsilon)^2}{2\Vert \xb-{\xb}^*\Vert_{\Hb^{-1}_{w}}^2},\label{eq:last-eq-in-Z}
\end{align}
where \eqref{eq:4-epsilon-application} is because of \eqref{eq:4-epsilon}, and
\eqref{eq:last-eq-in-Z} holds because 
\begin{align*}
    \frac{\Hb_{b_2}}{\alpha\log T}
    &=\frac{\sum_{s=1}^{b_2}{\xb}_s{\xb}_s^\top}{\alpha \log T} \\
    &=\frac{1}{\alpha \log T}\bigg(\sum_{s=1}^{a_2}{\xb}_s{\xb}_s^\top +\sum_{s=a_2+1}^{b_2}{\xb}_s{\xb}_s^\top \bigg)\\
    &\geq \frac{1}{\alpha \log T}\sum_{s=a_2+1}^{b_2}{\xb}_s {\xb}_s^\top \\
    &=\frac{1}{\alpha\log T}\bigg( (\log T)^{1+\gamma} {\xb}^*{{\xb}^*}^\top+\sum_{\xb\in\cX^-}w_{\xb}\alpha\log T\cdot {\xb}_s{\xb}_s^\top \bigg)\\
    &=(\log T)^\gamma/\alpha\cdot {\xb}^*{{\xb}^*}^\top +\sum_{\xb\in\cX^-}w_{\xb} \cdot \xb\xb^\top\\
    &=\Hb_\wb.
\end{align*}
Then we can combine \eqref{eq:constant-lower-bound-for-Z} and \eqref{eq:last-eq-in-Z} to get
\begin{align}
    \frac{Z(b_2)}{\alpha\log T}\geq\min_{\xb\neq  {\xb}^*}\frac{(\hat\Delta_{\xb}-4\epsilon)^2}{2\Vert \xb-{\xb}^*\Vert_{\Hb^{-1}_{w}}^2}\geq 1.\label{eq:Zb2}
\end{align}

By the selection of $\alpha=(1+1/\log\log T)(1+d\log\log T/\log T)$ in \Cref{theorem:mainthem-optimal-algorithm}, we have
\begin{align}\label{eq:alpha-logT-selection}
    \alpha\log T=(1+1/\log\log T)(1+d\log\log T/\log T)\log T
    =(1+1/\log\log T)\log \bigg(\frac{(\log T)^{d}}{1/T} \bigg).
\end{align}
Besides, for sufficiently large $T$, by \Cref{lemma:convergence-proportion-w} we have
\begin{align*}
   \sum_{\xb\in\cX^-}w_{\xb}\leq 2\cdot \sum_{\xb\in\cX^-}w^*_{\xb}.
\end{align*}
Combining this with \eqref{eq:b2-the-second-batch-size}, we get
\begin{align*}
    b_2&=a_2+(\log T)^{1+\gamma}+\sum_{\xb\in\cX^-}w_{\xb}\cdot\alpha\log T \\
   &\leq a_2+(\log T)^{1+\gamma}+2\sum_{\xb\in\cX^-}w^*_{\xb}\cdot\alpha\log T\\
    &=\Theta\big((\log T)^{1/2}\big)+(\log T)^{1+\gamma}+2\sum_{\xb\in\cX^-}w^*_{\xb}\cdot\alpha\log T\\
    &=\Theta\big((\log T)^{1+\gamma}\big) ,
\end{align*}since $\{w^*_{\xb}\}_{\xb\in\cX^-}$ is determined by the bandit instance which is independent of $T$. Hence for sufficiently large $T$, we have $b_2\leq U(\log T)^{1+\gamma}$ for some constant number $U>0$.
    Then, putting this upper bound of $b_2$ into the the expression of $\beta(b_2,1/T)$ from \eqref{eq:beta-definition}, we get
\begin{align}\label{eq:beta}
    \beta(b_2,1/T)&= (1+1/\log\log T)\log \bigg(\frac{(b_2\log\log T )^{d/2}}{1/T}\bigg)\notag\\
    &\leq (1+1/\log\log T)\log \bigg(\frac{(U\log\log T )^{d/2}(\log T)^{d(1+\gamma)/2}}{1/T}\bigg)
\end{align}for some  constant number $U>0$.
Combining  \eqref{eq:Zb2}, \eqref{eq:alpha-logT-selection}, \eqref{eq:beta},  we have for sufficiently large $T$,
\begin{align*}
    Z(b_2)
    &\geq\alpha\log T\\
    &=(1+1/\log\log T)\log \bigg(\frac{(\log T)^{d}}{1/T} \bigg)\\
    &\geq (1+1/\log\log T){\log \bigg(\frac{(U\log\log T)^{d/2}(\log T)^{d(1+\gamma)/2}}{1/T}\bigg)}\\
    &\geq \beta(b_2,1/T),
\end{align*}which implies $Z(b_2)\geq \beta(b_2,1/T)$ for sufficiently large $T$.
    
    So we prove that if $\cE_2$ holds, for sufficiently large $T$, the stopping condition \eqref{eq:stopping-rule} holds in the second batch. In addition, by \Cref{lemma:concrete-concentration-first-two-batches-E-2}, $\cE_2$ happens with probability at least $1-2/(\log T)^2$. Now it can be concluded that the Chernoff's stopping rule in \eqref{eq:stopping-rule} holds with probability at least $1-2/(\log T)^2$
    
\subsection{Proof of \Cref{lemma-break-thrid-batch}}

    In this proof we continue using the good event $\cE_0$ \eqref{eq:good-event-in-elimination-algorithm} defined before. Since it happens with probability at least $1-1/T$, we can just assume it happens, then prove the elimination fact.

    Under the good event $\cE_0$ \eqref{eq:good-event-in-elimination-algorithm}, we have shown in \eqref{eq:eliminate} that arm $\xb\neq {\xb}^*$ with gap
    \begin{align}\label{eq:elimination-in-the-third-batch}
        \Delta_{\xb}>4\varepsilon_3=4\sqrt{d\log(KT^2)/(\log T)^{1+\gamma}}
    \end{align}would be eliminated after the third batch. As $T\rightarrow\infty$, \eqref{eq:elimination-in-the-third-batch} holds for all suboptimal arms $\xb\neq {\xb}^*$, since $\Delta_{\min}> 0$ and $\varepsilon_3\rightarrow 0$ as $T\rightarrow\infty$. So we finish our proof.
    
\subsection{Proof of \Cref{regretII}}

    With $\cE_2$ defined in \eqref{eq:concentration-first-two-batches-E-2}, we can decompose the regret of the second batch by
    \begin{align*}
Regret_2=Regret_{2\cE_2}\PP(\cE_2)+Regret_{2\cE_2^c}\PP(\cE_2^c),
    \end{align*}where we  let $Regret_{2\cE_2}$ and $Regret_{2\cE_2^c}$ be the regret of the second batch given $\cE_2$ and $\cE_2^c$, respectively.
    \Cref{lemma:concrete-concentration-first-two-batches-E-2} shows
    $\PP(\cE_2)\geq 1-2/(\log T)^2$. In $\cE_2^c$, the regret is negligible for this $O\big((\log T)^{1+\gamma}\big)$-size batch as $T\rightarrow\infty$, because
    \begin{align*}
     \lim_{T\rightarrow\infty}Regret_{2\cE_2^c}\PP(\cE_2^c)=
        \lim_{T\rightarrow\infty}\frac{2}{(\log T)^2}\cdot O\big((\log T)^{1+\gamma}\big)=0.
    \end{align*}
    This implies
    \begin{align*}
        \limsup_{T\rightarrow\infty}\frac{Regret_2}{\log T}
        &= \limsup_{T\rightarrow\infty}\frac{Regret_{2\cE_2}\PP(\cE_2)+Regret_{2\cE_2^c}\PP(\cE_2^c)}{\log T}\\
        &= \limsup_{T\rightarrow\infty}\frac{Regret_{2\cE_2}\PP(\cE_2)}{\log T}\\
        &\leq  \limsup_{T\rightarrow\infty}\frac{Regret_{2\cE_2}}{\log T}.
    \end{align*}
   Moreover, when $\cE_2$ appears, we can use the continuity of the problem to end the proof.  

    In fact,when the concentration result $\cE_2$ holds, for all arms $\xb$, $\vert\hat\mu_{\xb}-\mu_{\xb}\vert\leq\epsilon=1/\log\log T$. So $\lim_{T\rightarrow\infty}\hat\mu_{\xb}= \mu_{\xb}$. In a similar way, $\lim_{T\rightarrow\infty}\hat\Delta_{\xb}=\lim_{T\rightarrow\infty} (\hat\mu_{{\xb}^*}-\hat\mu_{\xb})=  \mu_{{\xb}^*}-\mu_{\xb}=\Delta_{\xb}$. Then using \Cref{lemma:convergence-proportion-w}, we get $\lim_{T\rightarrow\infty}w_{\xb}=w^*_{\xb}$ for all $\xb\in\cX$, where we use $\{w^*_{\xb}\}_{\xb\in\cX}$ to denote $\wb(\Delta)$ solved using the true gaps defined in \Cref{definition-real-proportion}, and $\{w_{\xb}\}_{\xb\in\cX}$ to denote $\wb(\hat\Delta)$ solved using  estimated gaps defined in \Cref{definition-tracking-proportion}.
    As we discussed in \eqref{eq:b2-the-second-batch-size}, for sufficiently large $T$, in the second batch the agent first explores with a D-optimal design multi-set with size $o(\log T)$, then pull the best arm for $(\log T)^{1+\gamma}$ times and each suboptimal arm $\xb\in\cX$ for $w_{\xb}\cdot\alpha\log T$ times.
    Thus the regret is
    \begin{align*}
        \limsup_{T\rightarrow\infty} \frac{Regret_2}{\log T}
        &\leq \limsup_{T\rightarrow\infty} \frac{Regret_{2\cE_2}}{\log T}\\
        &=\limsup_{T\rightarrow\infty} \frac{o(\log T)+\sum_{\xb\in\cX^- }\alpha w_{\xb}\log T\Delta_{\xb}}{\log T}\\
        &=\limsup_{T\rightarrow\infty} \frac{\sum_{\xb\in\cX^- }\alpha w_{\xb}\log T\Delta_{\xb}}{\log T}\\
        &=\sum_{\xb\in\cX^-}w_{\xb}^*\Delta_{\xb}=c^*.
    \end{align*}
This completes the proof.    
\end{document}